\documentclass[%
reprint,
superscriptaddress,
amsmath,amssymb,
aps,
prx
]{revtex4-2}

\usepackage{graphicx}
\usepackage{dcolumn}
\usepackage{bm}
\usepackage[colorlinks]{hyperref}
\usepackage{physics}
\usepackage[squaren]{SIunits}
\usepackage{mathtools}
\usepackage{amsthm}
\usepackage{xcolor}
\usepackage{mathrsfs}
\usepackage{times}
\usepackage[normalem]{ulem}
\usepackage{algorithm2e}
\usepackage{xurl}
\usepackage{comment}

\usepackage{appendix}
\usepackage{titletoc}
\usepackage{tocloft}
\usepackage{etoolbox}
\usepackage{titlesec}
\usepackage{textcase}

\usepackage{tocloft}

\newtheorem{theorem}{Theorem}
\newtheorem{lemma}[theorem]{Lemma}

\newtheorem{definition}[theorem]{Definition}

\hypersetup{
	colorlinks=true,
	linkcolor=blue,
	filecolor=blue,      
	urlcolor=blue,
	citecolor=blue
}
\newcommand{\rev}[1]{{\color{black}{#1}}}

\newcommand{\mathi}{\mathrm{i}}
\newcommand{\dt}{\delta t}
\newcommand{\FPto}{\longrightarrow}

\begin{document}

\title[]{Local Diffusion Models and Phases of Data Distributions}


\makeatletter

\author{Fangjun Hu}
\email{fhu@quera.com}
\affiliation{Department of Electrical and Computer Engineering, Princeton University, Princeton, NJ 08544, USA}
\affiliation{QuEra Computing Inc., 1284 Soldiers Field Road, Boston, MA 02135, USA}

\author{Guangkuo Liu}
\email{guangkuo.liu@colorado.edu}
\affiliation{JILA and Department of Physics, University of Colorado Boulder, Boulder, CO 80309, USA}

\author{Yifan F. Zhang}
\email{yz4281@princeton.edu}
\affiliation{Department of Electrical and Computer Engineering, Princeton University, Princeton, NJ 08544, USA}

\author{Xun Gao}
\email{xun.gao@colorado.edu}
\affiliation{JILA and Department of Physics, University of Colorado Boulder, Boulder, CO 80309, USA}

\date{\today}

\begin{abstract}

As a class of generative artificial intelligence frameworks inspired by statistical physics, diffusion models have shown extraordinary performance in synthesizing complicated data distributions through a denoising process gradually guided by score functions.
Real-life data, like images, is often spatially structured in low-dimensional spaces. However, ordinary diffusion models ignore this local structure and learn spatially global score functions, which are often computationally expensive.
\rev{In this work, motivated by recent advances in non-equilibrium statistical physics, we develop a generic framework for defining phases of data distributions and use it to analyze the locality requirements of denoisers in diffusion models.}
We define two distributions as belonging to the same data distribution phase if they can be mutually connected via spatially local operations such as local denoisers, \rev{along the same evolution path as the diffusion}.
We demonstrate that the reverse denoising process consists of an early trivial phase and a late data phase, sandwiching a rapid phase transition where local denoisers must fail.
\rev{We further demonstrate that the performance of local denoisers is closely tied to spatial Markovianity, which provides an operational criterion for diagnosing such phase transitions. We validate this criterion through numerical experiments on real-world datasets.}
Our work suggests guidance for simpler and more efficient architectures of diffusion models: far from the phase transition point, we can use small local neural networks to compute the score function; global neural networks are only necessary around the narrow time interval of phase transitions.
This result also opens up new directions for studying phases of data distributions, the broader science of generative artificial intelligence, and guiding the design of neural networks inspired by physics concepts.

\end{abstract}

\maketitle

\section{Introduction}

Inspired by the analogy to diffusion processes in non-equilibrium thermodynamics, diffusion models offer a physically intuitive framework for learning and generating complex data distributions \cite{sohl_deep_2015, song_generative_2019, ho_denoising_2020, song_ddim_2021, song_sde_2021}. After numerous testaments in practice, the denoising diffusion probabilistic model (DDPM) \cite{ho_denoising_2020} and its variants, like the denoising diffusion implicit model (DDIM) \cite{song_ddim_2021} and flow matching \cite{lipman2022flow}, have performed excellently in generating high-quality and diverse images and videos.
These advantages have made diffusion models cornerstones of many recent breakthroughs in text-to-image and text-to-video generation \cite{midjourney, stable2022, openai_dalle, sora_openai, imagen_google_web}. 

Although diffusion models have achieved huge successes in image and video generation, their training cost is also tremendous.
In general, diffusion models generate complicated data distributions by a diffusion process that evolves the desired distribution to another simple distribution (usually white noise obeying a pixel-wise independent Gaussian distribution); and then denoising from the white noise to the desired distributions (see Fig.\,\ref{fig:schematic}a).
The denoising process is constructed by introducing a distribution-dependent drift term -- called the \textit{score function}.
While the forward diffusion is usually performed locally in each pixel, the time-reversal denoiser in practice acts globally on the entire image. Therefore, score functions usually have to be learned by training a complicated neural network on a large dataset, such as score matching methods \cite{hyvarinen_score_2005, song_generative_2019}. Training and generation of these scores are computationally expensive, which constitutes a bottleneck in saving the overhead of diffusion models.

However, real-life data often exhibits a structure of \textit{spatial locality}. In images, for instance, the position of a pixel and its correlation with its neighborhood carry meaningful information.
As ordinary diffusion models neglect this locality information, diffusion models incorporating local denoising mechanisms have attracted growing interest in the machine learning community. 
This idea of computing score functions locally -- referred to as \emph{local diffusion models} (also known as patch diffusion models) -- has shown empirical success \cite{wang2023patch, ding2023patched, kamb2024analytic, niedoba2024towards}. Nevertheless, a thorough theoretical understanding of such models remains underdeveloped, and it is still unclear under what conditions this local approximation is valid.

This work aims to understand the locality of denoisers by introducing a new perspective -- the \textit{phases of data distributions} \cite{chen2010local, coser_classification_2019, sang2024mixed, sang_statbility_2025}.
\rev{Our work is motivated by the historical evolution of phase definitions in physics, which have undergone repeated upgrades toward greater generality and operational meaning.
Many traditional notions of phases are not well-suited for modern machine learning settings.
For example, thermodynamic definitions based on non-analyticities of the free energy do not apply to non-equilibrium diffusion dynamics.
Similarly, in diffusion models with pixel-wise noise, two-point correlations are constrained by the data-processing inequality to decay monotonically along the forward process, preventing them from diagnosing phase transitions in DDPM or DDIM.
Symmetry-breaking or clustering bifurcation has therefore been explored as an alternative indicator of transitions
\cite{biroli_dynamical_2024, raya_spontaneous_2023, li_critical_2024, sclocchi_phase_2024, li_blink_2025, bachtis_2025_cascade}.
However, symmetry-breaking-based analyses typically require the identification of appropriate order parameters, which can be challenging outside of highly structured settings. In many practical cases, such order parameters are not known a prior and must be introduced in a data-specific manner, often relying on deliberately selected or preprocessed datasets.
This makes symmetry-breaking-based approaches difficult to generalize to arbitrary real-world data.
}

\rev{Until very recently, frontier developments in the theory of quantum mixed states proposed a recovery-based notion of phases, in which two mixed states are said to belong to the same phase if they can be connected by a finite-depth sequence of local quantum channels along the same evolution path \cite{coser_classification_2019, sang2024mixed, sang_statbility_2025, sang_2025_mixed}.
This recovery-based viewpoint traces back to the locality-based classification of quantum pure state phases initiated by Chen, Gu, and Wen  \cite{chen2010local} and has since developed extensively within a broader and active research field known as \textit{topological order} \cite{wen_1989_vacuum, wen_1990_topological, wen_2007_quantum, chen_2011_two, chen_2013_symmetry, haah_2011_local, haah_2016_invariant, zeng_2019_quantum}.
This operational definition avoids equilibrium assumptions and does not rely on symmetry or correlation length.
Only with these recent advances has such a framework reached sufficient generality and robustness to be meaningfully applied to data distributions arising in modern machine learning.
}

\rev{Building on this perspective, the main contribution of this paper is to introduce a recovery-based, operational definition of phases for classical probability distributions and to demonstrate its relevance for diffusion models.}
Informally, we define two classical distributions as belonging to the same phase if they can be transformed into each other through a series of local channels \rev{along the same evolution path in probability distribution space}.
And in the context of diffusion models, channels in the forward and backward processes correspond to the forward diffusion operations and the backward denoisers, respectively.
\rev{Such a recovery-based definition applies to arbitrary data distributions and does not rely on symmetry assumptions, making it particularly well suited for generative modeling in high-dimensional, unstructured data.}

By analyzing the minimal sizes of the denoisers, we reveal a phase transition from the \textit{trivial phase} to the \textit{data phase} during denoising. In both the early and late stages of denoising, the transient distributions reside in the trivial and data phases, respectively, and the score functions can be computed locally. However, there exists a narrow intermediate time window during which a phase transition occurs, requiring global information to accurately compute the score function.

This perspective provides important guidance for designing neural networks in diffusion models.
Specifically, during the diffusion process, whenever the data distribution is inside a phase, we can always design a local denoiser to reverse the diffusion process at this time step.
Although a global denoiser is necessary during the phase transition, the transition typically occurs over a short time span, suggesting a net reduction in overall computational cost.

We investigate and validate data phase transitions through multiple approaches. First, we diagnose the phase transition by using the \rev{classical} conditional mutual information (CMI) along the diffusion path. CMI quantifies the amount of non-local information needed to compute the score function, and we prove that local denoising is possible if CMI decays exponentially with distance, along the whole diffusion path.
Additionally, we train local denoisers of varying sizes (namely, \textit{receptive fields}) and benchmark their output score functions. 
When applied to the \emph{MNIST} and \emph{Fashion-MNIST} database \cite{lecun1998mnist, xiao2017fashionmnist}, simple datasets of handwritten digit images and fashion images, both techniques reveal a phase transition at roughly the same time point, marked by the emergence of long-range CMI and the failure of small-sized local denoisers.

We remark that our proof techniques are inspired by the recent advances in understanding the local recovery channels and phases of matters in open quantum systems \cite{sang_statbility_2025}.
Unsurprisingly, we can show that the local reversibility of classical data distributions can be derived by taking the decoherence limit of the local reversibility of quantum mixed states.
This establishes a fundamental classical-quantum correspondence between these two concepts.


The discovery of data phase transitions opens up new directions for both theoretical understanding and practical engineering of diffusion models. From a physics perspective, this introduces a new domain to explore phases of matter, classification of phases, and universality classes of phase transitions. From a machine learning perspective, locality and phase transitions emerge as intrinsic structures of data that neural networks can utilize. This also may help explain attributes such as creativity and generalization \cite{kamb2024analytic}, which underlie the success of diffusion models. Looking forward, we also hope that this work could stimulate more discussions around the physical principles behind generative artificial intelligence.

\begin{figure}[t]
    \centering
    \includegraphics[width=\linewidth]{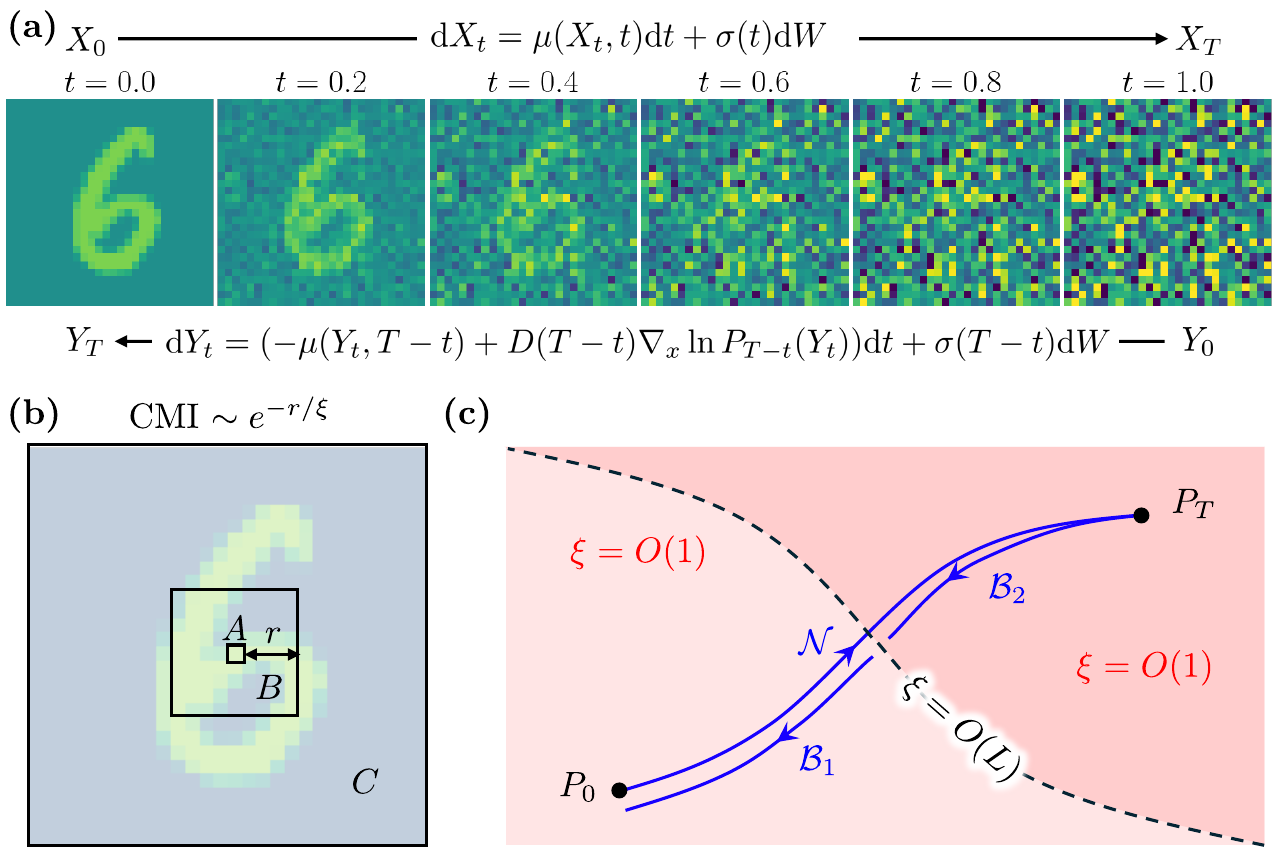}
    \caption{Schematic of diffusion models and phases of data distributions. Panel (b, c) is modified from Fig.\,1 of Ref.\,\cite{sang_statbility_2025}.
    (a) Diffusion models for image generation, presenting noisy images at different time steps. The forward SDE diffuses data to white noise, and the backward SDE denoises white noise to data. (b) Tripartition of data $X$ (sampled from $P$) into $A,B$ and $C$. Region $A$ has a constant diameter $k$. The width $r = \mathrm{dist}(A, C)$ of $B$ characterizes the separation between $A$ and $C$. Data distribution $P$ has a Markov length $\xi$ if CMI $I(X_A:X_C|X_B)_P \sim e^{-r/\xi}$. (c) During the diffusion process $\mathcal{N}$, Markov length is finite on both sides of the phase boundary, so there exist local denoisers $\mathcal{B}_1$ and $\mathcal{B}_2$. However, Markov length diverges near the critical time so global denoisers are required there. For the dataset of the handwritten digits, the phase transition during the diffusion occurs roughly at $t_c = 0.3 \sim 0.4.$
    }
    \label{fig:schematic}
\end{figure}

\section{Diffusion Models}

Consider a $d$-dimensional lattice $\Lambda$ of linear size $L$. Each site supports a continuous random variable in $\mathbb R$ so the sample space is $\mathcal{X} = \mathbb{R}^K$ with data space dimension $K=L^d$.
A dataset $\{X^{(i)}\}_{i \in [N_{\mathrm{data}}]}$ is randomly sampled from some target distribution $P_0(x)$. 
Here, $[N_{\mathrm{data}}]$ represents the integer set $\{1, \cdots, N_{\mathrm{data}} \}$ where $N_{\mathrm{data}}$ is the number of samples in the dataset. Each $x \in \mathcal{X}$ represents an image or a video embedded in a $K$-dimensional space.
For instance, in the MNIST dataset, images are grayscale and 2D (i.e., $d=2$) with a width and height of $L=28$. Hence, all these images consist of $K=784$ pixels (see Fig.\,\ref{fig:schematic}a) and each image is sampled from a desired distribution $P_0$.

In general, the transformation between different distributions is realized through \textit{noisy channels}. Let $P$ be any probability distribution, and a noisy channel $\mathcal{N} (y |  x) $ is a conditional probability that induces the transformation $\mathcal{N} (P) (y) = \int \dd x \, \mathcal{N} (y | x) P (x)$. According to Bayes' theorem, we define the recovery channel $\mathcal{B}_{\mathcal{N}, P}$ that maps $\mathcal{N} (P)$ to $P$ as 
\begin{equation}
    \mathcal{B}_{\mathcal{N}, P} (x |  y) = \frac{\mathcal{N} (y |  x) P (x)}{\mathcal{N} (P) (y)} . \label{eq:bayes}
\end{equation}
One can verify that $(\mathcal{B}_{\mathcal{N}, P} \circ \mathcal{N})P(x) = P(x)$.

In the general DDPM formalism, diffusion models can be formulated by the evolution of data distributions.
Given the desired dataset $\{X^{(i)}\}_{i \in [N_{\mathrm{data}}]}$, we first sample input data $X_{t=0}$ from $P_0$ and evolve it through a series of infinitesimally weak noisy channels. If the evolution time $\dt$ of the channel $\mathcal{N}$ is infinitesimal, then the operation acting on the random variable $X_t$ can be characterized by a stochastic differential equation (SDE) $\dd X_t = \mu (X_t, t) \dd t + \sigma (t) \dd W$, where $\mu \in \mathbb{R}^{K}$ is the drift vector, $\dd W \in \mathbb{R}^K$ is a standard Wiener increment vector, and $\sigma \in \mathbb{R}^{K \times K}$ is a matrix characterizing the diffusion strength. 
The dynamics of the probability distribution $P_t(x)$ is given by the \textit{Fokker-Planck equation} $\partial_t P= \mathcal{L}_{\mathrm{FP}}P$, under the continuous-time limit $\dt \to 0$. More concretely,
\begin{equation}
    \partial_t P = - \nabla_x \cdot (\mu P) + \frac{1}{2} \nabla_x \cdot (D \nabla_x P), \label{eq:forward}
\end{equation}
where $D(t) = \sigma(t) \sigma (t)^T \in \mathbb{R}^{K \times K}$ is the diffusion matrix. We say a Fokker-Planck equation is $k$-local if it  holds that $\mathcal{L}_{\mathrm{FP}}(t) = \sum_l \mathcal{L}_{\mathrm{FP},l}(t)$, where each $ \mathcal{L}_{\mathrm{FP},l}$ is a differential operator acting on a support indexed by $l$. Each support has a linear size at most constant $k$. Governed by this local Fokker-Planck equation, $P_t(x)$ ultimately evolves to a distribution $P_{t=T}(x)$ which is usually very close to the steady distribution $P_\infty$.
In the simplest form of diffusion models, $\mu(x) = - (x_1, \cdots, x_K)$ and $\sigma(t) \equiv I$ is a constant matrix. This SDE describes an Ornstein-Uhlenbeck process, whose steady distribution is a pixel-wise independent Gaussian distribution.

The core idea of the diffusion models is to denoise from the steady distribution backward to the desired distribution along the same path in the diffusion process.
Concretely, we generate a random sample $Y_{t=0} \sim Q_{0}=P_\infty$ and then evolve it to $Y_{t=T} \sim Q_T$, such that $Q_T$ is (approximately) the same as the target $P_0$.
This time-reversal evolution can be implemented by the backward denoising Fokker-Planck equation
\begin{equation}
    \partial_t Q = - \nabla_x \cdot ( (-\mu + D s) Q) + \frac{1}{2} \nabla_x \cdot (D \nabla_x Q), \label{eq:backward}
\end{equation}
where an extra drift term $s (x,t) = \nabla_x \ln P_t (x)$ called \textit{score function} was introduced in the literature \cite{anderson_reverse_1982, song_sde_2021}. 
This reverse evolution was derived from Bayes' theorem \cite{sohl_deep_2015, ho_denoising_2020}.
For completeness, we also provide a derivation of Eq.\,(\ref{eq:backward}) in SM \ref{apxsec:cv}, by directly taking the limit $\dt \to 0$, and computing the generator of $\mathcal{B}_{\mathcal{N}, P}$ associated with Eq.\,(\ref{eq:forward}).
The corresponding SDE of Eq.\,(\ref{eq:backward}) is $\dd Y_t = (- \mu (Y_t, T - t) + D(T - t) s(Y_t, T - t) ) \dd t + \sigma (T - t) \dd W$. 
Since $s(x, t)$ depends on $P_t(x)$, whose value is unknown, we need to use a neural network to learn it through methods like score matching.

In practice, the forward process is decomposed into $N = T/\dt$ discrete time points $0=t_0<t_1<\cdots<t_N=T$. We will always use $n$ as the discrete labels of the time step instead of $t$ if no confusion is caused. In each time interval $[t_{n-1}, t_n]$, we use a noisy channel $\mathcal{N}_n$ generated by the local Fokker-Planck equation evolving for short duration $\dt$, such that the overall channel is $\mathcal{N}_{\mathrm{tot}} = \mathcal{N}_{N} \circ \cdots \circ \mathcal{N}_{2} \circ \mathcal{N}_{1}$. The recovery of each $\mathcal{N}_n$ can be done via $\mathcal{B}_{n} = \mathcal{B}_{\mathcal{N}_n, P_{n-1}}$.
Here, the \textit{denoiser} $\mathcal{B}_{\mathcal{N}_n, P_{n-1}}$ is the Bayes recovery channel defined in Eq.\,(\ref{eq:bayes}) and $P_{n} = \mathcal{N}_{n} \circ \cdots \circ \mathcal{N}_{1} (P_0)$ is the distribution at time $t_{n}$. The overall denoiser can be expressed as $\mathcal{B}_{\mathrm{tot}} = \mathcal{B}_{1} \circ \mathcal{B}_{2} \circ \cdots \circ \mathcal{B}_{N}$ and $\mathcal{B}_{\mathrm{tot}}(P_\infty) \approx P_0$.

Typically, learning this score requires a global neural network that operates over the entire $K$-dimensional data space, which makes both training and inference computationally expensive.
We refer to the denoiser of a diffusion model at time $t$ as a local denoiser if its corresponding Bayes recovery channel is governed by the backward Fokker–Planck equation in Eq.\,(\ref{eq:backward}) whose associated score function is local at that time. Otherwise, we call it a global denoiser.

\section{Theoretical Results}
\label{sec:TR}

The theoretical contributions of this paper consist of two main parts.
First, we reveal that the existence of local denoisers in diffusion models is closely related to \rev{the spatial Markovianity, which is usually quantified in literature by an information-theoretic quantity -- conditional mutual information (CMI).}

Heuristically, local reversibility is ensured when and only when the CMI of the data distribution is small. We provide quantitative analyses of the relationship between local reversibility and CMI in Secs.\,\ref{sec:cmi}–\ref{sec:local-multiple}.
Second, building on the concept of local reversibility, we formally define the phases of data distributions in Sec.\,\ref{sec:phase}. This new perspective on data structure offers valuable insights and guiding principles for neural network design in diffusion models, which we further discuss in Sec.\,\ref{sec:guidance}.

\subsection{Conditional mutual information and Markov length}
\label{sec:cmi}

Suppose the underlying lattice $\Lambda$ is spatially partitioned into three regions $A,B,C$ (see Fig.\,\ref{fig:schematic}b). Here, $A$ is a local region -- supporting the forward diffusion operation -- with constant linear size $k$; $B$ is an annulus-shaped buffer region surrounding $A$, and $C$ is the remaining region outside $B$.
The distance between $A$ and $C$ is denoted as $r = \mathrm{dist}(A,C)$, which is also the width of $B$. For the data $X_t$ taking values $x = (x_A, x_B, x_C)$, we provide a criterion of local reversibility by introducing the CMI, defined as $I(X_A:X_C|X_B) = H(X_{AB}) + H(X_{BC}) - H(X_{ABC}) - H(X_{B})$
for the tripartition regions, where $H$ is the Shannon entropy. 
We also say distribution $P$ has approximate \textit{spatial Markovianity} with a finite \textit{Markov length} $\xi$, if its CMI decays exponentially as 
\begin{equation}
    I (X_A : X_C |  X_B)_P \leq \gamma \, e^{-r/\xi}, \label{eq:fml}
\end{equation}
for some constant $\gamma$.
Our main theoretical result is that having a constant Markov length implies local reversibility.
\begin{theorem}[\textit{Informal}]
    Let $P_n$ be the distribution during the diffusion at time-step $n$. Suppose $P_n$ has a constant Markov length for all introduced $A, B, C$ tripartitions, then the denoising operation at time $n$ can be performed with a local denoiser.
\end{theorem}
We prove this theorem by first giving an intuition of the local reversibility of a single-step denoiser in Sec.\,\ref{sec:local-single}, and then generalizing the result to the general case of multiple-step denoisers in Sec.\,\ref{sec:local-multiple}.

\subsection{Local reversibility of a single-step denoiser}
\label{sec:local-single}

Let us start by providing an intuition that for any time-step $n$, a short-range CMI is equivalent to the approximate locality of the score function, indicating that CMI is a natural indicator of the existence of local denoisers.
In general, the decomposition $P_{ABC} = P_{AB} P_{C|AB}$ always holds, where $P_{AB}$ is the marginal distribution on $AB$ and $P_{C|AB}$ is the conditional distribution on $C$ given $AB$.
If $P_{ABC}$ has a small CMI, this approximate conditional independence implies $P_{ABC} \approx P_{AB} P_{C|B}$. Recall that the score function is defined as $\partial_x \ln P$. Taking the logarithm and $x_A$-derivative on both sides, and noting that $\partial_{x_A} \ln P_{C|B} \equiv 0$, we find $\partial_{x_A} \ln P \approx \partial_{x_A} \ln P_{AB}$, which depends only on the local region $A \cup B$.

The intuition above is for the special case of $\dt \to 0$. Moreover, we can rigorously generalize the idea to a broader class of any finite-time $\mathcal{N}$ and $\mathcal{B}_{\mathcal{N}, P}$.
Suppose an arbitrary noisy channel $\mathcal{N} (y_A | x_A)$ acting only locally on $A$ with constant linear size $k$. We find that we can reverse the effect of $\mathcal{N}$ by only applying an approximated recovery channel on $A \cup B$, as long as the CMI $I (X_A : X_C |  X_B)_P$ is small. In fact, for any $\mathcal{N} (y_A | x_A)$ on $A$ and marginal distribution $P_{A B} (x_A, x_B)$, we can construct a local Bayes recovery channel
\begin{equation}
    \mathcal{B}_{\mathcal{N}, P_{A B}} (x_A, x_B |  y_A, x_B) \!=\! \frac{\mathcal{N} (y_A | x_A) P_{A B} (x_A, x_B)}{\int \! \dd x_A \mathcal{N} (y_A |  x_A) P_{A B} (x_A, x_B)} . \label{eq:local_bayes_ch}
\end{equation}
According to the classical Fawzi-Renner inequality \cite{li_squashed_2018, fawzi_quantum_2015} (see also SM \ref{apxsec:local_bayes_cmi}), we can bound the recovery error between $P$ and $\hat{P} = \mathcal{B}_{\mathcal{N}, P_{A B}} \circ \mathcal{N} (P)$, by the CMI of $P$
\begin{equation}
    \mathrm{TV} (P, \hat{P} )^2 \leq D_{\mathrm{KL}} (P \| \hat{P} ) \leq I (X_A : X_C |  X_B)_P , \label{eq:local_bayes}
\end{equation}
where $\mathrm{TV} (P, \hat{P}) = \int \dd x\, |P(x) - \hat{P}(x)| /2$ is the \textit{total variation distance} and $D_{\mathrm{KL}} (P \| \hat{P} ) = \int \dd x P(x) \ln \small( P(x)/\hat{P}(x) \small)$ is the \textit{Kullback-Leibler (KL) divergence}.

We remark that $\mathcal{B}_{\mathcal{N}, P_{A B}} (x_A, x_B |  y_A, x_B)$ in Eq.\,(\ref{eq:local_bayes_ch}) requires the knowledge $(y_A, x_B)$ on $A \cup B$ but its operation $y_A \to x_A$ is only executed locally on $A$.
By taking the $\dt \to 0$ limit, the backward drift term in the SDE of $\mathcal{B}_{\mathcal{N}, P_{A B}}$ is exactly the local score function $\partial_{x_A} \ln P_{A B} \approx \partial_{x_A} \ln P$ (see Eq.\,(\ref{apxeq:local_sde}) of SM \ref{apxsec:local_bayes_sde}). The corresponding SDE acting on $A$ is explicitly
\begin{equation}
    \dd Y_A = \left( - \mu + \sigma^2 \partial_{x_A} \ln P_{AB} \right) \dd t + \sigma \dd W_A.
\end{equation}
Even though all the results above are derived under the DDPM formalism, they are also rigorously applicable to DDIM and flow matching, because it is well-known that DDIM and flow matching have exactly the same score form of the backward drift term as that of DDPM \cite{song_ddim_2021, lipman2022flow}.

\subsection{Local reversibility of multiple-step denoisers}
\label{sec:local-multiple}

So far, the connection between CMI decay and approximate local reversibility that we established is only for a single-step denoiser. Furthermore, we can generalize the conclusion to the scenario of multi-step denoisers.

Specifically, let us consider the forward diffusion channel at each time $\mathcal{N}_{n} = \prod_l \mathcal{N}_{n, l}$.
Each $\mathcal{N}_{n, l}$ acts on a region $A_{n,l}$ whose linear size is at most a constant $k$. Here, we use $l$ as the spatial labels of region $A_{n,l}$.
We denote the Markov length at time $n$ as $\xi_n$.
In SM \ref{apxsec:local_bayes_reorg}, we prove that there exists local denoisers $\mathcal{B}_{n} = \prod_l \mathcal{B}_{n, l}$, such that the overall total variation error is bounded by $\mathrm{TV} (P_0, \hat{P}_0) < \varepsilon$ where $\hat{P}_0 = \mathcal{B}_{\mathrm{tot}} \circ \mathcal{N}_{\mathrm{tot}} (P_0)$.
Here, each denoiser $\mathcal{B}_{n, l}$ is supported on $A_{n, l} \cup B_{n,l}$ where $B_{n,l}$ (an annulus-shaped region surrounding $A_{n, l}$, see Fig.\,\ref{fig:schematic}b or Fig.\,\ref{fig:guidance}) has a width $r_n$ as long as:
\begin{equation}
    r_n \gtrsim \xi_n \ln \left(N K / \varepsilon \right). \label{eq:bayes_radius}
\end{equation}
When all $\xi_n$ are finite, this implies a series of local denoising channels evolving the white noise to the desired data distribution.
The proof of the condition Eq.\,(\ref{eq:bayes_radius}) utilizes a reorganization trick (see SM \ref{apxsec:local_bayes_reorg}) that was initially proposed in Ref.\,\cite{sang_statbility_2025} for proving quantum mixed state local recoverability.
We remark that the term $K$ arises from the total number of local channels $\{\mathcal{N}_{n,l}\}$ at each time step $n$; and the factor $N$ in Eq.\,(\ref{eq:bayes_radius}) is kept merely for some technical reason, and we believe this factor is not essential (see comments in SM \ref{apxsec:local_bayes_reorg}).

\subsection{Phases of data distributions}
\label{sec:phase} 

The local reversibility result shown above provides a completely new way to understand data distribution. In analogy to the phases of quantum mixed states \cite{coser_classification_2019, sang_statbility_2025}, we can define those data distributions as being in the same phase if they can be mutually connected via \rev{a diffusion path of (quasi)-local Fokker-Planck equations, whose corresponding denoiser is also (quasi)-local.} 
Here, for a Fokker-Planck equation $\partial_t P = \mathcal{L}_{\mathrm{FP}} P$, the operator $\mathcal{L}_{\mathrm{FP}}(t) = \sum_l \mathcal{L}_{\mathrm{FP},l}(t)$ being \rev{local (or quasi-local)} means each $\mathcal{L}_{\mathrm{FP},l}(t)$ has \rev{constant (or $O(\mathrm{polylog} \, L)$)} spatial support and operator norm at any time $t$.
\rev{The distinction between locality and quasi-locality only for technical reason \cite{coser_classification_2019, sang_statbility_2025}, thus we will only use term ``local'' in this work, when no confusion is caused.}

\rev{To be more specific, suppose there exists a time-dependent local Fokker-Planck equation $\partial_t P = \mathcal{L}_{\mathrm{FP}} P$ whose solution $P_t(x)$ evolves $P_0$ to the $\varepsilon$-neighborhood of $Q_0$ within a unit time duration $t \in [0, 1]$ for an arbitrarily given $\varepsilon$, namely $\mathrm{TV}(P_{t=1}, Q_0) \leq \varepsilon$.
We denote this property of local diffusion as $P_0 \FPto Q_0$.
Then, we say that two distributions $P_0$ and $Q_0$ are in the same phase if and only if there \textit{exists} at least one diffusion path along which $P_0 \FPto Q_0$ holds, and its corresponding denoising process $Q_t$ is also governed by another \textit{local} Fokker-Planck equation $\partial_t Q = \tilde{\mathcal{L}}_{\mathrm{FP}} Q$ such that:
\begin{equation}
    \mathrm{TV}(Q_{t=1}, P_0) \leq \varepsilon~\text{and}~\mathrm{TV}(Q_t, P_{1-t}) \leq \varepsilon \label{eq:LR}
\end{equation}
for all time $t$.
The latter condition in Eq.\,(\ref{eq:LR}) emphasizes the fact that denoising and diffusion in diffusion models are approximately on the same path in the space of probability.
This definition is generic for any distributions, and does not depend on any symmetry structures or equilibrium assumptions of the distributions.
}

\rev{
Under this definition, phases are determined by the existence of a two-way connection along a single noise path via local channels, rather than by the properties of any particular noise realization.
In this sense, the phase equivalence relation depends only on whether such a local connecting path exists between the two distributions.
At the same time, we emphasize that while phase equivalence is defined at the level of existence, the manifestation of a phase transition -- such as the emergence of a diverging recovery length -- can depend strongly on the specific choice of connecting noise path.
This distinction between path-independent phase equivalence and path-dependent transition behavior will be further discussed in Sec.\,\ref{sec:discussion}.
}


One may also ask whether this definition of phases agrees with the thermodynamic phases. Some progress has been made to address this question.
For example, it was shown in Ref.\,\cite{zhang2025conditional} that above a threshold temperature, all Gibbs distributions can be mutually connected via local channels. It is also strongly believed that at low temperatures and for finite-dimensional systems, two distributions being in the same thermodynamic phase implies mutual local connectivity \cite{ma2025circuit}.

\begin{figure}[t]
    \centering
    \includegraphics[width=\linewidth]{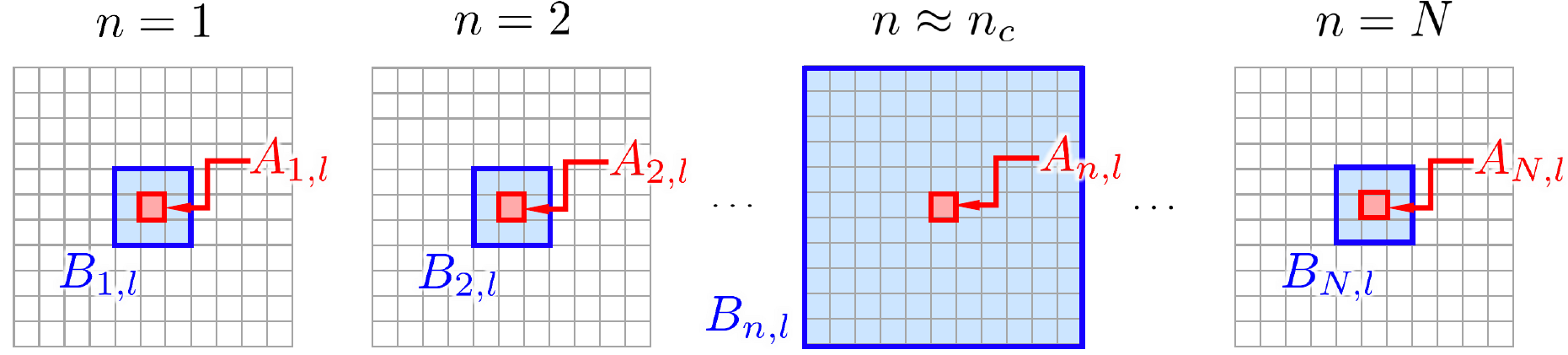}
    \caption{Schematics of designing local denoisers. For time step $n$ being far from the phase transition step $n_c$, denoising the forward channel acting on $A_{n, l}$ (in red, in examples of images, $A_{n, l}$ is a pixel whose coordinate is labeled by $l$), requires only a local denoiser acting on a small neighbouthood $A_{n, l} \cup B_{n, l}$ (in blue). Global denoisers are necessary when $n \approx n_c$.}
    \label{fig:guidance}
\end{figure}

\section{Guidance of Neural Network Design in Diffusion Models}
\label{sec:guidance}


According to the definition of the phase of data distributions based on the local recoverability, we can provide three guiding principles of designing neural networks for learning score functions in diffusion models.

First, we only need a small neural network to learn the score function when the data distribution $P_n$ is far from the phase boundary, and use a large neural network when $P_n$ is close to the phase boundary (see Fig.\,\ref{fig:guidance}). This is because local denoisers can connect two distributions in the same phase by definition.
Second, in the practice of diffusion models, the time step length $\dt$ is not fixed over the whole diffusion process. One can choose arbitrary step-dependent $\{\dt_n\}$, and the series $\{\dt_n\}$ is called the \textit{noise schedule} in diffusion models. 
The perspective of data distribution phases suggests that, for those commonly used schedules, we may insert more time steps when $P_n$ is close to the phase boundary to increase the quality of the denoised images.
Third, in the case where $P_n$ is far from the phase boundary, if the distance $r_n$ is sufficiently small, we can even learn the score function directly from the data distribution without using any neural networks. Because the local denoiser only requires the information of a small region due to Eq.\,(\ref{eq:local_bayes_ch}), the corresponding marginal probability value can be estimated with not too many samples of data $X_n$, e.g., through kernel density estimation \cite{rosenblatt1956remarks, parzen1962estimation}.

In this work, we focus on the first guiding principle mentioned above, depicted in Fig.\,\ref{fig:guidance}. 
If the Markov length $\xi_{n_c} = O(L)$ at some step $n_c \in [N]$, a phase transition occurs. In other words, the CMI at time step $n$ near $n_c$ becomes large even at a large $r_n$. There are two possible cases along the forward diffusion process. 
The first case is that $n$ is far from $n_c$ and the Markov length $\xi_n = O(1)$ is small for $P_n$. It means that $P_n$ is inside one phase. It massively mitigates the hardness and cost to learn the score at this time step. 
According to Eq.\,(\ref{eq:local_bayes_ch}), we only need to learn the score function $\partial_{x_A} \ln P_{A_{n,l}B_{n,l}}$ based on the information on the local region $A_{n,l} \cup B_{n,l}$ of image or video. Thus, learning this score function could be done patch-by-patch, which should be much less expensive.
The other case is when $n \approx n_c$, that is, close to the phase transition. In this case, we will need to set $r_{n \approx n_c} = L$ and carry out the ordinary score learning algorithms on the whole data.

\begin{figure}[t]
    \centering
    \includegraphics[width=\linewidth]{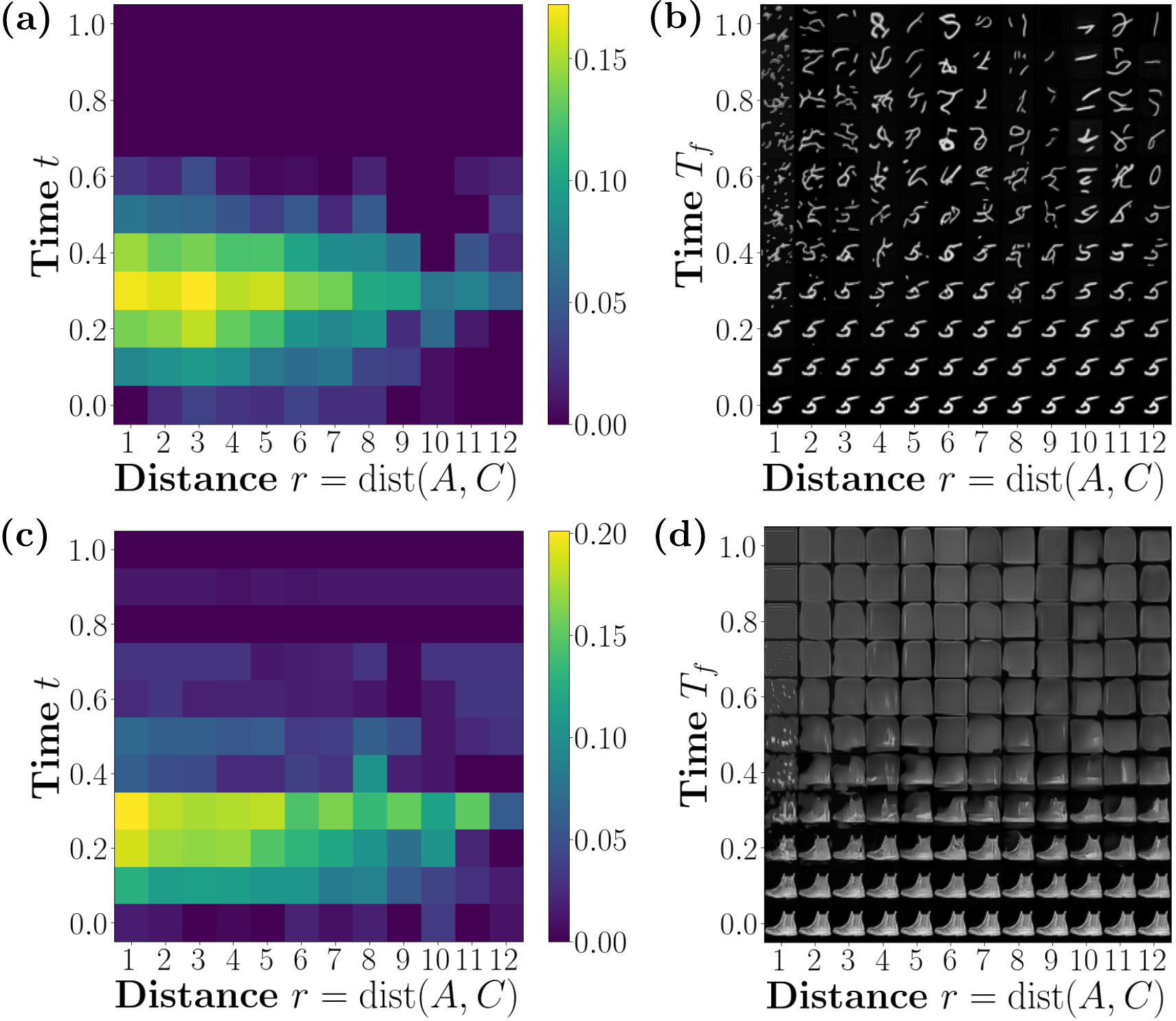}
    \caption{
    (a) CMI $I (X_A : X_C |  X_B)_{P_t}$ of MNIST as a function of distance $r = \mathrm{dist} (A, C)$ at different time $t$. 
    (b) The images locally denoised from corrupted images at $t=T_f$. Each local denoiser acts on region $A \cup B$ with diameter $2r+1$. Local denoisers with any $r$ perform badly if $T_f > 0.4$. 
    \rev{(c, d) The same plots of CMI heatmaps and denoised images for Fashion-MNIST dataset.}
    }
    \label{fig:mnist}
\end{figure}

\section{Numerical Results}
\label{sec:numerics}

We show that generating real-world data distributions using diffusion models may exhibit a phase transition that influences the network design. In our analysis, we focus on the MNIST dataset, and indeed, it exhibits a phase transition. At each time step, we apply diffusion by independently mixing every pixel with a standard Gaussian noise. In DDIM and flow matching, this process can be described by \cite{lipman2022flow, song_ddim_2021, Heitz2023alphaBlending}
\begin{equation}
    X_t = (1 - \alpha_t) X_0 + \alpha_t Z, \label{eq:gaussian}
\end{equation}
where $Z \in \mathbb{R}^{28 \times 28}$ represents pixel-wise independent Gaussian noise with zero mean and unit variance. The function $\alpha_t \in [0, 1]$, which governs the time dependence of the noise level, is the schedule. In this work, we adopt a linear schedule given by $\alpha_t = t$.

We numerically evaluate the CMI of the distribution of $X_t$ throughout the whole diffusion process. We rewrite the CMI into the form of $I (X_A : X_C |  X_B) = I (X_A : X_B X_C) - I (X_A : X_B)$, and then we utilize the mutual information neural estimation (MINE) method \cite{belghazi_mutual_2015} to train neural networks for estimating mutual information respectively (see details in SM \ref{apxsec:mine}).
We select the central pixels of the images to be $A$ so that $k=1$. Then, the CMI as a function of distance $r = \mathrm{dist} (A, C)$ at different time steps $t$ is shown in Fig.\,\ref{fig:mnist}a. 
In the limit case of $t=1$ and $t=0$, we observed that both CMI values are small even for a small distance $r$.
At $t = 1$ (trivial phase), the CMI is trivially zero because $X_{t = 1}$ is a pixel-wise independent Gaussian noise. For noiseless data at $t = 0$ (data phase), the reason for their small CMI is as follows. In general, the CMI can be upper-bounded by the conditional entropy $I (X_A : X_C | X_B) \leq H (X_A | X_B) $. For a noiseless image, when $B$ -- neighborhood surroundings of $A$ -- is given, $A$ is almost determined. Therefore, $H (X_A |  X_B)$ is small enough so that the CMI is also suppressed.
At $t_c \approx 0.3 \sim 0.4$, we observe a significant CMI barrier in our numerics, which indicates that there is a phase transition around this time step.
\rev{This CMI peak is not tied to a special spatial location. As shown in Fig.\,\ref{fig:mnist-apx}c of SM \ref{apxsec:mnist}, non-central pixels in MNIST exhibit qualitatively similar CMI behavior, demonstrating that the phenomenon is robust to the choice of pixel.}

\rev{We further emphasize that this transition is not captured by traditional two-point correlation measures. In DDPM and DDIM, Gaussian noise is added independently at each pixel, and the data-processing inequality implies that the mutual information between any two pixels $A_1$ and $A_2$ decreases monotonically along the diffusion process:
\begin{equation}
    I(X_{A_1} : X_{A_2})_{P_{t'}} \le I(X_{A_1} : X_{A_2})_{P_t},
\end{equation}
for any $t' > t$.
We verify this monotonic decay numerically in Fig.\,\ref{fig:cov} (see SM \ref{apxsec:cmi-2pc}). As a result, two-point correlations remain insensitive to the recovery-based phase transition revealed by CMI.
To further elucidate the distinction between CMI and two-point correlation, in SM \ref{apxsec:cmi-2pc}, we present several examples of distributions in which one of these two quantities is large while the other remains small.
}

\begin{figure}[t]
    \centering
    \includegraphics[width=\linewidth]{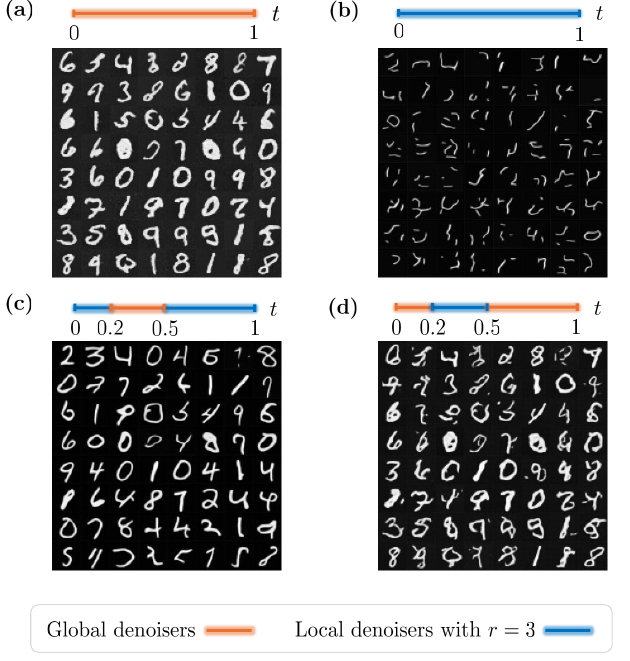}
    \caption{
    64 samples of denoised images, with local denoisers ($r=3$, in blue) within different time intervals during the denoising process. $t=0$ and $t=1$ correspond to the data phase and the trivial phase, respectively. (a) Ordinary diffusion models with global denoisers (in orange), no local denoisers used. (b) Only using local denoisers, consistent with the patterns of $( T_f, r)=(1, 3)$ in Fig.\,\ref{fig:mnist}b. (c) Using global denoisers only when around the phase transition (essentially Fig.\,\ref{fig:guidance}), the performance is as good as (a). (d) Using global denoisers only when far from the phase transition, many digits are hardly recognizable compared to (c). 
    }
    \label{fig:recovery}
\end{figure}


We validate the phase transition, probed via the CMI, by testing the efficacy of local denoisers. 
We sample clean data $X_0$ from the original dataset, and we evolve the data to $X_{T_f}$ under Eq.\,(\ref{eq:gaussian}) for a duration $T_f \in [0, 1]$. Then, we use flow matching to train local denoisers for recovering $X_0$. To get these local denoisers, we train a series of modified U-nets with small sizes (see details in SM \ref{apxsec:unet}).
For the denoiser acting on the pixel $A_{n,l}$ (i.e., $k=1$) at time step $n$, the small U-Net learns a score function whose input is a region $A_{n,l} \cup B_{n,l}$ where $B_{n,l}$ has a width $r$. We denote this denoised image as $Y_{T_f, r}$. 
All denoised images $Y_{T_f, r}$ with different $T_f$ and $r$ are depicted in Fig.\,\ref{fig:mnist}b. We observe that all local denoisers perform badly when $T_f > 0.4$, demonstrating that local denoisers always fail after the phase transition occurs. 
Local denoisers with smaller $r$ fail earlier than those with larger $r$. However, for different $r$, the deviation of the time when such failure occurs is small, consistent with the rapid growth of the CMI near the phase transition.

We also visualize the efficacy of our design principle in Fig.\,\ref{fig:recovery} by applying local denoisers in different stages of denoising.
Since the MNIST images have a finite size, the large CMI values in Fig.\,\ref{fig:mnist}a are concentrated within a finite interval instead of suddenly peaking at a single time step.
The intuition of this phenomenon can be explained through the analogy to the phase transition in statistical physics: around the phase transition, the Markov length scales as $\xi \propto 1/|t-t_c|^\nu$ where $\nu$ is some constant called the critical exponent and $t_c$ is the phase transition time \cite{sang_statbility_2025}.
If $t$ is close to but not $t_c$, the Markov length $\xi$ is comparable to the finite system size $L$, although $\xi$ is still finite.
This phase transition of data distributions is ubiquitous for many datasets. We provide more demonstrations of phase structures and phase transition for the diffusive generation of the Fashion-MNIST, a database of fashion images \rev{in Fig.\,\ref{fig:mnist}c and Fig.\,\ref{fig:mnist}d}, see the numerical details in SM \ref{apxsec:fashion-mnit}.

In the following numerics, we select $t \in [0.2, 0.5]$ as the interval around the phase transition.
For benchmark, in Fig.\,\ref{fig:recovery}a, we plot the denoised images using global denoisers over the whole denoising process $t \in [0, 1]$. The global denoisers are standard U-Nets \cite{ronneberger2015unet} (see details in SM \ref{apxsec:unet}). This is essentially the ordinary diffusion models \cite{ho_denoising_2020}.
As a sanity check, we show in Fig.\,\ref{fig:recovery}b that using local denoisers with $r=3$ over the whole denoising process $t \in [0, 1]$ fails to generate any recognizable digit images.
To verify our first guiding principle, in Fig.\,\ref{fig:recovery}c, we use the global denoisers within the interval $[0.2, 0.5]$ but the local denoisers with $r=3$ in $[0, 0.2] \cup [0.5, 1]$. We find the denoising performance is as good as the ordinary diffusion models. This is exactly the denoising scheme we proposed based on our perspective of the phases of data distributions.
Finally, in Fig.\,\ref{fig:recovery}d, we replace the global denoisers within the time interval $[0.2, 0.5]$ with the local denoisers with $r=3$; while keeping the denoisers in the rest of the time $[0, 0.2] \cup [0.5, 1]$ global. The denoising performance decreases dramatically, and many digits are not recognizable. This shows that local denoisers must fail around the phase transition.

\rev{
\section{Relation to Phases of Mixed States}

The technique of local diffusion models and the definition of phases for data distributions is inspired by the very recent study of Lindbladian local reversibility and mix-state phases in open quantum systems \cite{sang_statbility_2025}, which we introduce in this section.

To make this connection concrete, we briefly recall the notion of local reversibility in quantum mixed states and how it leads to a generic phase definition. In this setting, classical probability distributions and noisy channels are generalized to quantum mixed state $\rho$ and quantum channel $\mathcal{N}$.
A central result is that when a channel $\mathcal{N}$ acts locally on a region $A$, it can be approximately reversed by another channel acting on a slightly enlarged region $A B$, provided the quantum conditional mutual information $I(A:C|B)$ is small for the tripartition $\Lambda=ABC$ of a qubit lattice, with distance $r=\mathrm{dist}(A, C)$.
The corresponding recovery map admits an explicit construction, called \textit{twirled Petz map} \cite{junge_universal_2018, sang_statbility_2025} (see also Eq.\,(\ref{apxeq:TPRM}) in SM \ref{apxsec:cPetz}), which can be viewed as the quantum analogue of classical Bayesian recovery.

In the quantum setting, the analogue of the diffusion process is given by dissipative dynamics described by a \textit{Lindblad equation} $\dot{\rho} = \mathcal{D} [a] \rho = a \rho a^{\dagger} - (a^{\dagger} a \rho + \rho a^{\dagger} a) / 2$ with the jump operator $a$ acting locally.
The corresponding recovery dynamics is generated by the continuous-time limit of the twirled Petz map, which plays the role of a quantum generalization of denoisers and approximately reconstructs the initial state (see Theorem \ref{thm:lqd} of SM \ref{apxsec:cPetz}).
Within this framework, mixed-state phases can be defined as follows \cite{sang_statbility_2025, sang_2025_mixed}: two mixed states $\rho$ and $\sigma$ are said to belong to the same phase if $\rho$ can be transformed into $\sigma$ via a local Lindblad evolution, and the associated twirled Petz recovery maps remain local throughout the process.

Finally, the correspondence between quantum twirled Petz maps and classical diffusion models can be made explicit in the classical limit.
In fact, considering full decoherence in the quantum dynamics, diagonal mixed states reduce to classical probability distributions $\rho = \int \dd x \, P (x) \ket{x} \bra{x}$.
The quantum evolution generated by momentum jump operator $p$ reduces to the standard diffusion processes. This is because $\mathcal{D} [p] \rho = \int \dd x \, (\partial^2_x P/2) \ket{x} \bra{x} $, where the momentum jump operator does not cause the off-diagonal terms transition (see SM \ref{apxsec:diag_forward}). 
In this limit, the associated recovery dynamics given by the continuous-time twirled Petz map reduces to the classical denoising operation used in diffusion models: $\int \dd x ( - \partial_x ((\partial_x \ln P) P) + \partial^2_x P / 2 )  \ket{x} \bra{x}$ (see SM \ref{apxsec:diag_rotated}).
This establishes a direct connection between recovery-based phase definitions in quantum mixed states and the structure of classical diffusion models.
}

\section{Discussions}
\label{sec:discussion}

In this work, we use the approximated spatial Markovianity as a criterion for constructing local denoisers in diffusion models, and propose a definition of phases for different data distributions in machine learning. 
We verify that the phase transitions occur in the diffusion models of the real-world dataset by using different methods, including monitoring the CMI and recovery errors of local denoisers along the diffusion path.


Our framework of local reversibility also paves new paths for understanding machine learning from a physics perspective.
Notably, Markov length offers a refined notion of data phase transition by exploring the \emph{spatial locality} in the data structure. Earlier works have established the reverse generation process as a symmetry-breaking phase transition \cite{biroli_dynamical_2024, raya_spontaneous_2023, li_critical_2024, sclocchi_probing_2024, sclocchi_phase_2024, li_blink_2025}. The final Gaussian distribution is ``high-temperature" and contains only one valley in the energy landscape, whereas the data distribution is ``low-temperature" and possesses a complex energy landscape with many local minima. All these traditional views are ignorant of the spatial information in the pixels: they apply to images flattened to a $K$-dimensional vector. On the other hand, the Markov length constructions rely on the spatial information, offering a finer-grained approach to understanding phases of data distributions. An intriguing open question is whether these two types of phase transitions coincide in real-world data, and if so, whether they are driven by the same mechanism. 

In practice, we always need certain probes to diagnose the phase transition, based on which we can determine the radius $r_n$ for guiding the design of neural networks.
As mentioned in Sec.\,\ref{sec:TR}, the CMI is not the sole indicator for probing the locality of the denoiser. We may employ other methods to diagnose the phase transitions. 
For example, we can monitor the score function along the diffusion path or even train a highly efficient network for predicting phase transitions.
Moreover, as the last two guiding tools of the three we previously pointed out, we can further investigate the connection between phase transition and the noise schedule, as well as explore the training-free local diffusion models.

\rev{The phase perspective on data distributions we proposed naturally raises the question of more general noise choices in diffusion models, beyond standard pixel-wise Gaussian noise.
Allowing alternative noise models -- such as non-standard Gaussian noise or discrete Poisson-type noise -- enlarges the space of possible diffusion paths connecting the same pair of data distributions.
From this viewpoint, different noise choices correspond to different trajectories in probability distribution space, along which the locality of the associated denoisers may differ.}
Indeed, there may exist multiple paths connecting $P_{t=0}$ and $P_{t=1}$ such that one path has a Markov length divergence, but the Markov length is always finite along the other path (see an example in SM \ref{apxsec:LVPT}).
\rev{According to our definition in Sec.\,\ref{sec:phase}, these two distribution are still in the same phase.}
\rev{This highlights a fundamental feature in our framework: while the definition of a phase depends only on the existence of a local connecting path between two distributions, the presence or absence of a phase transition can depend on the particular path chosen.}
This scenario in diffusion models is analogous to the \textit{liquid-vapor phase transition} of water by bypassing the critical point.
\rev{From this perspective, the liquid–vapor analogy provides a useful conceptual guide for diffusion models: it suggests that some diffusion paths may avoid phase transitions and therefore require less nonlocal denoising.
Recent studies on phases of matter have already constructed efficient algorithms to search for such a generative denoising path using only local operations, for both the cases of quantum mixed states and classical distributions, as long as the state (distribution) is in the same phase as the product state (distribution), regardless of the existence of liquid-vapor-type phase boundary \cite{landau_2025_learning, kim_2024_learning, hu_2026_learning}.
}


Finally, the distinction between creativity and hallucination in diffusion models remains an active and important open question in the AI community.
Recent studies have suggested a link between creativity (or generalization) and two fundamental inductive biases in neural networks: locality and equivariance \cite{kamb2024analytic, lukoianov2025locality}. 
In essence, diffusion models generate creative outputs by randomly assembling small patches whose boundaries are locally consistent. 
While such images often appear coherent at a local level, the assembly of these patches still results in a lack of the same global correlations as the desired data.
We speculate that this absence of global correlations is what differentiates hallucination from genuine creativity \cite{Farquhar2024, aithal2024understanding}.
In contrast, states within the same phase exhibit similar global correlations. Accordingly, we propose that sampling within a phase provides a more refined notion of creativity -- where appropriate global or long-range correlations emerge near phase transitions through non-local networks -- thereby distinguishing meaningful creativity from undesirable hallucination.

\section*{Data and Code Availability}

The data generated for numerical results in this article and the code used in this study have been deposited in the GitHub repository under the accession link: \url{https://github.com/Fangjun-Hu/Local-Diffusion-Models-and-Phases-of-Data-Distributions}.

\section*{Acknowledgement}

F.\,H., G.\,L., Y.\,F.\,Z.\,and X.\,G.\,would like to thank Shengqi Sang for detailed technical discussions of the phases of quantum mixed states.
The authors acknowledge Yizhuang You, Xinyu Liu, and Yubei Chen for detailed discussions on topics of quantum and classical diffusion models.
The authors also thank Hakan Türeci, Sarang Gopalakrishnan, Soonwon Choi, Sitan Chen, Weiyuan Gong, Di Luo, Chi-Fang Chen, Liang Jiang, Su-un Lee, Bikun Li, Guo Zheng, Siddhant Midha, Rose Yu, Shengtao Wang, Milan Kornjača, and Pedro Lopes for stimulating discussions about the work that went into this manuscript.
This work was completed during F.\,H.'s Ph.\,D.\,Program at Princeton, prior to graduation, with support from AFOSR MURI award FA9550-22-1-0203.
Y.\,Z.\,acknowledges support from NSF QuSEC-TAQS OSI 2326767.
G.\,L.\,and X.\,G.\,acknowledge support from NSF PFC grant No.\ PHYS 2317149. 

\bibliography{bibtex}

\begin{thebibliography}{78}%
\makeatletter
\providecommand \@ifxundefined [1]{%
 \@ifx{#1\undefined}
}%
\providecommand \@ifnum [1]{%
 \ifnum #1\expandafter \@firstoftwo
 \else \expandafter \@secondoftwo
 \fi
}%
\providecommand \@ifx [1]{%
 \ifx #1\expandafter \@firstoftwo
 \else \expandafter \@secondoftwo
 \fi
}%
\providecommand \natexlab [1]{#1}%
\providecommand \enquote  [1]{``#1''}%
\providecommand \bibnamefont  [1]{#1}%
\providecommand \bibfnamefont [1]{#1}%
\providecommand \citenamefont [1]{#1}%
\providecommand \href@noop [0]{\@secondoftwo}%
\providecommand \href [0]{\begingroup \@sanitize@url \@href}%
\providecommand \@href[1]{\@@startlink{#1}\@@href}%
\providecommand \@@href[1]{\endgroup#1\@@endlink}%
\providecommand \@sanitize@url [0]{\catcode `\\12\catcode `\$12\catcode `\&12\catcode `\#12\catcode `\^12\catcode `\_12\catcode `\%12\relax}%
\providecommand \@@startlink[1]{}%
\providecommand \@@endlink[0]{}%
\providecommand \url  [0]{\begingroup\@sanitize@url \@url }%
\providecommand \@url [1]{\endgroup\@href {#1}{\urlprefix }}%
\providecommand \urlprefix  [0]{URL }%
\providecommand \Eprint [0]{\href }%
\providecommand \doibase [0]{https://doi.org/}%
\providecommand \selectlanguage [0]{\@gobble}%
\providecommand \bibinfo  [0]{\@secondoftwo}%
\providecommand \bibfield  [0]{\@secondoftwo}%
\providecommand \translation [1]{[#1]}%
\providecommand \BibitemOpen [0]{}%
\providecommand \bibitemStop [0]{}%
\providecommand \bibitemNoStop [0]{.\EOS\space}%
\providecommand \EOS [0]{\spacefactor3000\relax}%
\providecommand \BibitemShut  [1]{\csname bibitem#1\endcsname}%
\let\auto@bib@innerbib\@empty
\bibitem [{\citenamefont {Sohl-Dickstein}\ \emph {et~al.}(2015)\citenamefont {Sohl-Dickstein}, \citenamefont {Weiss}, \citenamefont {Maheswaranathan},\ and\ \citenamefont {Ganguli}}]{sohl_deep_2015}%
  \BibitemOpen
  \bibfield  {author} {\bibinfo {author} {\bibfnamefont {J.}~\bibnamefont {Sohl-Dickstein}}, \bibinfo {author} {\bibfnamefont {E.}~\bibnamefont {Weiss}}, \bibinfo {author} {\bibfnamefont {N.}~\bibnamefont {Maheswaranathan}},\ and\ \bibinfo {author} {\bibfnamefont {S.}~\bibnamefont {Ganguli}},\ }\bibfield  {title} {\bibinfo {title} {Deep unsupervised learning using nonequilibrium thermodynamics},\ }in\ \href {https://dl.acm.org/doi/10.5555/3045118.3045358} {\emph {\bibinfo {booktitle} {Proceedings of the 32nd International Conference on Machine Learning - Volume 37}}}\ (\bibinfo {year} {2015})\ pp.\ \bibinfo {pages} {2256--2265}\BibitemShut {NoStop}%
\bibitem [{\citenamefont {Song}\ and\ \citenamefont {Ermon}(2019)}]{song_generative_2019}%
  \BibitemOpen
  \bibfield  {author} {\bibinfo {author} {\bibfnamefont {Y.}~\bibnamefont {Song}}\ and\ \bibinfo {author} {\bibfnamefont {S.}~\bibnamefont {Ermon}},\ }\bibfield  {title} {\bibinfo {title} {Generative modeling by estimating gradients of the data distribution},\ }in\ \href {https://dl.acm.org/doi/10.5555/3454287.3455354} {\emph {\bibinfo {booktitle} {Proceedings of the 33rd International Conference on Neural Information Processing Systems}}}\ (\bibinfo {year} {2019})\ pp.\ \bibinfo {pages} {11918--11930}\BibitemShut {NoStop}%
\bibitem [{\citenamefont {Ho}\ \emph {et~al.}(2020)\citenamefont {Ho}, \citenamefont {Jain},\ and\ \citenamefont {Abbeel}}]{ho_denoising_2020}%
  \BibitemOpen
  \bibfield  {author} {\bibinfo {author} {\bibfnamefont {J.}~\bibnamefont {Ho}}, \bibinfo {author} {\bibfnamefont {A.}~\bibnamefont {Jain}},\ and\ \bibinfo {author} {\bibfnamefont {P.}~\bibnamefont {Abbeel}},\ }\bibfield  {title} {\bibinfo {title} {Denoising diffusion probabilistic models},\ }in\ \href {https://dl.acm.org/doi/abs/10.5555/3495724.3496298} {\emph {\bibinfo {booktitle} {Proceedings of the 34th International Conference on Neural Information Processing Systems}}}\ (\bibinfo {year} {2020})\ pp.\ \bibinfo {pages} {6840--6851}\BibitemShut {NoStop}%
\bibitem [{\citenamefont {Song}\ \emph {et~al.}(2021{\natexlab{a}})\citenamefont {Song}, \citenamefont {Meng},\ and\ \citenamefont {Ermon}}]{song_ddim_2021}%
  \BibitemOpen
  \bibfield  {author} {\bibinfo {author} {\bibfnamefont {J.}~\bibnamefont {Song}}, \bibinfo {author} {\bibfnamefont {C.}~\bibnamefont {Meng}},\ and\ \bibinfo {author} {\bibfnamefont {S.}~\bibnamefont {Ermon}},\ }\bibfield  {title} {\bibinfo {title} {Denoising diffusion implicit models},\ }in\ \href {https://openreview.net/forum?id=St1giarCHLP} {\emph {\bibinfo {booktitle} {International Conference on Learning Representations}}}\ (\bibinfo {year} {2021})\BibitemShut {NoStop}%
\bibitem [{\citenamefont {Song}\ \emph {et~al.}(2021{\natexlab{b}})\citenamefont {Song}, \citenamefont {Sohl-Dickstein}, \citenamefont {Kingma}, \citenamefont {Kumar}, \citenamefont {Ermon},\ and\ \citenamefont {Poole}}]{song_sde_2021}%
  \BibitemOpen
  \bibfield  {author} {\bibinfo {author} {\bibfnamefont {Y.}~\bibnamefont {Song}}, \bibinfo {author} {\bibfnamefont {J.}~\bibnamefont {Sohl-Dickstein}}, \bibinfo {author} {\bibfnamefont {D.~P.}\ \bibnamefont {Kingma}}, \bibinfo {author} {\bibfnamefont {A.}~\bibnamefont {Kumar}}, \bibinfo {author} {\bibfnamefont {S.}~\bibnamefont {Ermon}},\ and\ \bibinfo {author} {\bibfnamefont {B.}~\bibnamefont {Poole}},\ }\bibfield  {title} {\bibinfo {title} {Score-based generative modeling through stochastic differential equations},\ }in\ \href {https://openreview.net/forum?id=PxTIG12RRHS} {\emph {\bibinfo {booktitle} {International Conference on Learning Representations}}}\ (\bibinfo {year} {2021})\BibitemShut {NoStop}%
\bibitem [{\citenamefont {Lipman}\ \emph {et~al.}(2023)\citenamefont {Lipman}, \citenamefont {Chen}, \citenamefont {Ben-Hamu}, \citenamefont {Nickel},\ and\ \citenamefont {Le}}]{lipman2022flow}%
  \BibitemOpen
  \bibfield  {author} {\bibinfo {author} {\bibfnamefont {Y.}~\bibnamefont {Lipman}}, \bibinfo {author} {\bibfnamefont {R.~T.}\ \bibnamefont {Chen}}, \bibinfo {author} {\bibfnamefont {H.}~\bibnamefont {Ben-Hamu}}, \bibinfo {author} {\bibfnamefont {M.}~\bibnamefont {Nickel}},\ and\ \bibinfo {author} {\bibfnamefont {M.}~\bibnamefont {Le}},\ }\bibfield  {title} {\bibinfo {title} {Flow matching for generative modeling},\ }in\ \href {https://openreview.net/forum?id=PqvMRDCJT9t} {\emph {\bibinfo {booktitle} {International Conference on Learning Representations}}}\ (\bibinfo {year} {2023})\BibitemShut {NoStop}%
\bibitem [{\citenamefont {{Midjourney, Inc.}}(2022)}]{midjourney}%
  \BibitemOpen
  \bibfield  {author} {\bibinfo {author} {\bibnamefont {{Midjourney, Inc.}}},\ }\href {https://www.midjourney.com} {\bibinfo {title} {{Midjourney}}} (\bibinfo {year} {2022})\BibitemShut {NoStop}%
\bibitem [{\citenamefont {{Stability AI}}(2022)}]{stable2022}%
  \BibitemOpen
  \bibfield  {author} {\bibinfo {author} {\bibnamefont {{Stability AI}}},\ }\href {https://huggingface.co/CompVis/stable-diffusion} {\bibinfo {title} {{Stable Diffusion}}} (\bibinfo {year} {2022})\BibitemShut {NoStop}%
\bibitem [{\citenamefont {OpenAI}(2023)}]{openai_dalle}%
  \BibitemOpen
  \bibfield  {author} {\bibinfo {author} {\bibnamefont {OpenAI}},\ }\href {https://openai.com/dall-e} {\bibinfo {title} {{DALL·E 3}}} (\bibinfo {year} {2023})\BibitemShut {NoStop}%
\bibitem [{\citenamefont {OpenAI}(2024)}]{sora_openai}%
  \BibitemOpen
  \bibfield  {author} {\bibinfo {author} {\bibnamefont {OpenAI}},\ }\href {https://sora.com/} {\bibinfo {title} {Sora}} (\bibinfo {year} {2024})\BibitemShut {NoStop}%
\bibitem [{\citenamefont {{Google DeepMind}}(2025)}]{imagen_google_web}%
  \BibitemOpen
  \bibfield  {author} {\bibinfo {author} {\bibnamefont {{Google DeepMind}}},\ }\href {https://deepmind.google/models/imagen/} {\bibinfo {title} {Imagen 4}} (\bibinfo {year} {2025})\BibitemShut {NoStop}%
\bibitem [{\citenamefont {Hyv{\"a}rinen}(2005)}]{hyvarinen_score_2005}%
  \BibitemOpen
  \bibfield  {author} {\bibinfo {author} {\bibfnamefont {A.}~\bibnamefont {Hyv{\"a}rinen}},\ }\bibfield  {title} {\bibinfo {title} {Estimation of non-normalized statistical models by score matching},\ }\href {https://jmlr.org/papers/v6/hyvarinen05a.html} {\bibfield  {journal} {\bibinfo  {journal} {Journal of Machine Learning Research}\ }\textbf {\bibinfo {volume} {6}},\ \bibinfo {pages} {695} (\bibinfo {year} {2005})}\BibitemShut {NoStop}%
\bibitem [{\citenamefont {Wang}\ \emph {et~al.}(2023)\citenamefont {Wang}, \citenamefont {Jiang}, \citenamefont {Zheng}, \citenamefont {Wang}, \citenamefont {He}, \citenamefont {Wang}, \citenamefont {Chen}, \citenamefont {Zhou} \emph {et~al.}}]{wang2023patch}%
  \BibitemOpen
  \bibfield  {author} {\bibinfo {author} {\bibfnamefont {Z.}~\bibnamefont {Wang}}, \bibinfo {author} {\bibfnamefont {Y.}~\bibnamefont {Jiang}}, \bibinfo {author} {\bibfnamefont {H.}~\bibnamefont {Zheng}}, \bibinfo {author} {\bibfnamefont {P.}~\bibnamefont {Wang}}, \bibinfo {author} {\bibfnamefont {P.}~\bibnamefont {He}}, \bibinfo {author} {\bibfnamefont {Z.}~\bibnamefont {Wang}}, \bibinfo {author} {\bibfnamefont {W.}~\bibnamefont {Chen}}, \bibinfo {author} {\bibfnamefont {M.}~\bibnamefont {Zhou}}, \emph {et~al.},\ }\bibfield  {title} {\bibinfo {title} {Patch diffusion: Faster and more data-efficient training of diffusion models},\ }in\ \href {https://dl.acm.org/doi/abs/10.5555/3666122.3669280} {\emph {\bibinfo {booktitle} {Proceedings of the 37th International Conference on Neural Information Processing Systems}}}\ (\bibinfo {year} {2023})\ pp.\ \bibinfo {pages} {72137--72154}\BibitemShut {NoStop}%
\bibitem [{\citenamefont {Ding}\ \emph {et~al.}(2023)\citenamefont {Ding}, \citenamefont {Zhang}, \citenamefont {Wu},\ and\ \citenamefont {Tu}}]{ding2023patched}%
  \BibitemOpen
  \bibfield  {author} {\bibinfo {author} {\bibfnamefont {Z.}~\bibnamefont {Ding}}, \bibinfo {author} {\bibfnamefont {M.}~\bibnamefont {Zhang}}, \bibinfo {author} {\bibfnamefont {J.}~\bibnamefont {Wu}},\ and\ \bibinfo {author} {\bibfnamefont {Z.}~\bibnamefont {Tu}},\ }\bibfield  {title} {\bibinfo {title} {Patched denoising diffusion models for high-resolution image synthesis},\ }in\ \href {https://openreview.net/forum?id=TgSRPRz8cI} {\emph {\bibinfo {booktitle} {International Conference on Learning Representations}}}\ (\bibinfo {year} {2023})\BibitemShut {NoStop}%
\bibitem [{\citenamefont {Kamb}\ and\ \citenamefont {Ganguli}(2024)}]{kamb2024analytic}%
  \BibitemOpen
  \bibfield  {author} {\bibinfo {author} {\bibfnamefont {M.}~\bibnamefont {Kamb}}\ and\ \bibinfo {author} {\bibfnamefont {S.}~\bibnamefont {Ganguli}},\ }\bibfield  {title} {\bibinfo {title} {An analytic theory of creativity in convolutional diffusion models},\ }\href {https://arxiv.org/abs/2412.20292} {\bibfield  {journal} {\bibinfo  {journal} {arXiv:2412.20292 [cs.LG]}\ } (\bibinfo {year} {2024})}\BibitemShut {NoStop}%
\bibitem [{\citenamefont {Niedoba}\ \emph {et~al.}(2024)\citenamefont {Niedoba}, \citenamefont {Zwartsenberg}, \citenamefont {Murphy},\ and\ \citenamefont {Wood}}]{niedoba2024towards}%
  \BibitemOpen
  \bibfield  {author} {\bibinfo {author} {\bibfnamefont {M.}~\bibnamefont {Niedoba}}, \bibinfo {author} {\bibfnamefont {B.}~\bibnamefont {Zwartsenberg}}, \bibinfo {author} {\bibfnamefont {K.}~\bibnamefont {Murphy}},\ and\ \bibinfo {author} {\bibfnamefont {F.}~\bibnamefont {Wood}},\ }\bibfield  {title} {\bibinfo {title} {Towards a mechanistic explanation of diffusion model generalization},\ }\href {https://arxiv.org/abs/2411.19339} {\bibfield  {journal} {\bibinfo  {journal} {arXiv:2411.19339 [cs.LG]}\ } (\bibinfo {year} {2024})}\BibitemShut {NoStop}%
\bibitem [{\citenamefont {Chen}\ \emph {et~al.}(2010)\citenamefont {Chen}, \citenamefont {Gu},\ and\ \citenamefont {Wen}}]{chen2010local}%
  \BibitemOpen
  \bibfield  {author} {\bibinfo {author} {\bibfnamefont {X.}~\bibnamefont {Chen}}, \bibinfo {author} {\bibfnamefont {Z.-C.}\ \bibnamefont {Gu}},\ and\ \bibinfo {author} {\bibfnamefont {X.-G.}\ \bibnamefont {Wen}},\ }\bibfield  {title} {\bibinfo {title} {Local unitary transformation, long-range quantum entanglement, wave function renormalization, and topological order},\ }\href {https://doi.org/10.1103/PhysRevB.82.155138} {\bibfield  {journal} {\bibinfo  {journal} {Physical Review B}\ }\textbf {\bibinfo {volume} {82}},\ \bibinfo {pages} {155138} (\bibinfo {year} {2010})}\BibitemShut {NoStop}%
\bibitem [{\citenamefont {Coser}\ and\ \citenamefont {Pérez-García}(2019)}]{coser_classification_2019}%
  \BibitemOpen
  \bibfield  {author} {\bibinfo {author} {\bibfnamefont {A.}~\bibnamefont {Coser}}\ and\ \bibinfo {author} {\bibfnamefont {D.}~\bibnamefont {Pérez-García}},\ }\bibfield  {title} {\bibinfo {title} {Classification of phases for mixed states via fast dissipative evolution},\ }\href {https://doi.org/10.22331/q-2019-08-12-174} {\bibfield  {journal} {\bibinfo  {journal} {Quantum}\ }\textbf {\bibinfo {volume} {3}},\ \bibinfo {pages} {174} (\bibinfo {year} {2019})}\BibitemShut {NoStop}%
\bibitem [{\citenamefont {Sang}\ \emph {et~al.}(2024)\citenamefont {Sang}, \citenamefont {Zou},\ and\ \citenamefont {Hsieh}}]{sang2024mixed}%
  \BibitemOpen
  \bibfield  {author} {\bibinfo {author} {\bibfnamefont {S.}~\bibnamefont {Sang}}, \bibinfo {author} {\bibfnamefont {Y.}~\bibnamefont {Zou}},\ and\ \bibinfo {author} {\bibfnamefont {T.~H.}\ \bibnamefont {Hsieh}},\ }\bibfield  {title} {\bibinfo {title} {Mixed-state quantum phases: Renormalization and quantum error correction},\ }\href {https://doi.org/10.1103/physrevx.14.031044} {\bibfield  {journal} {\bibinfo  {journal} {Physical Review X}\ }\textbf {\bibinfo {volume} {14}},\ \bibinfo {pages} {031044} (\bibinfo {year} {2024})}\BibitemShut {NoStop}%
\bibitem [{\citenamefont {Sang}\ and\ \citenamefont {Hsieh}(2025)}]{sang_statbility_2025}%
  \BibitemOpen
  \bibfield  {author} {\bibinfo {author} {\bibfnamefont {S.}~\bibnamefont {Sang}}\ and\ \bibinfo {author} {\bibfnamefont {T.~H.}\ \bibnamefont {Hsieh}},\ }\bibfield  {title} {\bibinfo {title} {Stability of mixed-state quantum phases via finite markov length},\ }\href {https://doi.org/10.1103/physrevlett.134.070403} {\bibfield  {journal} {\bibinfo  {journal} {Physical Review Letters}\ }\textbf {\bibinfo {volume} {134}},\ \bibinfo {pages} {070403} (\bibinfo {year} {2025})}\BibitemShut {NoStop}%
\bibitem [{\citenamefont {Biroli}\ \emph {et~al.}(2024)\citenamefont {Biroli}, \citenamefont {Bonnaire}, \citenamefont {de~Bortoli},\ and\ \citenamefont {Mézard}}]{biroli_dynamical_2024}%
  \BibitemOpen
  \bibfield  {author} {\bibinfo {author} {\bibfnamefont {G.}~\bibnamefont {Biroli}}, \bibinfo {author} {\bibfnamefont {T.}~\bibnamefont {Bonnaire}}, \bibinfo {author} {\bibfnamefont {V.}~\bibnamefont {de~Bortoli}},\ and\ \bibinfo {author} {\bibfnamefont {M.}~\bibnamefont {Mézard}},\ }\bibfield  {title} {\bibinfo {title} {Dynamical regimes of diffusion models},\ }\href {https://doi.org/10.1038/s41467-024-54281-3} {\bibfield  {journal} {\bibinfo  {journal} {Nature Communications}\ }\textbf {\bibinfo {volume} {15}},\ \bibinfo {pages} {9957} (\bibinfo {year} {2024})}\BibitemShut {NoStop}%
\bibitem [{\citenamefont {Raya}\ and\ \citenamefont {Ambrogioni}(2023)}]{raya_spontaneous_2023}%
  \BibitemOpen
  \bibfield  {author} {\bibinfo {author} {\bibfnamefont {G.}~\bibnamefont {Raya}}\ and\ \bibinfo {author} {\bibfnamefont {L.}~\bibnamefont {Ambrogioni}},\ }\bibfield  {title} {\bibinfo {title} {Spontaneous symmetry breaking in generative diffusion models},\ }in\ \href {https://openreview.net/forum?id=lxGFGMMSVl} {\emph {\bibinfo {booktitle} {Thirty-seventh Conference on Neural Information Processing Systems}}}\ (\bibinfo {year} {2023})\BibitemShut {NoStop}%
\bibitem [{\citenamefont {Li}\ and\ \citenamefont {Chen}(2024)}]{li_critical_2024}%
  \BibitemOpen
  \bibfield  {author} {\bibinfo {author} {\bibfnamefont {M.}~\bibnamefont {Li}}\ and\ \bibinfo {author} {\bibfnamefont {S.}~\bibnamefont {Chen}},\ }\bibfield  {title} {\bibinfo {title} {Critical windows: non-asymptotic theory for feature emergence in diffusion models},\ }\href {https://arxiv.org/abs/2403.01633} {\bibfield  {journal} {\bibinfo  {journal} {arXiv:2403.01633 [cs.LG]}\ } (\bibinfo {year} {2024})}\BibitemShut {NoStop}%
\bibitem [{\citenamefont {Sclocchi}\ \emph {et~al.}(2024{\natexlab{a}})\citenamefont {Sclocchi}, \citenamefont {Favero},\ and\ \citenamefont {Wyart}}]{sclocchi_phase_2024}%
  \BibitemOpen
  \bibfield  {author} {\bibinfo {author} {\bibfnamefont {A.}~\bibnamefont {Sclocchi}}, \bibinfo {author} {\bibfnamefont {A.}~\bibnamefont {Favero}},\ and\ \bibinfo {author} {\bibfnamefont {M.}~\bibnamefont {Wyart}},\ }\bibfield  {title} {\bibinfo {title} {A phase transition in diffusion models reveals the hierarchical nature of data},\ }\href {https://arxiv.org/abs/2402.16991} {\bibfield  {journal} {\bibinfo  {journal} {arXiv:2402.16991 [stat.ML]}\ } (\bibinfo {year} {2024}{\natexlab{a}})}\BibitemShut {NoStop}%
\bibitem [{\citenamefont {Li}\ \emph {et~al.}(2025)\citenamefont {Li}, \citenamefont {Karan},\ and\ \citenamefont {Chen}}]{li_blink_2025}%
  \BibitemOpen
  \bibfield  {author} {\bibinfo {author} {\bibfnamefont {M.}~\bibnamefont {Li}}, \bibinfo {author} {\bibfnamefont {A.}~\bibnamefont {Karan}},\ and\ \bibinfo {author} {\bibfnamefont {S.}~\bibnamefont {Chen}},\ }\bibfield  {title} {\bibinfo {title} {Blink of an eye: a simple theory for feature localization in generative models},\ }\href {https://arxiv.org/abs/2502.00921} {\bibfield  {journal} {\bibinfo  {journal} {arXiv:2502.00921 [cs.LG]}\ } (\bibinfo {year} {2025})}\BibitemShut {NoStop}%
\bibitem [{\citenamefont {Bachtis}\ \emph {et~al.}(2025)\citenamefont {Bachtis}, \citenamefont {Biroli}, \citenamefont {Decelle},\ and\ \citenamefont {Seoane}}]{bachtis_2025_cascade}%
  \BibitemOpen
  \bibfield  {author} {\bibinfo {author} {\bibfnamefont {D.}~\bibnamefont {Bachtis}}, \bibinfo {author} {\bibfnamefont {G.}~\bibnamefont {Biroli}}, \bibinfo {author} {\bibfnamefont {A.}~\bibnamefont {Decelle}},\ and\ \bibinfo {author} {\bibfnamefont {B.}~\bibnamefont {Seoane}},\ }\bibfield  {title} {\bibinfo {title} {Cascade of phase transitions in the training of energy-based models},\ }\href {https://doi.org/10.1088/1742-5468/adec64} {\bibfield  {journal} {\bibinfo  {journal} {Journal of Statistical Mechanics: Theory and Experiment}\ }\textbf {\bibinfo {volume} {2025}},\ \bibinfo {pages} {074004} (\bibinfo {year} {2025})}\BibitemShut {NoStop}%
\bibitem [{\citenamefont {Sang}\ \emph {et~al.}(2025)\citenamefont {Sang}, \citenamefont {Lessa}, \citenamefont {Mong}, \citenamefont {Grover}, \citenamefont {Wang},\ and\ \citenamefont {Hsieh}}]{sang_2025_mixed}%
  \BibitemOpen
  \bibfield  {author} {\bibinfo {author} {\bibfnamefont {S.}~\bibnamefont {Sang}}, \bibinfo {author} {\bibfnamefont {L.~A.}\ \bibnamefont {Lessa}}, \bibinfo {author} {\bibfnamefont {R.~S.~K.}\ \bibnamefont {Mong}}, \bibinfo {author} {\bibfnamefont {T.}~\bibnamefont {Grover}}, \bibinfo {author} {\bibfnamefont {C.}~\bibnamefont {Wang}},\ and\ \bibinfo {author} {\bibfnamefont {T.~H.}\ \bibnamefont {Hsieh}},\ }\bibfield  {title} {\bibinfo {title} {Mixed-state phases from local reversibility},\ }\href {https://arxiv.org/abs/2507.02292} {\bibfield  {journal} {\bibinfo  {journal} {arXiv:2507.02292 [quant-ph]}\ } (\bibinfo {year} {2025})}\BibitemShut {NoStop}%
\bibitem [{\citenamefont {Wen}(1989)}]{wen_1989_vacuum}%
  \BibitemOpen
  \bibfield  {author} {\bibinfo {author} {\bibfnamefont {X.-G.}\ \bibnamefont {Wen}},\ }\bibfield  {title} {\bibinfo {title} {Vacuum degeneracy of chiral spin states in compactified space},\ }\href {https://doi.org/10.1103/physrevb.40.7387} {\bibfield  {journal} {\bibinfo  {journal} {Physical Review B}\ }\textbf {\bibinfo {volume} {40}},\ \bibinfo {pages} {7387–7390} (\bibinfo {year} {1989})}\BibitemShut {NoStop}%
\bibitem [{\citenamefont {Wen}(1990)}]{wen_1990_topological}%
  \BibitemOpen
  \bibfield  {author} {\bibinfo {author} {\bibfnamefont {X.-G.}\ \bibnamefont {Wen}},\ }\bibfield  {title} {\bibinfo {title} {Topological orders in rigid states},\ }\href {https://doi.org/10.1142/s0217979290000139} {\bibfield  {journal} {\bibinfo  {journal} {International Journal of Modern Physics B}\ }\textbf {\bibinfo {volume} {04}},\ \bibinfo {pages} {239–271} (\bibinfo {year} {1990})}\BibitemShut {NoStop}%
\bibitem [{\citenamefont {Wen}(2007)}]{wen_2007_quantum}%
  \BibitemOpen
  \bibfield  {author} {\bibinfo {author} {\bibfnamefont {X.-G.}\ \bibnamefont {Wen}},\ }\href {https://doi.org/10.1093/acprof:oso/9780199227259.001.0001} {\emph {\bibinfo {title} {Quantum Field Theory of Many-Body Systems}}}\ (\bibinfo  {publisher} {Oxford University Press},\ \bibinfo {year} {2007})\BibitemShut {NoStop}%
\bibitem [{\citenamefont {Chen}\ \emph {et~al.}(2011)\citenamefont {Chen}, \citenamefont {Liu},\ and\ \citenamefont {Wen}}]{chen_2011_two}%
  \BibitemOpen
  \bibfield  {author} {\bibinfo {author} {\bibfnamefont {X.}~\bibnamefont {Chen}}, \bibinfo {author} {\bibfnamefont {Z.-X.}\ \bibnamefont {Liu}},\ and\ \bibinfo {author} {\bibfnamefont {X.-G.}\ \bibnamefont {Wen}},\ }\bibfield  {title} {\bibinfo {title} {Two-dimensional symmetry-protected topological orders and their protected gapless edge excitations},\ }\href {https://doi.org/10.1103/physrevb.84.235141} {\bibfield  {journal} {\bibinfo  {journal} {Physical Review B}\ }\textbf {\bibinfo {volume} {84}},\ \bibinfo {pages} {235141} (\bibinfo {year} {2011})}\BibitemShut {NoStop}%
\bibitem [{\citenamefont {Chen}\ \emph {et~al.}(2013)\citenamefont {Chen}, \citenamefont {Gu}, \citenamefont {Liu},\ and\ \citenamefont {Wen}}]{chen_2013_symmetry}%
  \BibitemOpen
  \bibfield  {author} {\bibinfo {author} {\bibfnamefont {X.}~\bibnamefont {Chen}}, \bibinfo {author} {\bibfnamefont {Z.-C.}\ \bibnamefont {Gu}}, \bibinfo {author} {\bibfnamefont {Z.-X.}\ \bibnamefont {Liu}},\ and\ \bibinfo {author} {\bibfnamefont {X.-G.}\ \bibnamefont {Wen}},\ }\bibfield  {title} {\bibinfo {title} {Symmetry protected topological orders and the group cohomology of their symmetry group},\ }\href {https://doi.org/10.1103/physrevb.87.155114} {\bibfield  {journal} {\bibinfo  {journal} {Physical Review B}\ }\textbf {\bibinfo {volume} {87}},\ \bibinfo {pages} {155114} (\bibinfo {year} {2013})}\BibitemShut {NoStop}%
\bibitem [{\citenamefont {Haah}(2011)}]{haah_2011_local}%
  \BibitemOpen
  \bibfield  {author} {\bibinfo {author} {\bibfnamefont {J.}~\bibnamefont {Haah}},\ }\bibfield  {title} {\bibinfo {title} {Local stabilizer codes in three dimensions without string logical operators},\ }\href {https://doi.org/10.1103/physreva.83.042330} {\bibfield  {journal} {\bibinfo  {journal} {Physical Review A}\ }\textbf {\bibinfo {volume} {83}},\ \bibinfo {pages} {042330} (\bibinfo {year} {2011})}\BibitemShut {NoStop}%
\bibitem [{\citenamefont {Haah}(2016)}]{haah_2016_invariant}%
  \BibitemOpen
  \bibfield  {author} {\bibinfo {author} {\bibfnamefont {J.}~\bibnamefont {Haah}},\ }\bibfield  {title} {\bibinfo {title} {An invariant of topologically ordered states under local unitary transformations},\ }\href {https://doi.org/10.1007/s00220-016-2594-y} {\bibfield  {journal} {\bibinfo  {journal} {Communications in Mathematical Physics}\ }\textbf {\bibinfo {volume} {342}},\ \bibinfo {pages} {771–801} (\bibinfo {year} {2016})}\BibitemShut {NoStop}%
\bibitem [{\citenamefont {Zeng}\ \emph {et~al.}(2019)\citenamefont {Zeng}, \citenamefont {Chen}, \citenamefont {Zhou},\ and\ \citenamefont {Wen}}]{zeng_2019_quantum}%
  \BibitemOpen
  \bibfield  {author} {\bibinfo {author} {\bibfnamefont {B.}~\bibnamefont {Zeng}}, \bibinfo {author} {\bibfnamefont {X.}~\bibnamefont {Chen}}, \bibinfo {author} {\bibfnamefont {D.-L.}\ \bibnamefont {Zhou}},\ and\ \bibinfo {author} {\bibfnamefont {X.-G.}\ \bibnamefont {Wen}},\ }\href {https://doi.org/10.1007/978-1-4939-9084-9} {\emph {\bibinfo {title} {Quantum Information Meets Quantum Matter: From Quantum Entanglement to Topological Phases of Many-Body Systems}}}\ (\bibinfo  {publisher} {Springer New York},\ \bibinfo {year} {2019})\BibitemShut {NoStop}%
\bibitem [{\citenamefont {LeCun}\ \emph {et~al.}(1998)\citenamefont {LeCun}, \citenamefont {Cortes},\ and\ \citenamefont {Burges}}]{lecun1998mnist}%
  \BibitemOpen
  \bibfield  {author} {\bibinfo {author} {\bibfnamefont {Y.}~\bibnamefont {LeCun}}, \bibinfo {author} {\bibfnamefont {C.}~\bibnamefont {Cortes}},\ and\ \bibinfo {author} {\bibfnamefont {C.~J.}\ \bibnamefont {Burges}},\ }\href@noop {} {\bibinfo {title} {{MNIST} handwritten digit database}},\ \bibinfo {howpublished} {\url{http://yann.lecun.com/exdb/mnist/}} (\bibinfo {year} {1998})\BibitemShut {NoStop}%
\bibitem [{\citenamefont {Xiao}\ \emph {et~al.}(2017)\citenamefont {Xiao}, \citenamefont {Rasul},\ and\ \citenamefont {Vollgraf}}]{xiao2017fashionmnist}%
  \BibitemOpen
  \bibfield  {author} {\bibinfo {author} {\bibfnamefont {H.}~\bibnamefont {Xiao}}, \bibinfo {author} {\bibfnamefont {K.}~\bibnamefont {Rasul}},\ and\ \bibinfo {author} {\bibfnamefont {R.}~\bibnamefont {Vollgraf}},\ }\bibfield  {title} {\bibinfo {title} {Fashion-mnist: a novel image dataset for benchmarking machine learning algorithms},\ }\href {https://arxiv.org/abs/1708.07747} {\bibfield  {journal} {\bibinfo  {journal} {arXiv:1708.07747 [cs.LG]}\ } (\bibinfo {year} {2017})}\BibitemShut {NoStop}%
\bibitem [{\citenamefont {Anderson}(1982)}]{anderson_reverse_1982}%
  \BibitemOpen
  \bibfield  {author} {\bibinfo {author} {\bibfnamefont {B.~D.}\ \bibnamefont {Anderson}},\ }\bibfield  {title} {\bibinfo {title} {Reverse-time diffusion equation models},\ }\href {https://doi.org/10.1016/0304-4149(82)90051-5} {\bibfield  {journal} {\bibinfo  {journal} {Stochastic Processes and their Applications}\ }\textbf {\bibinfo {volume} {12}},\ \bibinfo {pages} {313–326} (\bibinfo {year} {1982})}\BibitemShut {NoStop}%
\bibitem [{\citenamefont {Li}\ and\ \citenamefont {Winter}(2018)}]{li_squashed_2018}%
  \BibitemOpen
  \bibfield  {author} {\bibinfo {author} {\bibfnamefont {K.}~\bibnamefont {Li}}\ and\ \bibinfo {author} {\bibfnamefont {A.}~\bibnamefont {Winter}},\ }\bibfield  {title} {\bibinfo {title} {Squashed entanglement, $\mathbf{k}$-extendibility, quantum markov chains, and recovery maps},\ }\href {https://doi.org/10.1007/s10701-018-0143-6} {\bibfield  {journal} {\bibinfo  {journal} {Foundations of Physics}\ }\textbf {\bibinfo {volume} {48}},\ \bibinfo {pages} {910–924} (\bibinfo {year} {2018})}\BibitemShut {NoStop}%
\bibitem [{\citenamefont {Fawzi}\ and\ \citenamefont {Renner}(2015)}]{fawzi_quantum_2015}%
  \BibitemOpen
  \bibfield  {author} {\bibinfo {author} {\bibfnamefont {O.}~\bibnamefont {Fawzi}}\ and\ \bibinfo {author} {\bibfnamefont {R.}~\bibnamefont {Renner}},\ }\bibfield  {title} {\bibinfo {title} {Quantum conditional mutual information and approximate markov chains},\ }\href {https://doi.org/10.1007/s00220-015-2466-x} {\bibfield  {journal} {\bibinfo  {journal} {Communications in Mathematical Physics}\ }\textbf {\bibinfo {volume} {340}},\ \bibinfo {pages} {575–611} (\bibinfo {year} {2015})}\BibitemShut {NoStop}%
\bibitem [{\citenamefont {Zhang}\ and\ \citenamefont {Gopalakrishnan}(2025)}]{zhang2025conditional}%
  \BibitemOpen
  \bibfield  {author} {\bibinfo {author} {\bibfnamefont {Y.}~\bibnamefont {Zhang}}\ and\ \bibinfo {author} {\bibfnamefont {S.}~\bibnamefont {Gopalakrishnan}},\ }\bibfield  {title} {\bibinfo {title} {Conditional mutual information and information-theoretic phases of decohered gibbs states},\ }\href {https://arxiv.org/abs/2502.13210} {\bibfield  {journal} {\bibinfo  {journal} {arXiv:2502.13210 [quant-ph]}\ } (\bibinfo {year} {2025})}\BibitemShut {NoStop}%
\bibitem [{\citenamefont {Ma}\ \emph {et~al.}(2025)\citenamefont {Ma}, \citenamefont {Khemani},\ and\ \citenamefont {Sang}}]{ma2025circuit}%
  \BibitemOpen
  \bibfield  {author} {\bibinfo {author} {\bibfnamefont {R.}~\bibnamefont {Ma}}, \bibinfo {author} {\bibfnamefont {V.}~\bibnamefont {Khemani}},\ and\ \bibinfo {author} {\bibfnamefont {S.}~\bibnamefont {Sang}},\ }\bibfield  {title} {\bibinfo {title} {Circuit-based characterization of finite-temperature quantum phases and self-correcting quantum memory},\ }\href {https://arxiv.org/abs/2509.15204} {\bibfield  {journal} {\bibinfo  {journal} {arXiv:2509.15204 [quant-ph]}\ } (\bibinfo {year} {2025})}\BibitemShut {NoStop}%
\bibitem [{\citenamefont {Rosenblatt}(1956)}]{rosenblatt1956remarks}%
  \BibitemOpen
  \bibfield  {author} {\bibinfo {author} {\bibfnamefont {M.}~\bibnamefont {Rosenblatt}},\ }\bibfield  {title} {\bibinfo {title} {Remarks on some nonparametric estimates of a density function},\ }\href {https://doi.org/10.1214/aoms/1177728190} {\bibfield  {journal} {\bibinfo  {journal} {The Annals of Mathematical Statistics}\ }\textbf {\bibinfo {volume} {27}},\ \bibinfo {pages} {832} (\bibinfo {year} {1956})}\BibitemShut {NoStop}%
\bibitem [{\citenamefont {Parzen}(1962)}]{parzen1962estimation}%
  \BibitemOpen
  \bibfield  {author} {\bibinfo {author} {\bibfnamefont {E.}~\bibnamefont {Parzen}},\ }\bibfield  {title} {\bibinfo {title} {On estimation of a probability density function and mode},\ }\href {https://doi.org/10.1214/aoms/1177704472} {\bibfield  {journal} {\bibinfo  {journal} {The Annals of Mathematical Statistics}\ }\textbf {\bibinfo {volume} {33}},\ \bibinfo {pages} {1065} (\bibinfo {year} {1962})}\BibitemShut {NoStop}%
\bibitem [{\citenamefont {Heitz}\ \emph {et~al.}()\citenamefont {Heitz}, \citenamefont {Belcour},\ and\ \citenamefont {Chambon}}]{Heitz2023alphaBlending}%
  \BibitemOpen
  \bibfield  {author} {\bibinfo {author} {\bibfnamefont {E.}~\bibnamefont {Heitz}}, \bibinfo {author} {\bibfnamefont {L.}~\bibnamefont {Belcour}},\ and\ \bibinfo {author} {\bibfnamefont {T.}~\bibnamefont {Chambon}},\ }\bibfield  {title} {\bibinfo {title} {Iterative {$\alpha$}-(de)blending: a minimalist deterministic diffusion model},\ }in\ \href {https://doi.org/10.1145/3588432.3591540} {\emph {\bibinfo {booktitle} {ACM SIGGRAPH 2023 Conference Proceedings}}},\ pp.\ \bibinfo {pages} {1--8}\BibitemShut {NoStop}%
\bibitem [{\citenamefont {Belghazi}\ \emph {et~al.}(2018)\citenamefont {Belghazi}, \citenamefont {Baratin}, \citenamefont {Rajeshwar}, \citenamefont {Ozair}, \citenamefont {Bengio}, \citenamefont {Courville},\ and\ \citenamefont {Hjelm}}]{belghazi_mutual_2015}%
  \BibitemOpen
  \bibfield  {author} {\bibinfo {author} {\bibfnamefont {M.~I.}\ \bibnamefont {Belghazi}}, \bibinfo {author} {\bibfnamefont {A.}~\bibnamefont {Baratin}}, \bibinfo {author} {\bibfnamefont {S.}~\bibnamefont {Rajeshwar}}, \bibinfo {author} {\bibfnamefont {S.}~\bibnamefont {Ozair}}, \bibinfo {author} {\bibfnamefont {Y.}~\bibnamefont {Bengio}}, \bibinfo {author} {\bibfnamefont {A.}~\bibnamefont {Courville}},\ and\ \bibinfo {author} {\bibfnamefont {D.}~\bibnamefont {Hjelm}},\ }\bibfield  {title} {\bibinfo {title} {Mutual information neural estimation},\ }in\ \href {https://proceedings.mlr.press/v80/belghazi18a.html} {\emph {\bibinfo {booktitle} {Proceedings of the 35th International Conference on Machine Learning}}},\ Vol.~\bibinfo {volume} {80}\ (\bibinfo {year} {2018})\ pp.\ \bibinfo {pages} {531--540}\BibitemShut {NoStop}%
\bibitem [{\citenamefont {Ronneberger}\ \emph {et~al.}(2015)\citenamefont {Ronneberger}, \citenamefont {Fischer},\ and\ \citenamefont {Brox}}]{ronneberger2015unet}%
  \BibitemOpen
  \bibfield  {author} {\bibinfo {author} {\bibfnamefont {O.}~\bibnamefont {Ronneberger}}, \bibinfo {author} {\bibfnamefont {P.}~\bibnamefont {Fischer}},\ and\ \bibinfo {author} {\bibfnamefont {T.}~\bibnamefont {Brox}},\ }\bibfield  {title} {\bibinfo {title} {U-net: Convolutional networks for biomedical image segmentation},\ }in\ \href {https://doi.org/10.1007/978-3-319-24574-4_28} {\emph {\bibinfo {booktitle} {Medical Image Computing and Computer-Assisted Intervention (MICCAI)}}}\ (\bibinfo  {publisher} {Springer},\ \bibinfo {year} {2015})\ pp.\ \bibinfo {pages} {234--241}\BibitemShut {NoStop}%
\bibitem [{\citenamefont {Junge}\ \emph {et~al.}(2018)\citenamefont {Junge}, \citenamefont {Renner}, \citenamefont {Sutter}, \citenamefont {Wilde},\ and\ \citenamefont {Winter}}]{junge_universal_2018}%
  \BibitemOpen
  \bibfield  {author} {\bibinfo {author} {\bibfnamefont {M.}~\bibnamefont {Junge}}, \bibinfo {author} {\bibfnamefont {R.}~\bibnamefont {Renner}}, \bibinfo {author} {\bibfnamefont {D.}~\bibnamefont {Sutter}}, \bibinfo {author} {\bibfnamefont {M.~M.}\ \bibnamefont {Wilde}},\ and\ \bibinfo {author} {\bibfnamefont {A.}~\bibnamefont {Winter}},\ }\bibfield  {title} {\bibinfo {title} {Universal recovery maps and approximate sufficiency of quantum relative entropy},\ }\href {https://doi.org/10.1007/s00023-018-0716-0} {\bibfield  {journal} {\bibinfo  {journal} {Annales Henri Poincaré}\ }\textbf {\bibinfo {volume} {19}},\ \bibinfo {pages} {2955–2978} (\bibinfo {year} {2018})}\BibitemShut {NoStop}%
\bibitem [{\citenamefont {Sclocchi}\ \emph {et~al.}(2024{\natexlab{b}})\citenamefont {Sclocchi}, \citenamefont {Favero}, \citenamefont {Levi},\ and\ \citenamefont {Wyart}}]{sclocchi_probing_2024}%
  \BibitemOpen
  \bibfield  {author} {\bibinfo {author} {\bibfnamefont {A.}~\bibnamefont {Sclocchi}}, \bibinfo {author} {\bibfnamefont {A.}~\bibnamefont {Favero}}, \bibinfo {author} {\bibfnamefont {N.~I.}\ \bibnamefont {Levi}},\ and\ \bibinfo {author} {\bibfnamefont {M.}~\bibnamefont {Wyart}},\ }\bibfield  {title} {\bibinfo {title} {Probing the latent hierarchical structure of data via diffusion models},\ }\href {https://arxiv.org/abs/2410.13770} {\bibfield  {journal} {\bibinfo  {journal} {arXiv:2410.13770 [stat.ML]}\ } (\bibinfo {year} {2024}{\natexlab{b}})}\BibitemShut {NoStop}%
\bibitem [{\citenamefont {Landau}\ and\ \citenamefont {Liu}(2025)}]{landau_2025_learning}%
  \BibitemOpen
  \bibfield  {author} {\bibinfo {author} {\bibfnamefont {Z.}~\bibnamefont {Landau}}\ and\ \bibinfo {author} {\bibfnamefont {Y.}~\bibnamefont {Liu}},\ }\bibfield  {title} {\bibinfo {title} {Learning quantum states prepared by shallow circuits in polynomial time},\ }in\ \href {https://doi.org/10.1145/3717823.3718311} {\emph {\bibinfo {booktitle} {Proceedings of the 57th Annual ACM Symposium on Theory of Computing}}},\ \bibinfo {series and number} {STOC ’25}\ (\bibinfo  {publisher} {ACM},\ \bibinfo {year} {2025})\ p.\ \bibinfo {pages} {1828–1838}\BibitemShut {NoStop}%
\bibitem [{\citenamefont {Kim}\ \emph {et~al.}(2024)\citenamefont {Kim}, \citenamefont {Kim},\ and\ \citenamefont {Ranard}}]{kim_2024_learning}%
  \BibitemOpen
  \bibfield  {author} {\bibinfo {author} {\bibfnamefont {H.-S.}\ \bibnamefont {Kim}}, \bibinfo {author} {\bibfnamefont {I.~H.}\ \bibnamefont {Kim}},\ and\ \bibinfo {author} {\bibfnamefont {D.}~\bibnamefont {Ranard}},\ }\bibfield  {title} {\bibinfo {title} {Learning state preparation circuits for quantum phases of matter},\ }\href {https://arxiv.org/abs/2410.23544} {\bibfield  {journal} {\bibinfo  {journal} {arXiv:2410.23544 [quant-ph]}\ } (\bibinfo {year} {2024})}\BibitemShut {NoStop}%
\bibitem [{\citenamefont {Hu}\ \emph {et~al.}(2026)\citenamefont {Hu}, \citenamefont {Kokail}, \citenamefont {Kornjača}, \citenamefont {Lopes}, \citenamefont {Gong}, \citenamefont {Wang}, \citenamefont {Gao},\ and\ \citenamefont {Ostermann}}]{hu_2026_learning}%
  \BibitemOpen
  \bibfield  {author} {\bibinfo {author} {\bibfnamefont {F.}~\bibnamefont {Hu}}, \bibinfo {author} {\bibfnamefont {C.}~\bibnamefont {Kokail}}, \bibinfo {author} {\bibfnamefont {M.}~\bibnamefont {Kornjača}}, \bibinfo {author} {\bibfnamefont {P.~L.~S.}\ \bibnamefont {Lopes}}, \bibinfo {author} {\bibfnamefont {W.}~\bibnamefont {Gong}}, \bibinfo {author} {\bibfnamefont {S.-T.}\ \bibnamefont {Wang}}, \bibinfo {author} {\bibfnamefont {X.}~\bibnamefont {Gao}},\ and\ \bibinfo {author} {\bibfnamefont {S.}~\bibnamefont {Ostermann}},\ }\bibfield  {title} {\bibinfo {title} {Learning and generating mixed states prepared by shallow channel circuits},\ }\href {https://arxiv.org/abs/2604.01197} {\bibfield  {journal} {\bibinfo  {journal} {arXiv:2604.01197 [quant-ph]}\ } (\bibinfo {year} {2026})}\BibitemShut {NoStop}%
\bibitem [{\citenamefont {Lukoianov}\ \emph {et~al.}(2025)\citenamefont {Lukoianov}, \citenamefont {Yuan}, \citenamefont {Solomon},\ and\ \citenamefont {Sitzmann}}]{lukoianov2025locality}%
  \BibitemOpen
  \bibfield  {author} {\bibinfo {author} {\bibfnamefont {A.}~\bibnamefont {Lukoianov}}, \bibinfo {author} {\bibfnamefont {C.}~\bibnamefont {Yuan}}, \bibinfo {author} {\bibfnamefont {J.}~\bibnamefont {Solomon}},\ and\ \bibinfo {author} {\bibfnamefont {V.}~\bibnamefont {Sitzmann}},\ }\bibfield  {title} {\bibinfo {title} {Locality in image diffusion models emerges from data statistics},\ }\href {https://arxiv.org/abs/2509.09672} {\bibfield  {journal} {\bibinfo  {journal} {arXiv:2509.09672 [cs.CV]}\ } (\bibinfo {year} {2025})}\BibitemShut {NoStop}%
\bibitem [{\citenamefont {Farquhar}\ \emph {et~al.}(2024)\citenamefont {Farquhar}, \citenamefont {Kossen}, \citenamefont {Kuhn},\ and\ \citenamefont {Gal}}]{Farquhar2024}%
  \BibitemOpen
  \bibfield  {author} {\bibinfo {author} {\bibfnamefont {S.}~\bibnamefont {Farquhar}}, \bibinfo {author} {\bibfnamefont {J.}~\bibnamefont {Kossen}}, \bibinfo {author} {\bibfnamefont {L.}~\bibnamefont {Kuhn}},\ and\ \bibinfo {author} {\bibfnamefont {Y.}~\bibnamefont {Gal}},\ }\bibfield  {title} {\bibinfo {title} {Detecting hallucinations in large language models using semantic entropy},\ }\href {https://doi.org/10.1038/s41586-024-07421-0} {\bibfield  {journal} {\bibinfo  {journal} {Nature}\ }\textbf {\bibinfo {volume} {630}},\ \bibinfo {pages} {625–630} (\bibinfo {year} {2024})}\BibitemShut {NoStop}%
\bibitem [{\citenamefont {Aithal}\ \emph {et~al.}(2024)\citenamefont {Aithal}, \citenamefont {Maini}, \citenamefont {Lipton},\ and\ \citenamefont {Kolter}}]{aithal2024understanding}%
  \BibitemOpen
  \bibfield  {author} {\bibinfo {author} {\bibfnamefont {S.~K.}\ \bibnamefont {Aithal}}, \bibinfo {author} {\bibfnamefont {P.}~\bibnamefont {Maini}}, \bibinfo {author} {\bibfnamefont {Z.~C.}\ \bibnamefont {Lipton}},\ and\ \bibinfo {author} {\bibfnamefont {J.~Z.}\ \bibnamefont {Kolter}},\ }\bibfield  {title} {\bibinfo {title} {Understanding hallucinations in diffusion models through mode interpolation},\ }\href {https://arxiv.org/abs/2406.09358} {\bibfield  {journal} {\bibinfo  {journal} {arXiv:2406.09358 [cs.LG]}\ } (\bibinfo {year} {2024})}\BibitemShut {NoStop}%
\bibitem [{\citenamefont {Anderson}\ and\ \citenamefont {Rhodes}(1983)}]{anderson_smoothing_1983}%
  \BibitemOpen
  \bibfield  {author} {\bibinfo {author} {\bibfnamefont {B.~D.~O.}\ \bibnamefont {Anderson}}\ and\ \bibinfo {author} {\bibfnamefont {I.~B.}\ \bibnamefont {Rhodes}},\ }\bibfield  {title} {\bibinfo {title} {Smoothing algorithms for nonlinear finite-dimensional systems},\ }\href {https://doi.org/10.1080/17442508308833251} {\bibfield  {journal} {\bibinfo  {journal} {Stochastics}\ }\textbf {\bibinfo {volume} {9}},\ \bibinfo {pages} {139–165} (\bibinfo {year} {1983})}\BibitemShut {NoStop}%
\bibitem [{\citenamefont {Sun}\ \emph {et~al.}(2022)\citenamefont {Sun}, \citenamefont {Yu}, \citenamefont {Dai}, \citenamefont {Schuurmans},\ and\ \citenamefont {Dai}}]{sun_score_2022}%
  \BibitemOpen
  \bibfield  {author} {\bibinfo {author} {\bibfnamefont {H.}~\bibnamefont {Sun}}, \bibinfo {author} {\bibfnamefont {L.}~\bibnamefont {Yu}}, \bibinfo {author} {\bibfnamefont {B.}~\bibnamefont {Dai}}, \bibinfo {author} {\bibfnamefont {D.}~\bibnamefont {Schuurmans}},\ and\ \bibinfo {author} {\bibfnamefont {H.}~\bibnamefont {Dai}},\ }\bibfield  {title} {\bibinfo {title} {Score-based continuous-time discrete diffusion models},\ }\href {https://arxiv.org/abs/2211.16750} {\bibfield  {journal} {\bibinfo  {journal} {arXiv:2211.16750 [cs.LG]}\ } (\bibinfo {year} {2022})}\BibitemShut {NoStop}%
\bibitem [{\citenamefont {Sutter}\ \emph {et~al.}(2016)\citenamefont {Sutter}, \citenamefont {Tomamichel},\ and\ \citenamefont {Harrow}}]{Sutter2016}%
  \BibitemOpen
  \bibfield  {author} {\bibinfo {author} {\bibfnamefont {D.}~\bibnamefont {Sutter}}, \bibinfo {author} {\bibfnamefont {M.}~\bibnamefont {Tomamichel}},\ and\ \bibinfo {author} {\bibfnamefont {A.~W.}\ \bibnamefont {Harrow}},\ }\bibfield  {title} {\bibinfo {title} {Strengthened monotonicity of relative entropy via pinched petz recovery map},\ }in\ \href {https://doi.org/10.1109/isit.2016.7541401} {\emph {\bibinfo {booktitle} {2016 IEEE International Symposium on Information Theory (ISIT)}}}\ (\bibinfo  {publisher} {IEEE},\ \bibinfo {year} {2016})\ p.\ \bibinfo {pages} {760–764}\BibitemShut {NoStop}%
\bibitem [{\citenamefont {Clifford}\ and\ \citenamefont {Hammersley}(1971)}]{clifford1971markov}%
  \BibitemOpen
  \bibfield  {author} {\bibinfo {author} {\bibfnamefont {P.}~\bibnamefont {Clifford}}\ and\ \bibinfo {author} {\bibfnamefont {J.~M.}\ \bibnamefont {Hammersley}},\ }\href {https://ora.ox.ac.uk/objects/uuid:4ea849da-1511-4578-bb88-6a8d02f457a6} {\emph {\bibinfo {title} {Markov fields on finite graphs and lattices}}}\ (\bibinfo  {publisher} {University of Oxford},\ \bibinfo {year} {1971})\BibitemShut {NoStop}%
\bibitem [{\citenamefont {Montanari}(2015)}]{Montanari2015}%
  \BibitemOpen
  \bibfield  {author} {\bibinfo {author} {\bibfnamefont {A.}~\bibnamefont {Montanari}},\ }\bibfield  {title} {\bibinfo {title} {Computational implications of reducing data to sufficient statistics},\ }\href {https://doi.org/10.1214/15-ejs1059} {\bibfield  {journal} {\bibinfo  {journal} {Electronic Journal of Statistics}\ }\textbf {\bibinfo {volume} {9}},\ \bibinfo {pages} {2370–} (\bibinfo {year} {2015})}\BibitemShut {NoStop}%
\bibitem [{\citenamefont {Vuffray}\ \emph {et~al.}(2016)\citenamefont {Vuffray}, \citenamefont {Misra}, \citenamefont {Lokhov},\ and\ \citenamefont {Chertkov}}]{vuffray_interaction_2015}%
  \BibitemOpen
  \bibfield  {author} {\bibinfo {author} {\bibfnamefont {M.}~\bibnamefont {Vuffray}}, \bibinfo {author} {\bibfnamefont {S.}~\bibnamefont {Misra}}, \bibinfo {author} {\bibfnamefont {A.~Y.}\ \bibnamefont {Lokhov}},\ and\ \bibinfo {author} {\bibfnamefont {M.}~\bibnamefont {Chertkov}},\ }\bibfield  {title} {\bibinfo {title} {Interaction screening: efficient and sample-optimal learning of ising models},\ }in\ \href {https://dl.acm.org/doi/10.5555/3157382.3157389} {\emph {\bibinfo {booktitle} {Proceedings of the 30th International Conference on Neural Information Processing Systems}}}\ (\bibinfo {year} {2016})\ p.\ \bibinfo {pages} {2603–2611}\BibitemShut {NoStop}%
\bibitem [{\citenamefont {Donsker}\ and\ \citenamefont {Varadhan}(1983)}]{donsker_asymptotic_1983}%
  \BibitemOpen
  \bibfield  {author} {\bibinfo {author} {\bibfnamefont {M.~D.}\ \bibnamefont {Donsker}}\ and\ \bibinfo {author} {\bibfnamefont {S.~R.~S.}\ \bibnamefont {Varadhan}},\ }\bibfield  {title} {\bibinfo {title} {{Asymptotic evaluation of certain Markov process expectations for large time. IV}},\ }\href {http://dx.doi.org/10.1002/cpa.3160360204} {\bibfield  {journal} {\bibinfo  {journal} {Communications on Pure and Applied Mathematics}\ }\textbf {\bibinfo {volume} {30}},\ \bibinfo {pages} {182} (\bibinfo {year} {1983})}\BibitemShut {NoStop}%
\bibitem [{\citenamefont {Lu}\ \emph {et~al.}(2025)\citenamefont {Lu}, \citenamefont {Kanász-Nagy}, \citenamefont {Kukuljan},\ and\ \citenamefont {Cirac}}]{lu_tensor_2025}%
  \BibitemOpen
  \bibfield  {author} {\bibinfo {author} {\bibfnamefont {S.}~\bibnamefont {Lu}}, \bibinfo {author} {\bibfnamefont {M.}~\bibnamefont {Kanász-Nagy}}, \bibinfo {author} {\bibfnamefont {I.}~\bibnamefont {Kukuljan}},\ and\ \bibinfo {author} {\bibfnamefont {J.~I.}\ \bibnamefont {Cirac}},\ }\bibfield  {title} {\bibinfo {title} {Tensor networks and efficient descriptions of classical data},\ }\href {https://doi.org/10.1103/physreva.111.032409} {\bibfield  {journal} {\bibinfo  {journal} {Physical Review A}\ }\textbf {\bibinfo {volume} {111}},\ \bibinfo {pages} {032409} (\bibinfo {year} {2025})}\BibitemShut {NoStop}%
\bibitem [{\citenamefont {Srivastava}\ \emph {et~al.}(2014)\citenamefont {Srivastava}, \citenamefont {Hinton}, \citenamefont {Krizhevsky}, \citenamefont {Sutskever},\ and\ \citenamefont {Salakhutdinov}}]{srivastava_dropout_2014}%
  \BibitemOpen
  \bibfield  {author} {\bibinfo {author} {\bibfnamefont {N.}~\bibnamefont {Srivastava}}, \bibinfo {author} {\bibfnamefont {G.}~\bibnamefont {Hinton}}, \bibinfo {author} {\bibfnamefont {A.}~\bibnamefont {Krizhevsky}}, \bibinfo {author} {\bibfnamefont {I.}~\bibnamefont {Sutskever}},\ and\ \bibinfo {author} {\bibfnamefont {R.}~\bibnamefont {Salakhutdinov}},\ }\bibfield  {title} {\bibinfo {title} {Dropout: a simple way to prevent neural networks from overfitting},\ }\href {https://dl.acm.org/doi/10.5555/2627435.2670313} {\bibfield  {journal} {\bibinfo  {journal} {J. Mach. Learn. Res.}\ }\textbf {\bibinfo {volume} {15}},\ \bibinfo {pages} {1929–1958} (\bibinfo {year} {2014})}\BibitemShut {NoStop}%
\bibitem [{\citenamefont {Kingma}\ and\ \citenamefont {Ba}(2015)}]{kingma_adam_2015}%
  \BibitemOpen
  \bibfield  {author} {\bibinfo {author} {\bibfnamefont {D.~P.}\ \bibnamefont {Kingma}}\ and\ \bibinfo {author} {\bibfnamefont {J.}~\bibnamefont {Ba}},\ }\bibfield  {title} {\bibinfo {title} {Adam: A method for stochastic optimization},\ }in\ \href {http://arxiv.org/abs/1412.6980} {\emph {\bibinfo {booktitle} {International Conference on Learning Representations}}}\ (\bibinfo {year} {2015})\BibitemShut {NoStop}%
\bibitem [{\citenamefont {Loshchilov}\ and\ \citenamefont {Hutter}(2019)}]{loshchilov2019decoupled}%
  \BibitemOpen
  \bibfield  {author} {\bibinfo {author} {\bibfnamefont {I.}~\bibnamefont {Loshchilov}}\ and\ \bibinfo {author} {\bibfnamefont {F.}~\bibnamefont {Hutter}},\ }\bibfield  {title} {\bibinfo {title} {Decoupled weight decay regularization},\ }in\ \href {https://openreview.net/forum?id=Bkg6RiCqY7} {\emph {\bibinfo {booktitle} {International Conference on Learning Representations}}}\ (\bibinfo {year} {2019})\BibitemShut {NoStop}%
\bibitem [{\citenamefont {Lee}\ and\ \citenamefont {Rhee}(2025)}]{lee_benchmark_2025}%
  \BibitemOpen
  \bibfield  {author} {\bibinfo {author} {\bibfnamefont {K.}~\bibnamefont {Lee}}\ and\ \bibinfo {author} {\bibfnamefont {W.}~\bibnamefont {Rhee}},\ }\bibfield  {title} {\bibinfo {title} {A benchmark suite for evaluating neural mutual information estimators on unstructured datasets},\ }in\ \href {https://dl.acm.org/doi/10.5555/3737916.3739388} {\emph {\bibinfo {booktitle} {Proceedings of the 38th International Conference on Neural Information Processing Systems}}}\ (\bibinfo {year} {2025})\ pp.\ \bibinfo {pages} {46319--46338}\BibitemShut {NoStop}%
\bibitem [{\citenamefont {Perez}\ \emph {et~al.}(2024)\citenamefont {Perez}, \citenamefont {Strub}, \citenamefont {de~Vries}, \citenamefont {Dumoulin},\ and\ \citenamefont {Courville}}]{perez2017f}%
  \BibitemOpen
  \bibfield  {author} {\bibinfo {author} {\bibfnamefont {E.}~\bibnamefont {Perez}}, \bibinfo {author} {\bibfnamefont {F.}~\bibnamefont {Strub}}, \bibinfo {author} {\bibfnamefont {H.}~\bibnamefont {de~Vries}}, \bibinfo {author} {\bibfnamefont {V.}~\bibnamefont {Dumoulin}},\ and\ \bibinfo {author} {\bibfnamefont {A.}~\bibnamefont {Courville}},\ }\bibfield  {title} {\bibinfo {title} {Film: Visual reasoning with a general conditioning layer},\ }\href {https://arxiv.org/abs/1709.07871} {\bibfield  {journal} {\bibinfo  {journal} {arXiv:1709.07871 [cs.CV]}\ } (\bibinfo {year} {2024})}\BibitemShut {NoStop}%
\bibitem [{\citenamefont {Lee}\ \emph {et~al.}(2025)\citenamefont {Lee}, \citenamefont {You},\ and\ \citenamefont {Xu}}]{lee_2025_symmetry}%
  \BibitemOpen
  \bibfield  {author} {\bibinfo {author} {\bibfnamefont {J.~Y.}\ \bibnamefont {Lee}}, \bibinfo {author} {\bibfnamefont {Y.-Z.}\ \bibnamefont {You}},\ and\ \bibinfo {author} {\bibfnamefont {C.}~\bibnamefont {Xu}},\ }\bibfield  {title} {\bibinfo {title} {Symmetry protected topological phases under decoherence},\ }\href {https://doi.org/10.22331/q-2025-01-23-1607} {\bibfield  {journal} {\bibinfo  {journal} {Quantum}\ }\textbf {\bibinfo {volume} {9}},\ \bibinfo {pages} {1607} (\bibinfo {year} {2025})}\BibitemShut {NoStop}%
\bibitem [{\citenamefont {Lee}\ \emph {et~al.}(2023)\citenamefont {Lee}, \citenamefont {Jian},\ and\ \citenamefont {Xu}}]{lee_2023_quantum}%
  \BibitemOpen
  \bibfield  {author} {\bibinfo {author} {\bibfnamefont {J.~Y.}\ \bibnamefont {Lee}}, \bibinfo {author} {\bibfnamefont {C.-M.}\ \bibnamefont {Jian}},\ and\ \bibinfo {author} {\bibfnamefont {C.}~\bibnamefont {Xu}},\ }\bibfield  {title} {\bibinfo {title} {Quantum criticality under decoherence or weak measurement},\ }\href {https://doi.org/10.1103/prxquantum.4.030317} {\bibfield  {journal} {\bibinfo  {journal} {PRX Quantum}\ }\textbf {\bibinfo {volume} {4}},\ \bibinfo {pages} {030317} (\bibinfo {year} {2023})}\BibitemShut {NoStop}%
\bibitem [{\citenamefont {Lessa}\ \emph {et~al.}(2025)\citenamefont {Lessa}, \citenamefont {Ma}, \citenamefont {Zhang}, \citenamefont {Bi}, \citenamefont {Cheng},\ and\ \citenamefont {Wang}}]{lessa_2025_strong}%
  \BibitemOpen
  \bibfield  {author} {\bibinfo {author} {\bibfnamefont {L.~A.}\ \bibnamefont {Lessa}}, \bibinfo {author} {\bibfnamefont {R.}~\bibnamefont {Ma}}, \bibinfo {author} {\bibfnamefont {J.-H.}\ \bibnamefont {Zhang}}, \bibinfo {author} {\bibfnamefont {Z.}~\bibnamefont {Bi}}, \bibinfo {author} {\bibfnamefont {M.}~\bibnamefont {Cheng}},\ and\ \bibinfo {author} {\bibfnamefont {C.}~\bibnamefont {Wang}},\ }\bibfield  {title} {\bibinfo {title} {Strong-to-weak spontaneous symmetry breaking in mixed quantum states},\ }\href {https://doi.org/10.1103/prxquantum.6.010344} {\bibfield  {journal} {\bibinfo  {journal} {PRX Quantum}\ }\textbf {\bibinfo {volume} {6}},\ \bibinfo {pages} {010344} (\bibinfo {year} {2025})}\BibitemShut {NoStop}%
\bibitem [{\citenamefont {Kuno}\ \emph {et~al.}(2025)\citenamefont {Kuno}, \citenamefont {Orito},\ and\ \citenamefont {Ichinose}}]{kuno_2025_strong}%
  \BibitemOpen
  \bibfield  {author} {\bibinfo {author} {\bibfnamefont {Y.}~\bibnamefont {Kuno}}, \bibinfo {author} {\bibfnamefont {T.}~\bibnamefont {Orito}},\ and\ \bibinfo {author} {\bibfnamefont {I.}~\bibnamefont {Ichinose}},\ }\bibfield  {title} {\bibinfo {title} {Strong-to-weak spontaneous symmetry breaking and average symmetry protected topological order in the doubled hilbert space},\ }\href {https://doi.org/10.1103/physrevb.111.174110} {\bibfield  {journal} {\bibinfo  {journal} {Physical Review B}\ }\textbf {\bibinfo {volume} {111}},\ \bibinfo {pages} {174110} (\bibinfo {year} {2025})}\BibitemShut {NoStop}%
\bibitem [{\citenamefont {Petz}(1986)}]{petz_sufficient_1986}%
  \BibitemOpen
  \bibfield  {author} {\bibinfo {author} {\bibfnamefont {D.}~\bibnamefont {Petz}},\ }\bibfield  {title} {\bibinfo {title} {Sufficient subalgebras and the relative entropy of states of a von neumann algebra},\ }\href {https://doi.org/10.1007/bf01212345} {\bibfield  {journal} {\bibinfo  {journal} {Communications in Mathematical Physics}\ }\textbf {\bibinfo {volume} {105}},\ \bibinfo {pages} {123–131} (\bibinfo {year} {1986})}\BibitemShut {NoStop}%
\bibitem [{\citenamefont {Leifer}\ and\ \citenamefont {Spekkens}(2013)}]{leifer_towards_2013}%
  \BibitemOpen
  \bibfield  {author} {\bibinfo {author} {\bibfnamefont {M.~S.}\ \bibnamefont {Leifer}}\ and\ \bibinfo {author} {\bibfnamefont {R.~W.}\ \bibnamefont {Spekkens}},\ }\bibfield  {title} {\bibinfo {title} {Towards a formulation of quantum theory as a causally neutral theory of bayesian inference},\ }\href {https://doi.org/10.1103/physreva.88.052130} {\bibfield  {journal} {\bibinfo  {journal} {Physical Review A}\ }\textbf {\bibinfo {volume} {88}},\ \bibinfo {pages} {052130} (\bibinfo {year} {2013})}\BibitemShut {NoStop}%
\bibitem [{\citenamefont {Khatri}\ and\ \citenamefont {Wilde}(2020)}]{khatri_principles_2020}%
  \BibitemOpen
  \bibfield  {author} {\bibinfo {author} {\bibfnamefont {S.}~\bibnamefont {Khatri}}\ and\ \bibinfo {author} {\bibfnamefont {M.~M.}\ \bibnamefont {Wilde}},\ }\bibfield  {title} {\bibinfo {title} {Principles of quantum communication theory: A modern approach},\ }\href {https://arxiv.org/abs/2011.04672} {\bibfield  {journal} {\bibinfo  {journal} {arXiv:2011.04672 [quant-ph]}\ } (\bibinfo {year} {2020})}\BibitemShut {NoStop}%
\bibitem [{\citenamefont {Mark}(2016)}]{mark_quantum_2016}%
  \BibitemOpen
  \bibfield  {author} {\bibinfo {author} {\bibfnamefont {W.~M.}\ \bibnamefont {Mark}},\ }\href {https://www.cambridge.org/core/books/quantum-information-theory/247A740E156416531AA8CB97DFDAE438} {\emph {\bibinfo {title} {Quantum Information Theory}}}\ (\bibinfo  {publisher} {Cambridge University Press},\ \bibinfo {year} {2016})\BibitemShut {NoStop}%
\bibitem [{\citenamefont {Zhang}\ \emph {et~al.}(2024)\citenamefont {Zhang}, \citenamefont {Xu}, \citenamefont {Chen},\ and\ \citenamefont {Zhuang}}]{zhang_generative_2024}%
  \BibitemOpen
  \bibfield  {author} {\bibinfo {author} {\bibfnamefont {B.}~\bibnamefont {Zhang}}, \bibinfo {author} {\bibfnamefont {P.}~\bibnamefont {Xu}}, \bibinfo {author} {\bibfnamefont {X.}~\bibnamefont {Chen}},\ and\ \bibinfo {author} {\bibfnamefont {Q.}~\bibnamefont {Zhuang}},\ }\bibfield  {title} {\bibinfo {title} {Generative quantum machine learning via denoising diffusion probabilistic models},\ }\href {https://doi.org/10.1103/physrevlett.132.100602} {\bibfield  {journal} {\bibinfo  {journal} {Physical Review Letters}\ }\textbf {\bibinfo {volume} {132}},\ \bibinfo {pages} {100602} (\bibinfo {year} {2024})}\BibitemShut {NoStop}%
\bibitem [{\citenamefont {Liu}\ \emph {et~al.}(2025)\citenamefont {Liu}, \citenamefont {Zhuang}, \citenamefont {Hou},\ and\ \citenamefont {You}}]{liu2025measurement}%
  \BibitemOpen
  \bibfield  {author} {\bibinfo {author} {\bibfnamefont {X.}~\bibnamefont {Liu}}, \bibinfo {author} {\bibfnamefont {J.}~\bibnamefont {Zhuang}}, \bibinfo {author} {\bibfnamefont {W.}~\bibnamefont {Hou}},\ and\ \bibinfo {author} {\bibfnamefont {Y.-Z.}\ \bibnamefont {You}},\ }\bibfield  {title} {\bibinfo {title} {Measurement-based quantum diffusion models},\ }\href {https://arxiv.org/abs/2508.08799} {\bibfield  {journal} {\bibinfo  {journal} {arXiv:2508.08799 [quant-ph]}\ } (\bibinfo {year} {2025})}\BibitemShut {NoStop}%
\end{thebibliography}%





\clearpage

\appendix


\begin{widetext}



\begin{center}
    {\large \textbf{Supplementary Materials: Local Diffusion Models and Phases of Data Distributions}}\\[1em]
\end{center}

\startcontents[appendices]
\printcontents[appendices]{l}{1}{\section*{CONTENTS}\setcounter{tocdepth}{2}}

\setcounter{page}{1}
\setcounter{equation}{0}
\setcounter{figure}{0}
\counterwithout{equation}{section}
\renewcommand{\tablename}{Table}
\renewcommand{\figurename}{Fig.}
\renewcommand{\thetable}{S\arabic{table}}
\renewcommand{\thefigure}{S\arabic{figure}}
\renewcommand{\theequation}{S\arabic{equation}}
\renewcommand{\thetheorem}{S\arabic{theorem}}

\renewcommand{\theHtable}{S\arabic{table}}
\renewcommand{\theHfigure}{S\arabic{figure}}
\renewcommand{\theHequation}{S\arabic{equation}}
\renewcommand{\theHtheorem}{S\arabic{theorem}}

\renewcommand\thesection{S\arabic{section}}
\renewcommand\thesubsection{\Alph{subsection}}
\renewcommand\thesubsubsection{\arabic{subsubsection}}

\titleformat{\section}
  {\centering\normalfont\bfseries} 
  {\thesection}{1em}{\MakeTextUppercase} 

\titleformat{\subsection}
  {\centering\normalfont\bfseries} 
  {\thesubsection}{1em}{} 

\cftsetindents{section}{0em}{3em}
\cftsetindents{subsection}{3em}{3em}

\setlength{\parindent}{0em}

\makeatletter
\let\toc@pre\relax
\let\toc@post\relax
\makeatother



\section{Derivation of Score-based Denoising from Bayes Formula}
\label{apxsec:Bayes}

\subsection{Denoising for the continuous variable}
\label{apxsec:cv}

In this appendix, we only consider the simplest 1D diffusion model with forward diffusion process: $\partial_t P = - \partial_x (\mu P) + \frac{1}{2} \partial^2_x P$ and backward denoising is $\partial_t Q = - \partial_x ((-\mu+\partial_x \ln P) Q) + \frac{1}{2} \partial^2_x P$. 
The simplest way to prove this backward Fokker-Planck equation is to substitute $Q(t)$ with $P(T-t)$. Notice that $Q_t(x) = P_{T-t}(x)$ implies that $\partial_t Q_t(x) = - \partial_t P_{T-t}(x)$. Then, we have
\begin{equation}
    \partial_t Q = - \partial_t P = \partial_x (\mu P ) - \frac{1}{2} \partial^2_x P = \partial_x (\mu P - \partial_x P) + \frac{1}{2} \partial^2_x P = - \partial_x ((-\mu+\partial_x \ln P) P) + \frac{1}{2} \partial^2_x P.
\end{equation}

Here, we also provide a different way to derive the score function in the backward Fokker-Planck equation of diffusion models by directly taking the time-continuous limit of the Bayes recovery channel of the forward diffusion channel. This perspective of derivation can provide a useful tool for generalization when we derive the stochastic differential equation of local Bayes recovery channels in Section \ref{apxsec:local_bayes_sde}.

Let $P : \mathcal{X} \rightarrow \mathbb{R}$ be a probability distribution  with continuous space $\mathcal{X}$, and $\mathcal{N} (y |  x) : \mathcal{X} \rightarrow \mathcal{X}$ is a noisy channel. It induces the transformation
\begin{equation}
    \mathcal{N} (P) (y) = \int_{\mathcal{X}} \dd x \, \mathcal{N} (y |  x) P (x) .
\end{equation}
Then the Bayes recovery channel $\mathcal{B}_{\mathcal{N}, P} (x |  y) : \mathcal{X} \rightarrow \mathcal{X}$ of $\mathcal{N}$ with reference probability $P$ is defined as
\begin{equation}
    \mathcal{B}_{\mathcal{N}, P} (x |  y) = \frac{\mathcal{N} (y |  x) P (x)}{\mathcal{N} (P) (y)} .
\end{equation}
For a infinitesimal transformation $\mathcal{N}_{\dt} (y | x)$, the transformed probability $\mathcal{N}_{\dt} (P) (y)$ is generated by the \textit{Fokker-Planck equation}:
\begin{equation}
    \mathcal{N}_{\dt} (P) (y) = P (y) + \dt \left[ - \frac{\partial}{\partial y} (\mu (y) P (y)) + \frac{1}{2} \frac{\partial^2 P}{\partial y^2} (y) \right] +\mathcal{O} (\dt^2) .
\end{equation}
Also, the adjoint generator acts on any test function $g (y)$ is
\begin{equation}
    \int_{\mathcal{X}} \dd y \, g (y) \mathcal{N}_{\dt} (y |  x) = g (x) + \dt \left[ \mu (x) \frac{\partial g}{\partial x} (x) + \frac{1}{2} \frac{\partial^2 g}{\partial y^2} (x) \right] +\mathcal{O} (\dt^2) .
\end{equation}
Now, we can compute the Bayes channel $\mathcal{B}_{\mathcal{N}_{\dt}, P} (x | y)$ for $\mathcal{N}_{\dt}$. Consider an arbitrary probability distribution $Q : \mathcal{X} \rightarrow \mathbb{R}$, we have
\begin{align}
    \mathcal{B}_{\mathcal{N}_{\dt}, P} (Q) (x) & = \int_{\mathcal{X}} \dd y \,  \frac{\mathcal{N} (y |  x) P (x)}{\mathcal{N} (P) (y)} Q (y) \nonumber\\
    & = P (x) \left( \frac{Q (x)}{\mathcal{N} (P) (x)} + \dt \left[ \mu (x) \frac{\partial}{\partial x} \left( \frac{Q (x)}{\mathcal{N} (P) (x)} \right) + \frac{1}{2} \frac{\partial^2}{\partial y^2} \left( \frac{Q (x)}{\mathcal{N} (P) (x)} \right) \right] +\mathcal{O} (\dt^2) \right) \nonumber\\
    & = P \cdot \left( \frac{Q}{P + \dt \left( - \frac{\partial}{\partial x} (\mu P) + \frac{1}{2} \frac{\partial^2 P}{\partial x^2} \right)} + \dt \left( \mu \frac{\partial}{\partial x} \left( \frac{Q}{P} \right) + \frac{1}{2} \frac{\partial^2}{\partial x^2} \left( \frac{Q}{P} \right) \right) +\mathcal{O} (\dt^2) \right) \nonumber\\
    & = P \cdot \left( \frac{Q}{P} - \dt \frac{Q}{P^2} \left( - \frac{\partial}{\partial x} (\mu P) + \frac{1}{2} \frac{\partial^2 P}{\partial x^2} \right) + \dt \left( \mu \frac{\partial}{\partial x} \left( \frac{Q}{P} \right) + \frac{1}{2} \frac{\partial^2}{\partial x^2} \left( \frac{Q}{P} \right) \right) \right) +\mathcal{O} (\dt^2) \nonumber\\
    & = Q (x) + \dt \left[ - \frac{\partial}{\partial x} \left( \left( - \mu (x) + \frac{\partial}{\partial x} (\ln P (x)) \right) Q \right) + \frac{1}{2} \frac{\partial^2 Q}{\partial x^2} \right] +\mathcal{O} (\dt^2).
\end{align}
We can read out the standard denoising Fokker-Planck equation:
\begin{equation}
    \frac{\partial Q}{\partial t} = - \frac{\partial}{\partial x} \left( \left( - \mu (x) + \frac{\partial}{\partial x} (\ln P (x)) \right) Q \right) + \frac{1}{2} \frac{\partial^2 Q}{\partial x^2} ,
\end{equation}
where the function $s(x):= \partial_x (\ln P (x))$ is usually called the \textit{score function}. 
We also emphasize that there is a degree of freedom in diffusion models. The solution to the Fokker-Planck equation $\partial_t Q = - \partial_x ((-\mu + \frac{\eta^2 + 1}{2} \partial_x \ln P) Q) + \frac{\eta^2}{2} \partial^2_x P$ is always $Q_t(x) = P_{T-t}(x)$ for any constant $\eta \geq 0$. 
We remark that $\eta = 1$ corresponds to the standard denoising diffusion probabilistic models, while $\eta=0$ corresponds to the standard \textit{denoising diffusion implicit models} (DDIM) models \cite{song_ddim_2021}.
Our derivation here shows that $\eta$ must be $1$ if the recovery channel is constructed by Bayes' theorem.

\subsection{Denoising for the discrete variable}
\label{apxsec:dv}

Even though diffusion models are usually defined in continuous variables, we note that they can also be applied in the case where variables are discrete.
We will encounter this scenario in the 2D classical toric code example in SM \ref{apxsec:LVPT}.

Let $P : \mathcal{X} \rightarrow [0, 1]$ be a probability distribution with discrete space $\mathcal{X}$, and $\mathcal{N} (y |  x) : \mathcal{X} \rightarrow \mathcal{X}$ is a stochastic channel. It induces the transformation
\begin{equation}
  \mathcal{N} (P) (y) = \sum_{x \in \mathcal{X}} \mathcal{N} (y |  x) P (x) .
\end{equation}
Then the Bayes recovery for $P$ is again $\mathcal{B}_{\mathcal{N}, P} (x |  y) = \frac{\mathcal{N} (y |  x) P (x)}{\mathcal{N} (P) (y)}$. For a infinitesimal transformation $\mathcal{N}_{\dt} (y |  x)$, the transformed probability $\mathcal{N} (P) (y)$ is generated by the master equation:
\begin{equation}
    \mathcal{N}_{\dt} (P) (y) = P (y) + \dt \sum_{x \in \mathcal{X}} \mathcal{L} (y |  x) P (x) +\mathcal{O} (\dt^2) .
\end{equation}
One constraint for $\mathcal{L} (y |  x)$ is that $\mathcal{L} (x | x) = - \sum_{y \neq x} \mathcal{L} (y |  x)$. Also, the adjoint generator acts on any test function $g (y)$ can be derived by $(g^T e^{\dt \mathcal{L}})^T = e^{\dt \mathcal{L}^T} g$ for any vector $g$:
\begin{equation}
    \sum_{y \in \mathcal{X}} g (y) \mathcal{N}_{\dt} (y |  x) = g (x) + \dt \sum_{y \in \mathcal{X}} \mathcal{L} (y |  x) g (y) +\mathcal{O} (\dt^2) .
\end{equation}
Now, we can compute the Bayes recovery $\mathcal{B}_{\mathcal{N}_{\dt}, P} (x | y)$ for $\mathcal{N}_{\dt}$. Consider an arbitrary probability distribution $Q : \mathcal{X} \rightarrow [0, 1]$, we have
\begin{align}
    \mathcal{B}_{\mathcal{N}_{\dt}, P} (Q) (x) & = \sum_{y \in \mathcal{X}}  \frac{\mathcal{N} (y |  x) P (x)}{\mathcal{N} (P) (y)} Q (y) \nonumber\\
    & = P (x) \left[ \frac{Q (x)}{\mathcal{N} (P) (x)} + \dt \sum_{y \in \mathcal{X}} \mathcal{L} (y |  x) \frac{Q (y)}{\mathcal{N} (P) (y)} +\mathcal{O} (\dt^2) \right] \nonumber\\
    & = P (x) \cdot \left( \frac{Q (x)}{P (x) + \dt \sum_{y \in \mathcal{X}} \mathcal{L} (x |  y) P (y)} + \dt \sum_{y \in \mathcal{X}} \mathcal{L} (y |  x) \frac{Q (y)}{P (y)} +\mathcal{O} (\dt^2) \right) \nonumber\\
    & = P (x) \cdot \left( \frac{Q (x)}{P (x)} - \dt \frac{Q (x)}{P^2 (x)} \sum_{y \in \mathcal{X}} \mathcal{L} (x |  y) P (y) + \dt \sum_{y \in \mathcal{X}} \mathcal{L} (y |  x) \frac{Q (y)}{P (y)} \right) +\mathcal{O} (\dt^2) \nonumber\\
    & = Q (x) + \dt \left[ \left( - \sum_{y \neq x} \mathcal{L} (x | y) \frac{P (y)}{P (x)} \right) Q (x) + \sum_{y \neq x} \left( \left( \mathcal{L} (y |  x) \frac{P (x)}{P (y)} \right) Q (y) \right) \right] +\mathcal{O} (\dt^2) . 
\end{align}
Namely, the denoising in discrete space is given by the transition strength $\mathcal{L} (y |  x) \frac{P (x)}{P (y)}$ for jump $y \rightarrow x$ with $y \neq x$ and $- \sum_{y \neq x} \mathcal{L} (x |  y) \frac{P (y)}{P (x)}$ for jump $x \rightarrow x$. This denosing process is well-known in machine learning literature \cite{anderson_smoothing_1983, sun_score_2022}.


\section{Recovery via Local Bayes Channels}
\label{apxsec:local_bayes}

\subsection{Bounds of errors in any local Bayes recovery channels}
\label{apxsec:local_bayes_cmi}

The ultimate goal of this work is to find a way of learning the backward dynamics without using the whole spatial information of $X_{t=t_n}$. For achieving this goal, we first introduce a very powerful tool in information theory called the \textit{classical Fawzi-Renner inequality}, which describes a generic upper bound of approximated recovery.

Formally speaking, let $P, Q : \mathcal{X} \rightarrow \mathbb{R}$ be two probability distributions, and $\mathcal{N} (y |  x):  \mathcal{X} \rightarrow \mathcal{X}$ is a \textit{noisy channel}. It induces the transformation
\begin{align}
    \mathcal{N} (P) (y) & = \sum_{x \in \mathcal{X}} \mathcal{N} (y | x) P (x), \\
    \mathcal{N} (Q) (y) & = \sum_{x \in \mathcal{X}} \mathcal{N} (y | x) Q (x) . 
\end{align}
Here, for simplicity, we assume that $\mathcal{X}$ is a discrete space and $\mathcal{N} (y |  x)$ is a stochastic transition matrix. Then the Bayes recovery channel for $Q$ is
\begin{equation}
    \mathcal{B}_{\mathcal{N}, Q} (x |  y) = \frac{\mathcal{N} (y |  x) Q (x)}{\mathcal{N} (Q) (y)} .
\end{equation}
Define the approximately recovered probability
\begin{equation}
    \hat{P} (x) := (\mathcal{B}_{\mathcal{N}, Q} \circ \mathcal{N} (P)) (x) = \sum_y \mathcal{B}_{\mathcal{N}, Q} (x |  y) \mathcal{N} (P) (y) .
\end{equation}

Now, we can state the \textit{classical Fawzi-Renner inequality}: for any two probability distributions $P, Q$, when we use Bayes recovery channel $\mathcal{B}_{\mathcal{N}, Q}$ to recover $\mathcal{N} (P)$, it always holds that 
\begin{equation}
    D_{\mathrm{KL}} (P \| Q ) - D_{\mathrm{KL}} (\mathcal{N} (P) \| \mathcal{N} (Q) ) \geq D_{\mathrm{KL}} (P \| \hat{P}) , \label{eq:cFR}
\end{equation}
where $\hat{P}(x) := \int \dd y \, \mathcal{B}_{\mathcal{N}, Q} (x |  y) \mathcal{N} (P) (y)$ is the distribution after recovery, and $D_{\mathrm{KL}} (P \| Q ) = \int \dd x P(x) \ln (P(x)/Q(x))$ is the \textit{Kullback-Leibler (KL) divergence}. 
The channel $\mathcal{B}_{\mathcal{N}, Q}$ can perfectly recover $Q$ from $\mathcal{N}(Q)$. But when we apply this recovery channel on $\mathcal{N}(P)$, Eq.\,(\ref{eq:cFR}) ensures that the KL-divergence between $P$ and $\hat{P}$ is at most the relative KL-divergence decreasing between $P$ and $Q$ after applying the channel $\mathcal{N}$. 
We refer the proof of Eq.\,(\ref{eq:cFR}) to the Lemma' 1 in Ref.\,\cite{li_squashed_2018}.
The inequality Eq.\,(\ref{eq:cFR}) is called ``classical'' because it can alternatively be obtained by decohering the Fawzi-Renner inequality in quantum information theory \cite{Sutter2016, fawzi_quantum_2015, li_squashed_2018}.

Now, suppose we partition the data $x$ into three spatial parts $A,B$, and $C$. Then the variable $X$ (before noise channel $\mathcal{N}$) and $Y$ (after) can also be partitioned into three parts: $X = X_A X_B X_C$ and $Y = Y_A Y_B Y_C$. We consider a local noisy channel $\mathcal{N}$ only acting on $A$ (that is $X_B X_C = Y_B Y_C$). 
Then, we set $P (x) = P_X (x_A, x_B, x_C)$ and $Q (x) = P_{X_A X_B} (x_A, x_B) P_{X_C} (x_C)$ in classical Fawzi-Renner inequality.
We emphasize that $\mathcal{N}$ only acting on $A$ means $P_{Y_B Y_C} = P_{X_B X_C}$: since $P_Y (y_A, x_B, x_C) = \int \dd x_A \mathcal{N} (y_A | x_A) P_X (x_A, x_B, x_C)$, we have
\begin{align}
    p_{Y_B Y_C} (x_B, x_C) & = \int \dd y_A P_Y (y_A, x_B, x_C) = \int \dd x_A \dd y_A \mathcal{N} (y_A |  x_A) P_X (x_A, x_B, x_C) \nonumber\\
    & = \int \dd x_A P_X (x_A, x_B, x_C) = P_{X_B X_C} (x_B, x_C) . 
\end{align}
The Bayes recovery $\mathcal{B}_{\mathcal{N}, Q}$ with $Q (x) = P_{X_A X_B} (x_A, x_B) P_{X_C} (x_C)$ can be simplified by
\begin{align}
    \mathcal{B}_{\mathcal{N}, Q} (x_A, x_B, x_C |  y_A, x_B, x_C) & = \frac{\mathcal{N} (y_A |  x_A) Q (x)}{\mathcal{N} (Q) (y)} = \frac{\mathcal{N} (y_A |  x_A) P_{X_A X_B} (x_A, x_B) P_{X_C} (x_C)}{\int \dd x_A \mathcal{N} (y_A |  x_A) P_{X_A X_B} (x_A, x_B) P_{X_C} (x_C)} \nonumber\\
    & = \frac{\mathcal{N} (y_A |  x_A) P_{X_A X_B} (x_A, x_B)}{\int \dd x_A \mathcal{N} (y_A |  x_A) P_{X_A X_B} (x_A, x_B)} \quad \left( \text{independent from } x_C \right) . 
\end{align}
Such $x_C$-independence means that we can well define:
\begin{definition}[\textbf{Local Bayes recovery channel}]
    Given the $A,B,C$ spatial partitions of the data $x = (x_A, x_B, x_C)$, for any noisy channel $\mathcal{N} (y_A | x_A)$ on $A$ and marginal distribution $P_{A B} (x_A, x_B)$ on $AB$ (here $P_{A B}$ is the abbreviation of marginal distribution $P_{X_A X_B}$ when it does not cause confusion), the local Bayes recovery is 
    \begin{equation}
        \mathcal{B}_{\mathcal{N}, P_{A B}} (x_A, x_B |  y_A, x_B) = \frac{\mathcal{N} (y_A | x_A) P_{A B} (x_A, x_B)}{\int \dd x_A \mathcal{N} (y_A |  x_A) P_{A B} (x_A, x_B)} . \label{apxeq:local_bayes}
\end{equation}
\end{definition}

By the definition of mutual information
\begin{align}
    D_{\mathrm{KL}} (P \| Q ) & = I (X_A X_B : X_C), \\
    D_{\mathrm{KL}} (\mathcal{N} (P) \| \mathcal{N} (Q) ) & = I (Y_A Y_B : Y_C) .
\end{align}
We will leverage the relation between mutual information and conditional mutual information:
\begin{align}
  I (X_A X_B : X_C) & = I (X_A : X_C |  X_B) + I (X_B : X_C), \\
  I (Y_A Y_B : X_C) & = I (Y_A : Y_C |  Y_B) + I (Y_B : Y_C), 
\end{align}
where $I (X_B : X_C) = I (Y_B : Y_C)$ because $P_{X_B X_C} (x_B, x_C) = P_{Y_B Y_C} (x_B, x_C)$. We can now bound the KL-divergence $D_{\mathrm{KL}} (P \| \hat{P} )$
\begin{equation}
    D_{\mathrm{KL}} (P \| \hat{P} ) \leq D_{\mathrm{KL}} (P \| Q ) - D_{\mathrm{KL}} (\mathcal{N} (P) \| \mathcal{N} (Q) ) = I (X_A : X_C |  X_B) - I (Y_A : Y_C |  Y_B) \leq I (X_A : X_C | X_B) . \label{apxeq:I-I}
\end{equation}


For bounding the error of multi-step denoising as what we will show in SM \ref{apxeq:local_bayes}, we need to introduce the \textit{total variance} $\mathrm{TV} (P, \hat{P}) = \frac{1}{2} \sum_x | P (x) - \hat{P} (x) |$.
According to Pinsker's inequality $2 \mathrm{TV} ( P, \hat{P} )^2 \leq D_{\mathrm{KL}} (P \| \hat{P} )$, we have 
\begin{equation}
    2 \mathrm{TV} ( P, \hat{P} )^2 \leq D_{\mathrm{KL}} (P \| \hat{P} ) \leq I (X_A : X_C |  X_B) .
\end{equation}

\textbf{Remark}. The converse of the Fawzi-Renner inequality is also true for conditional sampling \cite{fawzi_quantum_2015}. Suppose any local recovery channel $\mathcal{B}$ which recovery $(X_A, X_B)$ from $X_A$, and recovered state $\hat{P}(x_A, x_B, x_C) := \mathcal{B} (P_{X_B X_C} (x_B, x_C))$, one must have
\begin{equation}
    I (X_A : X_C |  X_B)_P \leq 7 n \, \mathrm{TV} (P, \hat{P} )^{1/2}. \label{eq:conFR}
\end{equation}
The proof of Eq.\,(\ref{eq:conFR}) can be found in the Eq.\,(10) of Ref.\,\cite{fawzi_quantum_2015}.

\subsection{Stochastic differential equation of local Bayes recovery channels}
\label{apxsec:local_bayes_sde}

Now, let us derive the continuous time version of Eq.\,(\ref{apxeq:local_bayes}) with local forward channel only acting on $A$. If the forward process only acts on $A$, then the forward SDE is:
\begin{align}
    \dd X_A & = \mu (X_A, t) \dd t + \sigma (t) \dd W_A, \\
    \dd X_B & = \dd X_C = 0. 
\end{align}
The forward Fokker-Planck equation is
\begin{equation}
    \frac{\partial P}{\partial t} (x, t) = - \frac{\partial}{\partial x_A} (\mu (x_A, t) P (x, t)) + \frac{1}{2} \nabla_{x} \cdot D (t) \nabla_{x} P (x, t),
\end{equation}
where $D (t) = \left(\begin{array}{ccc}
    \sigma (t)^2 & 0 & 0 \\
    0 & 0 & 0 \\
    0 & 0 & 0
\end{array}\right)$. According to SM \ref{apxsec:cv}, the backward Fokker-Planck equation for local Bayes denoising on $A \cup B$ is:
\begin{align}
    \left(\begin{array}{c}
            \dd Y_A \\
            \dd Y_B
        \end{array}\right) & = \left[ - 
        \left(\begin{array}{c}
            \mu (Y_A, T-t)\\
            0
        \end{array}\right) + 
        \left(\begin{array}{cc}
            \sigma (T-t)^2 & 0\\
            0 & 0
        \end{array}\right) 
        \left(\begin{array}{c}
            \frac{\partial \ln P_{AB}}{\partial x_A} (Y_A, Y_B, T-t)\\
            \frac{\partial \ln P_{AB}}{\partial x_B} (Y_A, Y_B, T-t)
        \end{array}\right) \right] \dd t + 
        \left(\begin{array}{c}
            D (T-t) \dd W_A\\
            0
        \end{array}\right) \nonumber\\
    & = \left(\begin{array}{c}
            \left( - \mu (Y_A, T-t) + \sigma (T-t)^2 \frac{\partial \ln P_{AB}}{\partial x_A} (Y_A, Y_B, T-t) \right) \dd t + D (T-t) \dd W_A\\
            0
        \end{array}\right), 
\end{align}
or equivalently,
\begin{align}
    \dd Y_A & = \left( - \mu (Y_A, T-t) + \sigma (T-t)^2 \frac{\partial \ln P_{AB}}{\partial x_A} (Y_A, Y_B, T-t) \right) \dd t + \sigma (T-t) \dd W_A, \label{apxeq:local_sde} \\
    \dd Y_B & = \dd Y_C = 0. 
\end{align}
Therefore, even if the local Bayes channel Eq.\,(\ref{apxeq:local_bayes}) acts on the marginal probability distribution on $A \cup B$, this process requires knowledge about $A \cup B$ while only operating on $A$.

\subsection{Bound of total variance for non-overlapping local Bayes recovery channels}
\label{apxsec:local_bayes_overall}

Recall that the reorganized diffusion process forward with $N = T/\dt$ steps is
\begin{align}
    \mathcal{N}_{\mathrm{tot}} & := \mathcal{N}_{n = N} \circ \cdots \circ \mathcal{N}_{n = 2} \circ \mathcal{N}_{n = 1}, \\
    \mathcal{N}_{n} & := \prod_l \mathcal{N}_{n, l} . 
\end{align}
Now we consider overall local recovery channels with:
\begin{align}
    \mathcal{B}_{\mathrm{tot}} & := \mathcal{B}_{n = 1} \circ \mathcal{B}_{n = 2} \circ \cdots \circ \mathcal{B}_{n = N}, \\
    \mathcal{B}_n & := \prod_l \mathcal{B}_{n, l}, \\
    \mathcal{B}_{n, l} & := \mathcal{B}_{\mathcal{N}_{n, l}, P_{A_{n, l} {B_{n, l}} }}.
\end{align}
Here $P_{A_{n, l} {B_{n, l}} }$ is the abbreviation of $P_{X_{A_{n, l}} X_{B_{n, l}}}$. In this sub-section, we assume that for a given $n$, all regions $\{B_{n, l}\}_{l}$ are non-overlapping. We leave the proof of the case with more generic $\{B_{n, l}\}_{l}$ in SM \ref{apxsec:local_bayes_reorg}. 

The recovery error of any one single forward-backward evolution step $\mathcal{B}_n \circ \mathcal{N}_n$ acting on any $P_{n - 1}$ (due to non-overlapping of $\{B_{n, l}\}_{l}$) is bounded by:
\begin{align}
    \mathrm{TV} (\mathcal{B}_n \circ \mathcal{N}_n (P_{n - 1}), P_{n - 1}) & = \left| \sum_{l = 1}^{l_{\max} - 1} \mathrm{TV} (\mathcal{B}_{n, l} \circ \mathcal{N}_{n, l} (\mathcal{B}_{n, < l} (P_{n - 1})), \mathcal{B}_{n, < l} (P_{n - 1})) \right| \nonumber\\
    & \overset{\text{(i)}}{\leq} \sum_{l = 1}^{l_{\max} - 1} | \mathrm{TV} (\mathcal{B}_{n, l} \circ \mathcal{N}_{n, l} (\mathcal{B}_{n, < l} (P_{n - 1})), \mathcal{B}_{n, < l} (P_{n - 1})) | \nonumber\\
    & \overset{\text{(ii)}}{=} \sum_{l = 1}^{l_{\max} - 1} | \mathrm{TV} (\mathcal{B}_{n, < l} (\mathcal{B}_{n, l} \circ \mathcal{N}_{n, l} (P_{n - 1})), \mathcal{B}_{n, < l} (P_{n - 1})) | \nonumber\\
    & \overset{\text{(iii)}}{\leq} \sum_{l = 1}^{l_{\max} - 1} | \mathrm{TV} (\mathcal{B}_{n, l} \circ \mathcal{N}_{n, l} (P_{n - 1}), P_{n - 1}) |, 
\end{align}
where $\mathcal{B}_{n, < l} := \prod_{l' < l} \mathcal{B}_{n, l'} \circ \mathcal{N}_{n, l'}$. The inequality (i) is from triangle inequality of total variance, the equality (ii) is from the commutativity between $\mathcal{B}_{n, l} \circ \mathcal{N}_{n, l}$ and $\mathcal{B}_{n, < l}$, and the inequality (iii) is from contractivity of total variance under noisy channels $\mathrm{TV} (\mathcal{C} (P), \mathcal{C} (Q)) \leq \mathrm{TV} (P, Q)$.

We define $\mathcal{B}_{\{ 1, \cdots, n \}} := \mathcal{B}_1 \circ \cdots \circ \mathcal{B}_n$. We have the following iteration relation:
\begin{align}
    \mathcal{B}_{\{ 1, \cdots, n \}} (P_n) & = \mathcal{B}_{\{ 1, \cdots, n - 1 \}} (\mathcal{B}_n (P_n)) =\mathcal{B}_{\{ 1, \cdots, n - 1 \}} (\mathcal{B}_n \circ \mathcal{N}_n (P_{n - 1})) \nonumber\\
    & = \mathcal{B}_{\{ 1, \cdots, n - 1 \}} (P_{n - 1}) +\mathcal{B}_{\{ 1, \cdots, n - 1 \}} (\mathcal{B}_n \circ \mathcal{N}_n (P_{n - 1}) - P_{n - 1}) . 
\end{align}
Then the overall error of the denoising process is
\begin{align}
    \mathrm{TV} (\mathcal{B}_{\mathrm{tot}} \circ \mathcal{N}_{\mathrm{tot}} (P_0), P_0) & = \mathrm{TV} (\mathcal{B}_1 \circ \cdots \circ \mathcal{B}_N \circ \mathcal{N}_1 \circ \cdots \circ \mathcal{N}_N (P), P) = \frac{1}{2} | \mathcal{B}_{\{ 1, \cdots, N \}} (P_N) - P_0 |_1 \nonumber\\
    & \overset{\text{(i)}}{=} \frac{1}{2} \left| \sum_{n = 1}^{N - 1} \mathcal{B}_{\{ 1, \cdots, n - 1 \}} (\mathcal{B}_n \circ \mathcal{N}_n (P_{n - 1}) - P_{n - 1}) \right|_1 \nonumber\\
    & \overset{\text{(ii)}}{\leq} \frac{1}{2} \sum_{n = 1}^{N - 1} | \mathcal{B}_{\{ 1, \cdots, n - 1 \}} (\mathcal{B}_n \circ \mathcal{N}_n (P_{n - 1}) - P_{n - 1}) |_1 \nonumber\\
    & \overset{\text{(iii)}}{\leq} \frac{1}{2} \sum_{n = 1}^{N - 1} | \mathcal{B}_n \circ \mathcal{N}_n (P_{n - 1}) - P_{n - 1} |_1 \nonumber\\
    & \overset{\text{(iv)}}{\leq} \sum_{n = 1}^{N - 1} \sum_{l = 1}^{l_{\max} - 1} \mathrm{TV} (\mathcal{B}_{n, l} \circ \mathcal{N}_{n, l} (P_{n - 1}), P_{n - 1}), \label{apxeq:TVB}
\end{align}
where equality (i) is from the iteration relation, inequality (ii) is from the triangle inequality of 1-norm, inequality (iii) is from the contractivity of 1-norm under noisy channels, and inequality (iv) is from the error bound of a single forward-backward evolution step. 

\subsection{Bound of total variance for generic local Bayes recovery channels via reorganization trick}
\label{apxsec:local_bayes_reorg}

In SM \ref{apxsec:local_bayes_overall}, we let each $\mathcal{N}_{n, l}$ in $\mathcal{N}_{\mathrm{tot}} = \mathcal{N}_{N} \circ \cdots \circ \mathcal{N}_{2} \circ \mathcal{N}_{1}$ acts on a region $A_{n,l}$ and $\{A_{n,l}\}$ do not overlap with each other. However, to undo the effect of $\mathcal{N}_{n,l}$, one usually has to apply a local Bayes channel $\mathcal{B}_{n,l}$ in a larger region $A_{n,l} \cup B_{n,l}$. In general, it is not guaranteed that these $\{B_{n,l}\}$ are non-overlapping for a given $n$. Now, we introduce a reorganization trick proposed in Ref.\,\cite{sang_statbility_2025} to handle the generic case where the newly constructed $\{B_{n,l}\}$ are non-overlapping. 

Roughly speaking, we just need to reorganize the forward diffusion process a little bit to make sure that when each $A_{n,l}$ is expanded into $A_{n,l} \cup B_{n,l}$, those $B_{n,l}$ are non-overlapping.
To be more specific, for the $n$-th diffusion step, we reorganize these $\mathcal{N}_{n, l}$ into $M_n$ diffusion sub-steps $t_{n} = t_{n, 0} < \cdots < t_{n, M_n} = t_{n+1}$, such that the local noisy channels within each new sub-step are at least distance $2 r_n$ separated from each other, with $r_n$ at each diffusion to be determined later in Eq.\,(\ref{eq:bayes_radius}). The sub-step number scales as $M_n = O(r_n^d)$, which is also shown to be $\mathrm{polylog}(L)$ later. See the 2D schematic of reorganization in Fig.\,\ref{fig:reorg}a and Fig.\,\ref{fig:reorg}b.

\begin{figure}[t]
    \centering
    \includegraphics[width=0.66\linewidth]{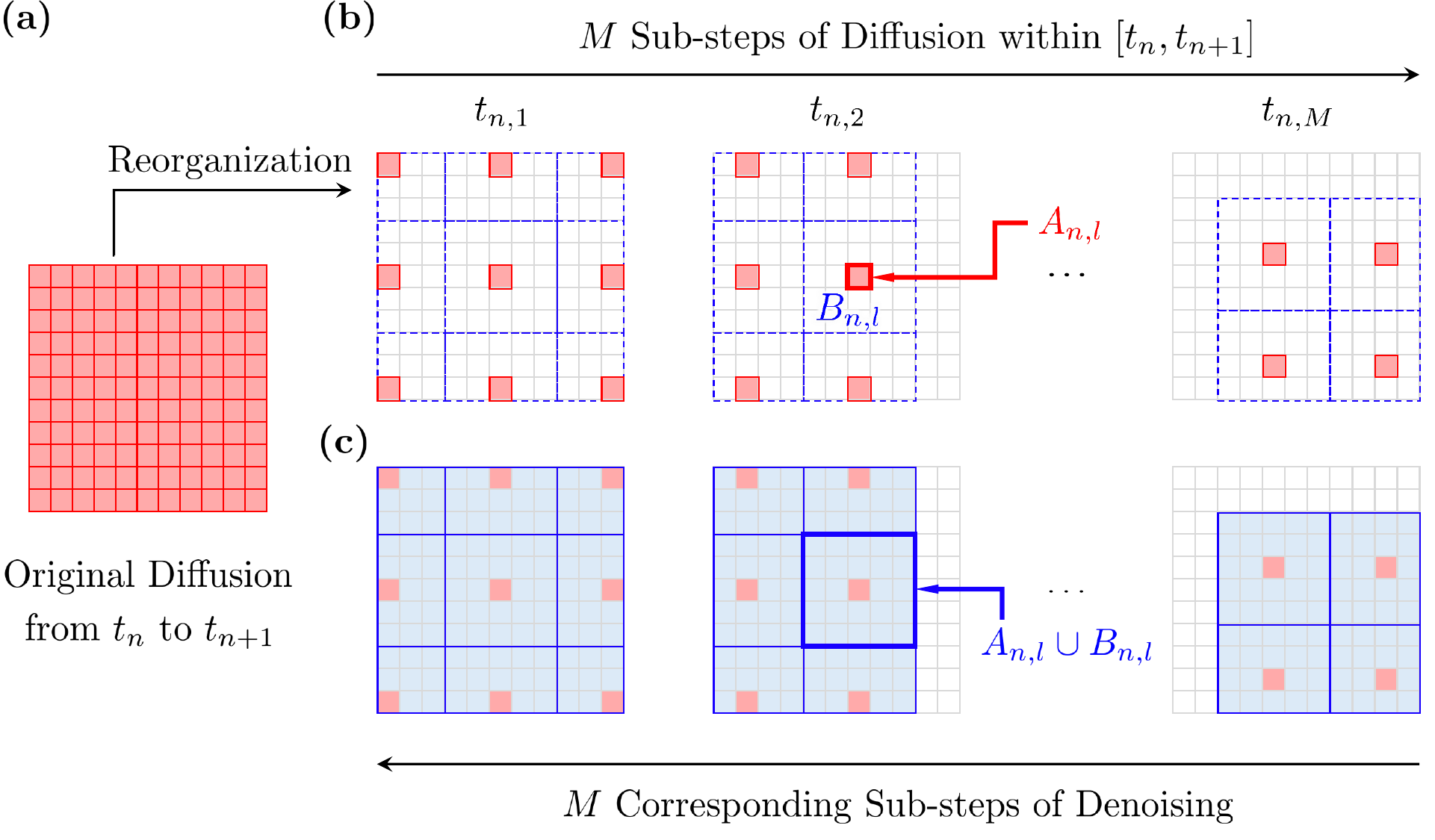}
    \caption{Schematic of reorganization with example $L=11$, $d=2$, $k=1$ and $r=2$. (a) In the $n$-th step of the original diffusion, the noise channel is added in parallel on $O(L^d)$ many $k$-sized regions. (b) Dividing the $n$-th step into $M$ sub-steps. In each sub-step, the local channels $\mathcal{N}_{n, l}$ act on local regions $A_{n,l}$ (with solid red boundary) that are separated by distance $2r$. Their $r$-distance surrounding regions are denoted as $B_{n,l}$ (with dashed blue boundary). (c) Reversal channels $\mathcal{B}_{n, l}$ of $\mathcal{N}_{n, l}$. The channel $\mathcal{B}_{n, l}$ acts on quasi-local regions $A_{n,l} \cup B_{n,l}$ (with solid blue boundary), but the operation of $\mathcal{B}_{n, l}$'s SDE only acts on $A_{n, l}$ locally.}
    \label{fig:reorg}
\end{figure}

From now on, we will always assume the diffusion process has been reorganized through Fig.\,\ref{fig:reorg}b. Again, let the overall local recovery channel be $\mathcal{B}_{\mathrm{tot}} = \mathcal{B}_{1} \circ \mathcal{B}_{2} \circ \cdots \circ \mathcal{B}_{N}$, where $\mathcal{B}_n = \prod_l \mathcal{B}_{n, l}$ and $\mathcal{B}_{n, l} = \mathcal{B}_{\mathcal{N}_{n, l}, P_{A_{n, l} {B_{n, l}} }}$. 
$B_{n,l}$ is a region surrounding $A_{n, l}$ with width $r_n$. Thanks to the reorganization trick, these $\mathcal{B}_{\mathcal{N}_{n, l}}$ within the $n$-th denoising step are also non-overlapping, because all $\{A_{n, l}\}$ are separated from each other by a distance at least $2 r_n$.

According to Eq.\,(\ref{apxeq:TVB}) in SM \ref{apxsec:local_bayes_overall}, we obtain that the overall error $\mathrm{TV} (\mathcal{B} \circ \mathcal{N} (P_0), P_0)$ of the denoising process is at most
\begin{equation}
    \sum_{n = 1}^{N - 1} \sum_{l = 1}^{l_{\max} - 1} \mathrm{TV} (\mathcal{B}_{n, l} \circ \mathcal{N}_{n, l} (P_{n - 1}), P_{n - 1}). \label{eq:local_bayes_overall}
\end{equation}

Finally, let us bound the total variance of generation in the case where the distribution $P_n$ after the $n$-th diffusion step always has a Markov length $\xi_n$.
According to Eq.\,(\ref{eq:local_bayes}) and Eq.\,(\ref{eq:fml}), finite Markov length at any time implies that each term in summation of Eq.\,(\ref{eq:local_bayes_overall}) is bounded by $N K \cdot \gamma^{1/2} e^{-r_n/2\xi_n}$. Therefore, for achieving the generation error $\mathrm{TV} (\mathcal{B} \circ \mathcal{N} (P_0), P_0) < \varepsilon$, we only need to take the width of $B_{n, l}$ for all $l$ be:
\begin{equation}
    r_n \geq 2 \xi_n \cdot \ln \left( \gamma^{1/2} N K / \varepsilon \right).
\end{equation}
Because $K = L^d$ is $\mathrm{poly}(L)$, the condition of $r_n$ also ensures that the sub-step number has a scaling $M_n = O(r_n^d) = \mathrm{polylog}(L)$.

On the other hand, the critical distance of $r_n$ in Eq.\,(\ref{eq:bayes_radius}) is explicit related to $N = O (\dt^{-1})$. But intuitively, the critical distance should not diverge when taking $\dt \to 0$. The same problem occurs in open quantum systems \cite{sang_statbility_2025}. Resolving this divergence requires an improved characterization of the CMI temporal decreasing $I(X_A : X_C |  X_B)_{P_n} - I(X_A : X_C |  X_B)_{\mathcal{N}_n(P_{n})} \propto \dt$. To the best of our knowledge, this is still an open question.

\subsection{Difference between CMI and two point correlation}
\label{apxsec:cmi-2pc}

\rev{In this subsection, we provide several examples to illustrate the difference between two-point correlation and CMI. This may help to understand why CMI actually captured the nature of local reversibility, instead of correlation.

\textbf{Large two-point correlation but zero CMI}.
Let a three-bit 1D random variable $X_A$-$X_B$-$X_C$ with $A=B=C$ and $X_A \sim \mathrm{Bernoulli}(1/2)$. It is easy to verify that it has a zero CMI $I(X_A:X_C|X_B)=0$ but a maximally non-zero correlation $I(X_A:X_C)=1$ and $\mathrm{Cov}(X_A, X_C)=1$. This zero CMI implies that both $A$ and $C$ can be locally determined and recovered by their neighbour $B$.

Another nontrivial example is the Gibbs distribution of the classical 2D Ising model near the critical temperature. Its correlation length is long-range, but its CMI is strictly zero for any $r=\mathrm{dist}(A,C) \geq 2$ due to the well-known Hammersley–Clifford theorem (which asserts that being a Gibbs is equivalent to Markovianity for classical distribution) \cite{clifford1971markov}. It is also known that such a Gibbs distribution is always locally learnable and recoverable \cite{Montanari2015, vuffray_interaction_2015}, reiterating the connection between short-range CMI and local reversibility.

\textbf{Large CMI but zero two-point correlation}.
Let $X_A$-$X_B$-$X_C$ with $X_A, X_C \sim \mathrm{Bernoulli}(1/2)$ and $X_B = X_A + X_C$ (binary addition). We can show that correlation vanishes anywhere $I(X_A:X_B)=I(X_B:X_C)=I(X_A:X_C)=0$ (and all covariances are zeros as well) but CMI $I(X_A:X_C|X_B)=1$. Since $X_A = X_B + X_C$, the lack of any one bit of $X_B$ and $X_C$ will make the recovery of $X_A$ impossible.
Such a non-zero CMI captures the essence that recovery of local information requires the knowledge of global information.

In the main text, we have already presented that for MNIST dataset, CMI has a phase transition peak at $t_c \approx 0.3 \sim 0.4$ (see Fig.\,\ref{fig:mnist}a). Also, in Sec.\,\ref{sec:numerics} of the main text, we mentioned that two-point correlation should monotonically decrease during diffusion process with pixel-wise noise. In Fig.\,\ref{fig:cov}, we numerically compute the two-point covariance between two pixels $A_1$ and $A_2$ with distance $r$, and and verify that it monotonically decrease during diffusion.

\begin{figure}[h]
    \centering
    \includegraphics[width=0.3\linewidth]{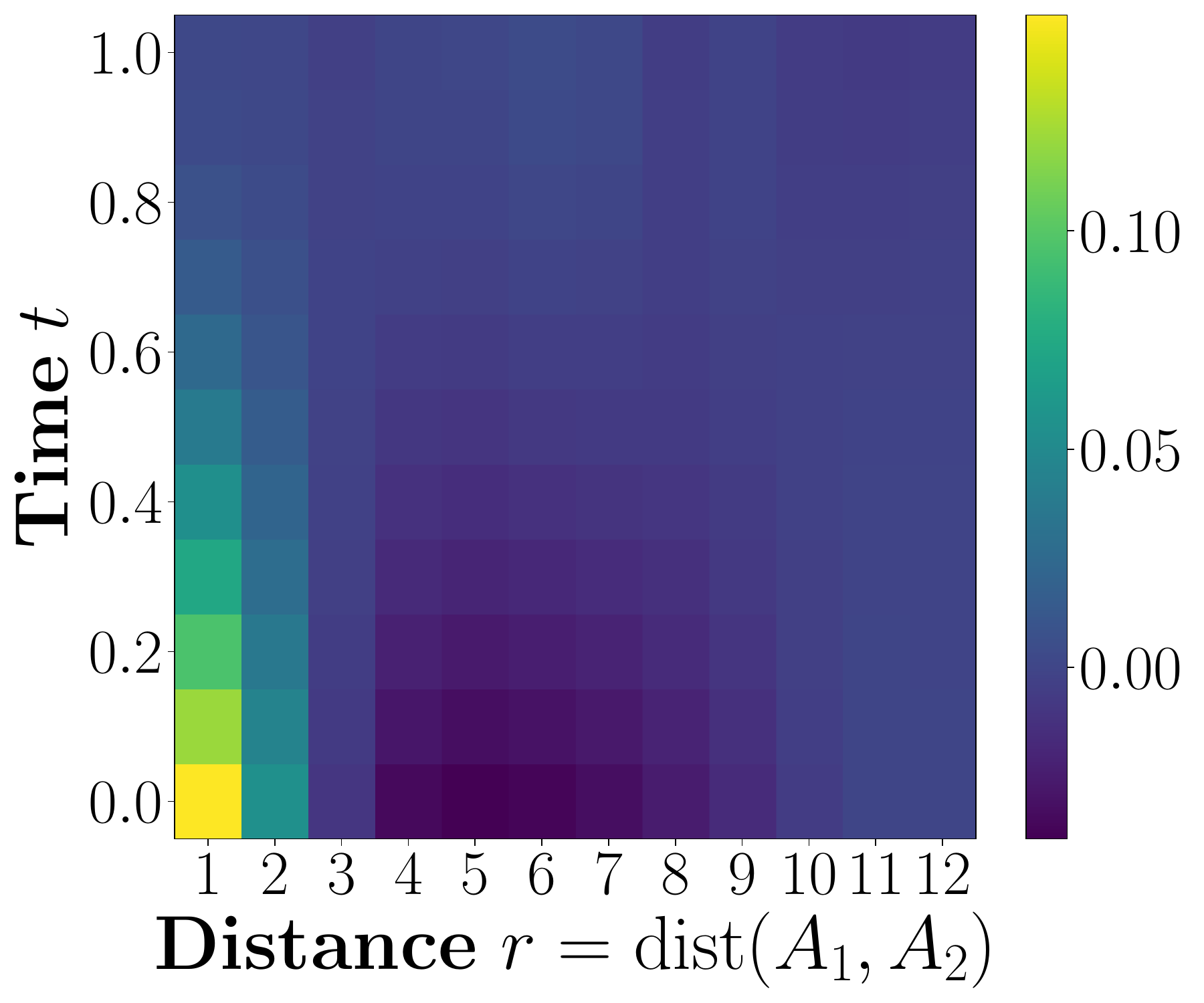}
    \caption{Covariance between the central pixel $A_1$ and another pixel $A_2$ that is $r$-distance away from $A_1$, at different diffusion time $t$. As predicted by the data processing inequality, the two-point covariance monotonically decreases during diffusion.}
    \label{fig:cov}
\end{figure}
}



\section{Numerical Details}
\label{apxsec:mnist}

\subsection{Mutual information neural estimator for CMI of MNIST}
\label{apxsec:mine}

In this section, we provide the details of evaluating the CMI of MNIST diffusion. As mentioned in the main text, we rewrite the CMI into the form of mutual information difference $I (X_A : X_C | X_B) = I (X_A : X_B X_C) - I (X_A : X_B)$. Here, $A$ is the central pixels of the images, $B$ is the neighbourhood of $A$ with a width $r$, and $C$ is the rest of the images (for well-define the center pixel $A$, we remove the first row and first column of MNIST. This turns out to make no difference in CMI calculation because the edge of MNIST are all almost close to $0$).
Therefore, we only need to resolve the mutual information in the form of $I (X_A : X_S)$ between $A$ and its surroundings. There are two types of surroundings, the first one is $X_S \leftarrow (X_B, X_C)$, which directly gives $I (X_A : X_S) = I (X_A : X_B X_C)$. The second one is $X_S = (X_B, \bm{0})$ with $| C |$ many repeated zeros as padding. This type of surroundings gives $I (X_A : X_S) = I (X_A : (X_B, \bm{0})) = I (X_A : X_B)$.

Now, we elaborate on how the mutual information neural estimator (MINE) works. The theoretical foundation of MINE is the Donsker-Varadhan dual representation of the KL divergence \cite{donsker_asymptotic_1983}. For any two distributions $P, Q$, one has
\begin{equation}
    D_{\mathrm{KL}} (P | | Q) = \sup_T \left(\mathbb{E}_P [T] - \ln (\mathbb{E}_Q [e^T]) \right) .
\end{equation}
where the supremum is taken over all functions $T$ such that the two expectations are finite. The supremum is achievable when $T^{\star}$ satisfies $\dd P = \frac{e^{T^{\star}}}{\mathbb{E}_Q [e^{T^{\star}}]} \dd Q$. Therefore, for mutual information $I (X_A : X_S) = D (P_{A S} | | P_A \otimes P_S)$, we have
\begin{equation}
    I (X_A : X_S) \geq I_{\mathrm{MINE}} (X_A : X_S) := \sup_{\theta} (\mathbb{E}_{P_{A S}} [T_{\theta}] - \ln (\mathbb{E}_{P_A \otimes P_S} [e^{T_{\theta}}])),
\end{equation}
where $T_{\theta}$ is a function represented by some neural network. The sampling of $P_A \otimes P_S$ is straightforward by just sampling the marginal distribution of $X_S$. See the pseudo-code for the MINE algorithm details of computing $I_{\mathrm{MINE}} (X_A : X_S)$.

Inspired by Ref.\,\cite{lu_tensor_2025}, we use a convolutional neural network (CNN) as the $T_{\theta}$. The CNN consists of four layers: one convolution layer, one average pooling layer with a ReLU activation and dropout with probability $p = 0.1$, one fully-connected layer with a ReLU activation and dropout with probability $p = 0.3$, and another final fully-connected output layer. The convolutional layer has a kernel size $3$. The average pooling layer has a kernel size $2$ and a stride of $2$. The dropout is a regularization technique that is used to prevent overfitting during training \cite{srivastava_dropout_2014}. 

For all times $t$ and all distances $r$, we use the same CNN architecture and keep all the following settings and hyperparameters the same. We use AdamW optimizer \cite{kingma_adam_2015, loshchilov2019decoupled} to train $T_{\theta}$. We set a batch size of $100$, a learning rate $10^{- 4}$, and a weight decay of $10^{-4}$. The total training dataset contains $60,000$ MNIST images. We train for $500$ epochs, namely a total training iteration number $300, 000$.
We also leverage the moving average trick presented in Ref.\,\cite{belghazi_mutual_2015, lee_benchmark_2025}, with a moving average rate $0.001$, to mitigate bias in minibatch sampling. 

We benchmark our numerical result by computing $I (X_A : X_B X_C)$ at $t = 0$ and $k = 1$ (that is $k/L=1/28$ in noiseless MNIST images). 
Our numerics show that, in this scenario, it yields a mutual information $I (X_A : X_B X_C) = 1.05$. This agrees with the $I(\mathrm{C}:\mathrm{S})$ at $\mathbf{L}/\mathbf{L}_{\max} = 1/28$, presented in th Fig.\,B.2b of Ref.\,\cite{lu_tensor_2025}. 

\RestyleAlgo{ruled}
\SetKwComment{Comment}{/* }{ */}
\SetKwInOut{input}{Input}
\SetKwInOut{output}{Output}
\SetKwFor{For}{For}{}{EndFor}
\SetKwFor{If}{If}{}{EndIf}
\begin{algorithm}[t]
    \caption{MINE Algorithm for $I(X_A:X_S)$}\label{alg:mine}
    \input{Dataset $\{(X_A^{(i)}, X_S^{(i)})\}_{i \in N_{\mathrm{data}}}$}
    \output{Mutual information estimator $I_{\mathrm{MINE}} (X_A : X_S)$}
    \For{$n \gets 1$ to $N_{\mathrm{iteration}}$}{
        Draw $N_{\mathrm{batch}}$ minibatch samples from joint distribution $(X_A^{(1)}, X_S^{(1)}), \cdots, (X_A^{(N_{\mathrm{batch}})}, X_S^{(N_{\mathrm{batch}})}) \sim P_{AS}$ \;
        Draw $N_{\mathrm{batch}}$ minibatch samples from marginal distribution $(\bar{X}_S^{(1)}, \cdots, \bar{X}_S^{(N_{\mathrm{batch}})} ) \sim P_S$ \;
        Evaluate $I_{\mathrm{MINE}} (X_A : X_S) = \frac{1}{N_{\mathrm{batch}}} \sum_{i=1}^{N_{\mathrm{batch}}} T_{\theta} (X_A^{(i)}, X_S^{(i)}) - \ln \left( \sum_{i=1}^{N_{\mathrm{batch}}} e^{T_{\theta} (X_A^{(i)}, \bar{X}_S^{(i)})} \right)$ \;
        Compute the gradient and update the parameters $\theta$
        \Comment*[r]{Moving average trick is needed, see Ref.\,\cite{belghazi_mutual_2015}}
    }
\end{algorithm}

\subsection{Global and local denoisers of MNIST with U-Nets}
\label{apxsec:unet}

For global denoisers, we employ a U-Net-based architecture~\cite{ronneberger2015unet}. 
A U-Net is a special convolutional neural network that allows global connections via pooling and skip connections between the encoder and decoder. 
In our numerics, each U-Net encoder block comprises two convolutional layers, group normalization, and SiLU activations, followed by $2\times2$ max pooling for down-sampling. 
The decoder mirrors this structure, using transposed convolutions for up-sampling and concatenating encoder features via skip connections. 
The channel width increases by a factor of two at each encoder stage, starting from $64$, and decreases by half at each corresponding decoder stage. Three pairs of encoders and decoders are used in this work. 
The bottleneck consists of a convolutional block with increased channel width. All convolutional blocks use a standard kernel size of 3, a stride of 1, and a padding of 1.

For timestep embedding, we use sinusoidal embeddings followed by a two-layer MLP with SiLU activations. These embeddings are injected into each encoder and decoder block using feature-wise affine transformations (FiLM)~\cite{perez2017f}.

We train the model by using the \textit{flow matching} \cite{lipman2022flow}.
Specifically, we predict the difference between the clean image and the noise, using a linear interpolation between the image and noise at a randomly sampled timestep. This schedule is also known as the $\alpha$-(de)Blending schedule~\cite{Heitz2023alphaBlending}. At each iteration, a clean image $X_0$ is sampled from the dataset, and a standard Gaussian noise vector $Z\sim \mathcal{N}(0, I)$ is randomly generated. A random timestep $t\in[0,1]$ is sampled from a standard logit-normal distribution (a random variable is standard \textit{logit-normal} if it is a sigmoid of a standard Gaussian variable). The noisy image is constructed as $X_t = (1 - t)X_0 + t Z$.
Then, the network receives data $X_t \in \mathbb{R}^K$ and time $t \in [0, 1]$. The network is trained to predict $X_0 - Z$. The loss function is the mean squared error $\mathrm{MSE}(V_t, X_0 - Z) $ between the model output $V_t$ and the target $X_0 -Z$.
Optimization is performed using AdamW with a learning rate of $10^{-3}$ and weight decay of $10^{-3}$. The model is trained for 15 epochs with a batch size of 512. 



During inference, image generation begins by sampling a batch of standard Gaussian noise $Y_0 \sim \mathcal{N}(0, I)$. We divide the denoising into $N=32$ steps. In each time step, the current image estimate $Y_t$ is passed to the U-Net model, along with the current timestep $t \in [0, 1]$. The model predicts the denoising flow direction, $V_t \approx Z-X_0$, which is scaled by the step size $1/N$ and added to $Y_t$ to produce the next estimate $Y_{t+1/N} = Y_t - \frac{1}{N}V_t$. This process is repeated, progressively reducing the noise and reconstructing image structure, until $t$ reaches zero. 

The local denoisers are essentially U-Nets but with the pooling layers removed so that we can constrain the receptive field to be small. Then, we control the kernel size in each layer to control the overall receptive field. We employ a three-layer U-Net with the kernel radius taking values between zero and two ($\text{kernel size} = 2\times \text{kernel radius}+1$). We also keep the kernel radius the same for the matching down and up layers. This gives the possible receptive field radius of $0,2,\cdots,12$. To test the odd receptive field radius, we add one more convolutional layer at the beginning with the kernel radius being zero or one. The training of the local denoisers follows from the training of the global denoiser. 

For the local recovery numerics shown in Fig.\,\ref{fig:mnist} of the main text, recall that we select a clean image $X_0$, diffuse $X_0$ to $X_{T_f}$, and then denoise $X_{T_f}$ to $Y_{T_f, r}$ with local denoisers whose kernel size is $2r+1$. 
As a complement, we compute the MSE between $Y_{T_f, r}$ and $X_{0}$. This error quantitatively reflects the fidelity of the learned flows through local denoisers, see Fig.\,\ref{fig:mnist-apx}a. We also scan the denoised images $Y_{T_f, r}$ in the interval $[0.2, 0.5]$, showing a more accurate phase tansition point $t_c = 0.38 \sim 0.41$, such that for any $T_f > t_c$, the denoised images are always significantly different from the original images for any $r$ (see Fig.\,\ref{fig:mnist-apx}b).

\begin{figure}[t]
    \centering
    \includegraphics[width=\linewidth]{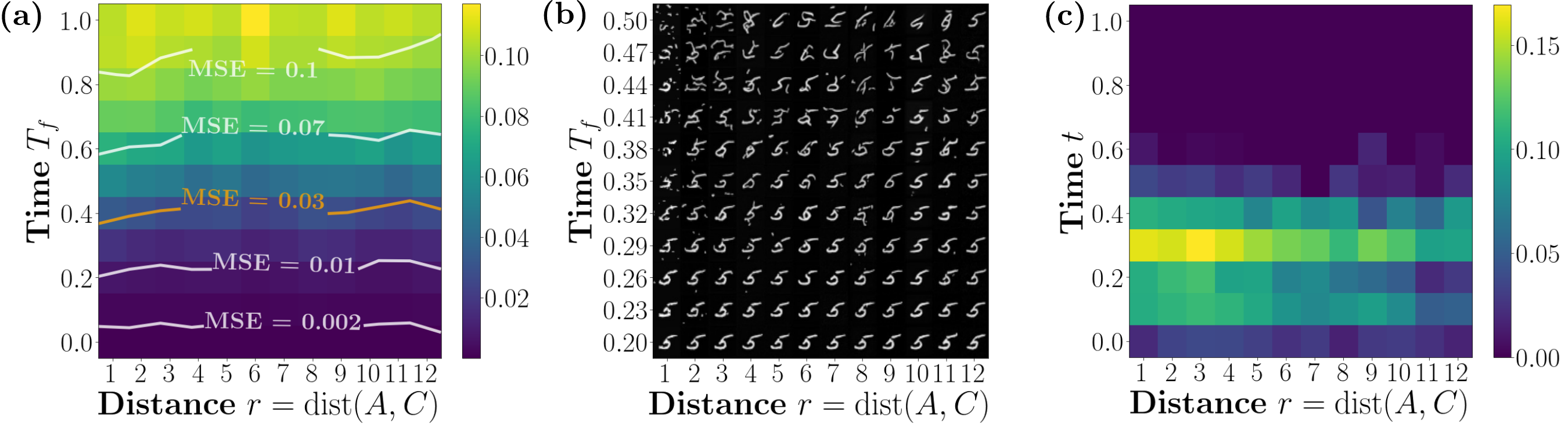}
    \caption{
    (a) MSE between $X_0$ (the clean images) and $Y_{T_f, r}$ (the images locally denoised from corrupted images at $t=T_f$). Each local denoiser acts on region $A \cup B$ with diameter $2r+1$. Contours show different MSE values, where $\mathrm{MSE}=0.03$ (in orange) qualitatively represents the threshold of $T_f$, after which the denoised images are always significantly different from the original images for any $r$. 
    (b) Scan of denoised images within the phase transition window $[0.2, 0.5]$.
    \rev{(c) CMI $I(X_A:X_C|X_B)$ of $27 \times 27$ MNIST with $A$ being the non-central pixel on the seventh row and seventh column (the upper left quarter of the image, along the diagonal).}
    }
    \label{fig:mnist-apx}
\end{figure}

\subsection{Numerical results of Fashion-MNIST}
\label{apxsec:fashion-mnit}

In this section, we present a similar numerical result of the phase transition during the diffusion of Fashion-MNIST, a large database of fashion images. Just as MNIST images, Fashion-MNIST images are $28 \times 28$.
We select the pixels at the 15th row and the 15th column of the image to be $A$ so that $k=1$. Then, the CMI as a function of distance $r = \mathrm{dist} (A, C)$ at different time steps $t$ is shown in Fig.\,\ref{fig:mnist}c. 
In the limit case of $t=1$ and $t=0$, we observed that both CMI values are small even for a small distance $r$. These two distributions show the trivial phase and data phase, respectively. 
At $t_c \approx 0.3 \sim 0.4$, we still observe a significant CMI barrier, which indicates the existence of the phase transition.

    

This phase transition can also be validated via the CMI by testing the efficacy of local denoisers. 
All denoised images with different $T_f$ (forward dynamics duration) and $r$ (width of $B$) are depicted in Fig.\,\ref{fig:mnist}d
We observe that all local denoisers perform poorly when $T_f > 0.4$, as they consistently fail after the phase transition occurs.

\rev{

\section{Liquid-vapor-type Phase Transitions}
\label{apxsec:LVPT}

In this section, we give a toy example of a liquid-vapor-type phase transition.
Recall that a liquid-vapor-type phase transition is that: there are two paths connecting two distributions $P_0$ and $P_1$ such that there exists Markov length divergence in one path, but the Markov length is always finite along the other path (see Fig.\,\ref{fig:lvpt}). 

Suppose a 1D chain $\Lambda$ with $L$ sites, and each site supports a spin taking binary random values from $\{ 0, 1 \}$.
The data $X$, residing on sites, takes values from the sample space $\{ 0, 1 \}^L$.
For notational simplicity, we define the following three distributions: 
$P_{\mathrm{zero}}$ represents that $X$ is deterministically the all zero string, namely $P_{\mathrm{zero}}(x)$ iff $x=00\cdots0$; $P_{\mathrm{uni}}$ is the uniform distribution of all possible bit strings $P_{\mathrm{uni}}(x) = \frac{1}{2^L}$ for all $x \in \{0, 1\}^L$; $P_{\mathrm{even}}$ is the uniform distribution of all bit strings with even parity
\begin{equation}
    P_{\mathrm{even}}(x) =
    \left\{
    \begin{array}{cr}
        \frac{1}{2^{L-1}}, & \text{even number of 1's in }x, \\
        0, & \text{odd number of 1's in }x.
    \end{array}
    \right.
\end{equation}
An important fact of the distribution $P_{\mathrm{even}}$ is that: \textit{it exhibits a long-range CMI}.
To see this, we notice that $P_{\mathrm{even}}$ contains bit strings with even parity. For any bit $A$ and any tripartition $\Lambda = ABC$, it is impossible that $\Pr[X_A|X_B] = \Pr[X_A|X_B, X_C]$, because $X_A$ is determined if $X_B, X_C$ is fixed, but $X_A$ is totally random if the knowledge of $X_C$ is missing.
Therefore, for any tripartition $\Lambda = ABC$, we always have $I(X_A : X_C | X_B) \equiv 1$.
This even parity distribution is a special case of the well-known \textit{strong-to-weak spontaneous symmetry breaking} (SWSSB) states \cite{lee_2025_symmetry, lee_2023_quantum, lessa_2025_strong, kuno_2025_strong}.

To define the evolution dynamics, we introduce the transition matrix $S_p$ describing the single-bit flip with probability $p$, and $D_p$ describing the bit-pair flip with probability $p$:
\begin{equation}
  S_p = \left(\begin{array}{cc}
    1 - p & p\\
    p & 1 - p
  \end{array}\right) \quad \text{and} \quad
  D_p = \left(\begin{array}{cccc}
    1 - p & 0 & 0 & p\\
    0 & 1-p & p & 0 \\
    0 & p & 1 - p & 0 \\
    p & 0 & 0 & 1-p
  \end{array}\right) . \label{eq:rbf}
\end{equation}

Now, we suppose two noise path that transform from $P_0 = P_{\mathrm{zero}}$ to $P_1 = P_{\mathrm{uni}}$:
\begin{eqnarray}
    & P_{\mathrm{zero}} \xrightarrow{\text{independently flip each bit } \left( p = \frac{1}{2} \right)} P_{\mathrm{uni}}, & \label{eq:path1} \\
    & P_{\mathrm{zero}} \xrightarrow{\text{independently flip each nearest-neighbor pair } \left( p = \frac{1}{2} \right)} P_{\mathrm{even}} \xrightarrow{\text{independently flip each bit } \left( p = \frac{1}{2} \right)} P_{\mathrm{uni}}. & \label{eq:path2}
\end{eqnarray}
We remark that any bit-flip process described above can be realized in a time-continuous manner via local master-equation evolution. 
In fact, by taking the generator of $S_{\dt/2} = I + \dt \mathcal{L}_{\mathrm{bit}} + O(\dt^2)$, we obtain a master equation of bit-flip for each bit:
\begin{equation}
  \frac{\dd}{\dd t} \left(\begin{array}{c}
    P_{\mathrm{up}}\\
    P_{\mathrm{down}}
  \end{array}\right) = \frac{1}{2} \left(\begin{array}{cc}
    - 1 & 1\\
    1 & - 1
  \end{array}\right) \left(\begin{array}{c}
    P_{\mathrm{up}}\\
    P_{\mathrm{down}}
  \end{array}\right) = : \mathcal{L}_{\mathrm{bit}} P
    ~\Rightarrow~
  \left(\begin{array}{c}
    P_{\mathrm{up}} (t)\\
    P_{\mathrm{down}} (t)
  \end{array}\right) = \left(\begin{array}{cc}
    \frac{1 + e^{- t}}{2} & \frac{1 - e^{- t}}{2}\\
    \frac{1 - e^{- t}}{2} & \frac{1 + e^{- t}}{2}
  \end{array}\right) \left(\begin{array}{c}
    P_{\mathrm{up}}(0)\\
    P_{\mathrm{down}}(0)
  \end{array}\right) .
\end{equation}
This matches the transition matrix $S_p$ in Eq.\,(\ref{eq:rbf}) for $t = - \ln (1 - 2 p)$, and $p$ is approximately $1/2$ if we choose a large but fixed $t=T$. 
Similarly, the generator of $D_{\dt/2} = I + \dt \mathcal{L}_{\mathrm{pair}} + O(\dt^2)$ for each bit-pair is (and we have $t = - \ln (1 - 2 p)$ as well)
\begin{equation}
  \mathcal{L}_{\mathrm{pair}} = \frac{1}{2} \left(\begin{array}{cccc}
    -1 & 0 & 0 & 1\\
    0 & -1 & 1 & 0 \\
    0 & 1 & -1 & 0 \\
    1 & 0 & 0 & -1
  \end{array}\right) . \label{eq:rbf}
\end{equation}

For the first path Eq.\,(\ref{eq:path1}), the master equation $\partial_t P = \sum^{L}_{i=1} \mathcal{L}_{\mathrm{bit}, i} P$ is $1$-local, where $\mathcal{L}_{\mathrm{bit}, i}$ are $2^L$-by-$2^L$ matrices, characterizing the flip of bit $i$.
Both $P_{\mathrm{zero}}$ and $P_{\mathrm{uni}}$ are product distributions, and all bits are fully independent; thus, CMI is always zero throughout the process Eq.\,(\ref{eq:path1}).

However, for the second path Eq.\,(\ref{eq:path2}): in the first stage, the master equation $\partial_t P = \sum^{L-1}_{i=1} \mathcal{L}_{\mathrm{pair}, \langle i, i+1 \rangle} P$ consists of $2$-local interaction, where $\mathcal{L}_{\mathrm{pair}, \langle i, i+1 \rangle}$ are $2^L$-by-$2^L$ matrices characterizing the flip of the nearest-neighbor bit-pair $\langle i, i+1 \rangle$; and in the second stage, the master equation $\partial_t P = \sum^{L}_{i=1} \mathcal{L}_{\mathrm{bit}, i}  P$ is again $1$-local.
The dynamics are still governed by a local Fokker-Planck equation. However, the intermediate distribution $P_{\mathrm{even}}$ has a long-range CMI, which cannot be locally denoised for any tripartition $\Lambda = ABC$.
This proves that $P_{\mathrm{zero}}$ and $P_{\mathrm{uni}}$ are in the same phase, even if they can be connected through one noise path that crosses the phase boundary.

}

\begin{figure}[t]
    \centering
    \includegraphics[width=0.4\linewidth]{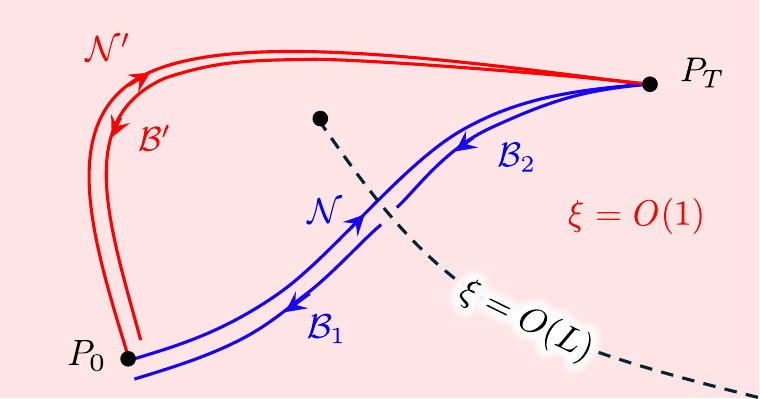
    }
    \caption{Schematic of data distribution phases for a liquid-vapor-type phase transition. There are two paths connecting the two distributions $P_0$ and $P_1$. The Markov length diverges in one path (blue), but the Markov length is always finite along the other path (red).
    The local Fokker-Planck evolution $\mathcal{N}'$ from $P_0$ to $P_T$ can be locally reversed by a $\mathcal{B}'$ if the Markov length remains finite along the forward path. If there is a Markov length divergence as in the case of $\mathcal{N}$, then local denoisers $\mathcal{B}_1$ and $\mathcal{B}_2$ exist on both sides of the phase boundary, but a global denoiser is required at the phase boundary.}
    \label{fig:lvpt}
\end{figure}

\section{Local Quantum Diffusion Models: Continuous-time Petz Maps}
\label{apxsec:cPetz}

For any quantum mixed state $\rho$ and quantum channel $\mathcal{N}$, it is well-known that the perfect recovery from $\mathcal{N}(\rho)$ to $\rho$ can be implemented by the \textit{Petz map} \cite{petz_sufficient_1986}
\begin{equation}
    \mathcal{P}_{\mathcal{N}, \rho}(\sigma) = \rho^{1/2} \mathcal{N}^{\dagger}(\mathcal{N} (\rho)^{-1/2} \sigma \mathcal{N} (\rho)^{-1/2}) \rho^{1/2}. 
\end{equation}
One can verify that $\mathcal{P}_{\mathcal{N},\rho}(\mathcal{N}(\rho)) = \rho$. In this sense, Petz map is regarded as a quantum version of Bayes formula \cite{leifer_towards_2013, khatri_principles_2020}. 
However, unlike in the classical case where the Bayes map is the unique perfect recovery channel, the Petz map is not the only perfect recovery channel. In fact, given any $\theta$, one can introduce an isometric map $\mathcal{U}_{\rho,\theta}(\sigma)=\rho^{\mathi \theta} \sigma \rho^{-\mathi \theta}$ to define a \textit{rotated Petz map}
\begin{equation}
    \mathcal{R}_{\mathcal{N},\rho,\theta} = \mathcal{U}_{\rho,-\theta/2} \circ \mathcal{P}_{\mathcal{N}, \rho} \circ \mathcal{U}_{\mathcal{N}(\rho),\theta/2}.
\end{equation}
It is easy to verify that $\mathcal{R}_{\mathcal{N},\rho, \theta}( \mathcal{N}(\rho) ) = \rho$ are also perfect recovery channels \cite{mark_quantum_2016}. 

However, neither $\mathcal{P}_{\mathcal{N},\rho}$ nor $\mathcal{R}_{\mathcal{N},\rho, \theta}$ is local because $\rho^{1/2}$ is essentially a very global quantity. To construct a local recovery map, it was found that a map called \textit{twirled Petz map}
\begin{equation}
    \mathcal{T}_{\mathcal{N},\rho} : = \int \dd \theta f(\theta) \mathcal{R}_{\mathcal{N},\rho, \theta} \label{apxeq:TPRM}
\end{equation}
can be leveraged to construct a local reversal channel \cite{junge_universal_2018}. Here, $f (\theta) = \pi / (2 \cosh (\pi \theta) + 2)$ is a probability distribution function of angles $\theta$. It was shown that if the quantum channel $\mathcal{N}$ acts only on $A$, then twirled Petz map with local reference state $\rho_{AB}$ yields a recovery error at most $| \mathcal{T}_{\mathcal{N},\rho_{AB}} \circ \mathcal{N} (\rho) - \rho |^2_1 \leq 2 \ln 2 \cdot I(A:B|C)_{\rho}$ \cite{junge_universal_2018, mark_quantum_2016}. 

Given all the essential background of local Lindbladian reversibility in open quantum systems, we will first give the expression of the continuous-time twirled Petz map in SM \ref{apxsec:cPetz}, and then we will prove that the continuous-time twirled Petz map is a quantum generalization of diffusion models in SM \ref{apxsec:classical_limit}.

Let us consider any Lindblad equation $\dot{\rho} = \mathcal{L} (\rho) = \mathcal{D} [a] \rho$ where $ \mathcal{D} [a] \rho = a \rho a^{\dagger} - \frac{1}{2} (a^{\dagger} a \rho + \rho a^{\dagger} a)$ and $\rho = \sum_i P_i \ket{i} \bra{i}$, the continuous time limit of any rotated Petz map $\mathcal{R}_{e^{\dt \mathcal{L}},\rho, \theta} (\sigma)$ must have form of
\begin{equation}
    \dot{\sigma} = - \mathi [ R_{\theta}, \sigma] + \mathcal{D} [b_{\theta}] \sigma,
\end{equation}
where $b_{\theta} = \rho^{(1 - \mathi \theta)/2} a^{\dagger} \rho^{(- 1 + \mathi \theta)/2}$ is the backward jump operator and $R_{\theta}$ is the backward Hamiltonian
\begin{equation}
    R_{\theta} = \mathi \sum_{i, j} \! \frac{2 P_i^{\frac{1 + \mathi \theta}{2}} \!\! P_j^{\frac{1 - \mathi \theta}{2}} \!-\! P_i \!-\! P_j}{2 (P_i \!-\! P_j)} \! \bra{i} b_{\theta}^{\dagger} b_{\theta} \! \ket{j} \! \ket{i} \! \bra{j} + \mathi \sum_{i, j} \! \frac{2 P_i^{\frac{1 - \mathi \theta}{2}} \!\! P_j^{\frac{1 + \mathi \theta}{2}} \!-\! P_i \!-\! P_j}{2 (P_i \!-\! P_j)} \! \bra{i} a^{\dagger} a \ket{j} \ket{i} \! \bra{j}.
\end{equation}

This result can be utilized to construct a quantum version of diffusion models -- \textit{local quantum diffusion models}.
Suppose any jump operator $a$ only acting on $A$, we have the following result:
\begin{theorem}[\textbf{Lindbladian of local quantum denoisers}]
    \label{thm:lqd}
    Suppose a forward quantum diffusion process $\dot{\rho} = \mathcal{L} (\rho) = \mathcal{D} [a] \rho$ (with $a$ only acting on local $A$) and the eigen-decomposition of the reduced density matrix $\rho_{AB, t} = \sum_i P_i \ket{i} \bra{i}$, the continuous time limit of any rotated Petz map $\mathcal{R}_{e^{\dt \mathcal{L}}, \rho_{AB, t}, \theta} (\sigma)$ must have a form of time-dependent Lindbladian
    \begin{equation}
        \dot{\sigma} = - \mathi [ R_{AB, \theta} (t), \sigma] + \mathcal{D} [b_{AB, \theta}(t)] \sigma,
    \end{equation}
    where $b_{AB, \theta} (t) = \rho_{AB,t}^{(1 - \mathi \theta)/2} a^{\dagger} \rho_{AB,t}^{(- 1 + \mathi \theta)/2}$ is the local backward jump operator and $R_{AB, \theta}$ is the local backward Hamiltonian
    \begin{equation}
        R_{AB, \theta} (t) = \mathi \sum_{i, j} \! \frac{2 P_i^{\frac{1 + \mathi \theta}{2}} \!\! P_j^{\frac{1 - \mathi \theta}{2}} \!-\! P_i \!-\! P_j}{2 (P_i \!-\! P_j)} \! \bra{i} b_{AB, \theta}^{\dagger} b_{AB, \theta} \! \ket{j} \! \ket{i} \! \bra{j} + \mathi \sum_{i, j} \! \frac{2 P_i^{\frac{1 - \mathi \theta}{2}} \!\! P_j^{\frac{1 + \mathi \theta}{2}} \!-\! P_i \!-\! P_j}{2 (P_i \!-\! P_j)} \! \bra{i} a^{\dagger} a \ket{j} \ket{i} \! \bra{j}. \label{apxeq:R_theta}
    \end{equation}
\end{theorem}
All the remaining part of the SM \ref{apxsec:cPetz} is about the derivation of the result in Theorem \ref{thm:lqd}. 

We note that quantum versions of diffusion models have been previously studied in the literature, e.g.\,the quantum denoising diffusion probabilistic models (QuDDPM) \cite{zhang_generative_2024}. We also refer to a closely related work that also generalizes diffusion models in the quantum regime through the Petz map \cite{liu2025measurement}.

\subsection{Continuous-time Petz map}

Suppose a quantum channel with infinitesimal time $\epsilon \to 0$
\begin{equation}
    \mathcal{N} (\rho) = e^{\epsilon \mathcal{L}} \rho.
\end{equation}
We can introduce the \textit{Petz map}:
\begin{equation}
    \mathcal{P}_{\mathcal{N}, \rho} (\sigma) = \rho^{\frac{1}{2}} \mathcal{N}^{\dagger} \left( \mathcal{N} (\rho)^{- \frac{1}{2}} \sigma \mathcal{N} (\rho)^{- \frac{1}{2}} \right) \rho^{\frac{1}{2}} .
\end{equation}
From now on, without loss of generality, we assume $\rho$ has eigen-decomposition $\rho \ket{i} = P_i \ket{i}$ with $P_i > 0$ for all $i$. This appendix sub-section aims to compute $\left. \frac{\partial}{\partial \epsilon} (\mathcal{P}_{\mathcal{N}, \rho} (\sigma)) \right|_{\epsilon = 0}$ for the given $\mathcal{N}= e^{\epsilon \mathcal{L}}$.

\subsubsection{Derivative of $\mathcal{N} (\rho)^{- \frac{1}{2}}$}

Now we let $\chi =\mathcal{N} (\rho)^{\frac{1}{2}}$, namely $\chi^2=\mathcal{N} (\rho)$. We notice that $ \chi |_{\epsilon = 0} =
\rho^{\frac{1}{2}}$ and
\begin{equation}
    \left. \frac{\partial}{\partial \epsilon} (\chi^{- 1}) \right|_{\epsilon = 0} = - ( \chi |_{\epsilon = 0})^{- 1} \left( \left. \frac{\partial \chi}{\partial \epsilon} \right|_{\epsilon = 0} \right) ( \chi |_{\epsilon = 0})^{- 1} = - \rho^{- \frac{1}{2}} \left( \left. \frac{\partial \chi}{\partial \epsilon} \right|_{\epsilon = 0} \right) \rho^{- \frac{1}{2}} .
\end{equation}
Then we only need to compute $\left. \frac{\partial \chi}{\partial \epsilon}\right|_{\epsilon = 0}$. Since $\chi^2 =\mathcal{N} (\rho)$, we have
\begin{equation}
    \chi \frac{\partial \chi}{\partial \epsilon} + \frac{\partial \chi}{\partial \epsilon} \chi = \frac{\partial}{\partial \epsilon} (\mathcal{N} (\rho)) .
\end{equation}
Here comes to the symmetric division at $\epsilon = 0$: the relation $\frac{1}{2} \left\{ \rho^{\frac{1}{2}}, \left. \frac{\partial \chi}{\partial \epsilon} \right|_{\epsilon = 0} \right\} = \frac{1}{2} \left. \frac{\partial}{\partial \epsilon} (\mathcal{N} (\rho)) \right|_{\epsilon = 0}$ implies
\begin{equation}
    \left. \frac{\partial \chi}{\partial \epsilon} \right|_{\epsilon = 0} = L_{\rho^{1 / 2}} \left( \frac{1}{2} \mathcal{L} (\rho) \right),
\end{equation}
where $Z = L_X (Y)$ is the well-known \textit{symmetric division} which uniquely satisfies $\frac{1}{2} \{ X, Z \} = Y$. More explicitly, for any $X$ with eigen-decomposition $X \ket{i} = \lambda_i \ket{i}$,
\begin{equation}
    L_X (Y) := \sum_{i, j} \frac{2}{\lambda_i + \lambda_j} \bra{i} Y \ket{j} \ket{i} \bra{j} .
\end{equation}
Then, after obtaining $\left. \frac{\partial \chi}{\partial \epsilon} \right|_{\epsilon = 0}$, we also immediately have
\begin{equation}
    \left. \frac{\partial}{\partial \epsilon} \left( \mathcal{N} (\rho)^{- \frac{1}{2}} \right) \right|_{\epsilon = 0} = - \rho^{- \frac{1}{2}} L_{\rho^{1 / 2}} \left( \frac{1}{2} \mathcal{L} (\rho) \right) \rho^{- \frac{1}{2}},
\end{equation}
or equivalently,
\begin{equation}
    \mathcal{N} (\rho)^{- \frac{1}{2}} = \rho^{- \frac{1}{2}} - \epsilon \rho^{- \frac{1}{2}} L_{\rho^{1 / 2}} \left( \frac{1}{2} \mathcal{L} (\rho) \right) \rho^{- \frac{1}{2}} + O (\epsilon^2) .
\end{equation}

\subsubsection{Derivative of $\mathcal{N}^{\dagger}$}

This part is easy. For any operator $\tau$, we have
\begin{equation}
    \left. \frac{\partial}{\partial \epsilon} (\mathcal{N}^{\dagger} (\tau)) \right|_{\epsilon = 0} =\mathcal{L}^{\dagger} (\tau) = a^{\dagger} \tau a - \frac{1}{2} (a^{\dagger} a \tau + \tau a^{\dagger} a) .
\end{equation}

\subsubsection{Derivative of $\mathcal{P}_{\mathcal{N}, \rho}$}

Now we can expand $\mathcal{P}_{\mathcal{N}, \rho} (\sigma)$ into
\begin{align}
    \mathcal{P}_{\mathcal{N}, \rho} (\sigma) =~& \rho^{\frac{1}{2}} \left( \mathcal{N} (\rho)^{- \frac{1}{2}} \sigma \mathcal{N} (\rho)^{- \frac{1}{2}} + \epsilon \mathcal{L}^{\dagger} \left( \mathcal{N} (\rho)^{- \frac{1}{2}} \sigma \mathcal{N} (\rho)^{- \frac{1}{2}} \right) + O (\epsilon^2) \right) \rho^{\frac{1}{2}} \nonumber\\
    =~& \rho^{\frac{1}{2}} \left( \rho^{- \frac{1}{2}} - \epsilon \rho^{- \frac{1}{2}} L_{\rho^{1 / 2}} \left( \frac{1}{2} \mathcal{L} (\rho) \right) \rho^{- \frac{1}{2}} + O (\epsilon^2) \right) \sigma \left( \rho^{- \frac{1}{2}} - \epsilon \rho^{- \frac{1}{2}} L_{\rho^{1 / 2}} \left( \frac{1}{2} \mathcal{L} (\rho) \right) \rho^{- \frac{1}{2}} + O (\epsilon^2) \right) \rho^{\frac{1}{2}} \nonumber\\
    & + \epsilon \rho^{\frac{1}{2}} \mathcal{L}^{\dagger} \left( \mathcal{N} (\rho)^{- \frac{1}{2}} \sigma \mathcal{N} (\rho)^{- \frac{1}{2}} \right) \rho^{\frac{1}{2}} + O (\epsilon^2) \nonumber\\
    =~& \sigma + \epsilon \left( - L_{\rho^{1 / 2}} \left( \frac{1}{2} \mathcal{L} (\rho) \right) \rho^{- \frac{1}{2}} \sigma - \sigma \rho^{- \frac{1}{2}} L_{\rho^{1 / 2}} \left( \frac{1}{2} \mathcal{L} (\rho) \right) + \rho^{\frac{1}{2}} \mathcal{L}^{\dagger} \left( \rho^{- \frac{1}{2}} \sigma \rho^{- \frac{1}{2}} \right) \rho^{\frac{1}{2}} \right) + O (\epsilon^2) . 
\end{align}
Therefore,
\begin{equation}
    \left. \frac{\partial}{\partial \epsilon} (\mathcal{P}_{\mathcal{N}, \rho} (\sigma)) \right|_{\epsilon = 0} = - L_{\rho^{1 / 2}} \left( \frac{1}{2} \mathcal{L} (\rho) \right) \rho^{- \frac{1}{2}} \sigma - \sigma \rho^{- \frac{1}{2}} L_{\rho^{1 / 2}} \left( \frac{1}{2} \mathcal{L} (\rho) \right) + \rho^{\frac{1}{2}} \mathcal{L}^{\dagger} \left( \rho^{- \frac{1}{2}} \sigma \rho^{- \frac{1}{2}} \right) \rho^{\frac{1}{2}} .
\end{equation}
Let us do some further simplification, by introducing $b = \rho^{\frac{1}{2}} a^{\dagger} \rho^{- \frac{1}{2}}$:
\begin{align}
    \rho^{\frac{1}{2}} \mathcal{L}^{\dagger} \left( \rho^{- \frac{1}{2}} \sigma \rho^{- \frac{1}{2}} \right) \rho^{\frac{1}{2}} & = \rho^{\frac{1}{2}} \left( a^{\dagger} \rho^{- \frac{1}{2}} \sigma \rho^{- \frac{1}{2}} a - \frac{1}{2} \left( a^{\dagger} a \rho^{- \frac{1}{2}} \sigma \rho^{- \frac{1}{2}} + \rho^{- \frac{1}{2}} \sigma \rho^{- \frac{1}{2}} a^{\dagger} a \right) \right) \rho^{\frac{1}{2}} \nonumber\\
    & = \rho^{\frac{1}{2}} a^{\dagger} \rho^{- \frac{1}{2}} \sigma \rho^{- \frac{1}{2}} a \rho^{\frac{1}{2}} - \frac{1}{2} \left( \rho^{\frac{1}{2}} a^{\dagger} a \rho^{- \frac{1}{2}} \sigma + \sigma \rho^{- \frac{1}{2}} a^{\dagger} a \rho^{\frac{1}{2}} \right) \nonumber\\
    & = b \sigma b^{\dagger} - \frac{1}{2} \left( \rho^{\frac{1}{2}} a^{\dagger} a \rho^{- \frac{1}{2}} \sigma + \sigma \rho^{- \frac{1}{2}} a^{\dagger} a \rho^{\frac{1}{2}} \right) . 
\end{align}
We define a Hermitian $R$ satisfying
\begin{align}
    - \mathi R =~& - L_{\rho^{1 / 2}} \left( \frac{1}{2} \mathcal{L} (\rho) \right) \rho^{- \frac{1}{2}} - \frac{1}{2} \rho^{\frac{1}{2}} a^{\dagger} a \rho^{- \frac{1}{2}} + \frac{1}{2} b^{\dagger} b \nonumber\\
    =~& - L_{\rho^{1 / 2}} \left( \frac{1}{2} \mathcal{L} (\rho) \rho^{- \frac{1}{2}} \right) - \frac{1}{2} \rho^{\frac{1}{2}} a^{\dagger} a \rho^{- \frac{1}{2}} + \frac{1}{2} \rho^{- \frac{1}{2}} a \rho a^{\dagger} \rho^{- \frac{1}{2}} \nonumber\\
    =~& - L_{\rho^{1 / 2}} \left( \frac{1}{2} \mathcal{L} (\rho) \rho^{- \frac{1}{2}} \right) + L_{\rho^{1 / 2}} \left( \frac{1}{2} \left\{ \rho^{\frac{1}{2}}, - \frac{1}{2} \rho^{\frac{1}{2}} a^{\dagger} a \rho^{- \frac{1}{2}} + \frac{1}{2} \rho^{- \frac{1}{2}} a \rho a^{\dagger} \rho^{- \frac{1}{2}} \right\} \right) \nonumber\\
    =~& L_{\rho^{1 / 2}} \left( - \frac{1}{2} a \rho a^{\dagger} \rho^{- \frac{1}{2}} + \frac{1}{4} a^{\dagger} a \rho^{\frac{1}{2}} + \frac{1}{4} \rho a^{\dagger} a \rho^{- \frac{1}{2}} - \frac{1}{4} \rho a^{\dagger} a \rho^{- \frac{1}{2}} + \frac{1}{4} a \rho a^{\dagger} \rho^{- \frac{1}{2}} - \frac{1}{4} \rho^{\frac{1}{2}} a^{\dagger} a + \frac{1}{4} \rho^{- \frac{1}{2}} a \rho a^{\dagger} \right) \nonumber\\
    =~& L_{\rho^{1 / 2}} \left( - \frac{1}{4} a \rho a^{\dagger} \rho^{- \frac{1}{2}} + \frac{1}{4} a^{\dagger} a \rho^{\frac{1}{2}} - \frac{1}{4} \rho^{\frac{1}{2}} a^{\dagger} a + \frac{1}{4} \rho^{- \frac{1}{2}} a \rho a^{\dagger} \right) \nonumber\\
    =~& \frac{1}{4} L_{\rho^{1 / 2}} \left( \mathi \left[ \rho^{\frac{1}{2}}, a^{\dagger} a + \rho^{- \frac{1}{2}} a \rho a^{\dagger} \rho^{- \frac{1}{2}} \right] \right),
\end{align}
or more explicitly,
\begin{equation}
    R = - \frac{\mathi}{2} \sum_{i, j} \frac{\sqrt{P_i} - \sqrt{P_j}}{\sqrt{P_i} + \sqrt{P_j}} \bra{i} a^{\dagger} a + b^{\dagger} b \ket{j} \ket{i} \bra{j} .
\end{equation}
Therefore,
\begin{align}
    \left. \frac{\partial}{\partial \epsilon} (\mathcal{P}_{\mathcal{N}, \rho} (\sigma)) \right|_{\epsilon = 0} & = - \mathi R \sigma + \mathi \sigma R + b \sigma b^{\dagger} - \frac{1}{2} b^{\dagger} b \sigma - \frac{1}{2} \sigma b^{\dagger} b \nonumber\\
    & = - \mathi [R, \sigma] +\mathcal{D} [b] \sigma . 
\end{align}
This derivation also shows that the jump operator of Petz map must be $b = \rho^{\frac{1}{2}} a^{\dagger} \rho^{- \frac{1}{2}}$.

\subsection{Continuous-time twirled Petz map}

The \textit{twirled Petz map} is defined as
\begin{equation}
    \mathcal{T}_{\mathcal{N}, \rho} (\sigma) = \int_{- \infty}^{\infty} f (\theta) \rho^{\frac{1 - \mathi \theta}{2}} \mathcal{N}^{\dagger} \left[ \mathcal{N} (\rho)^{\frac{- 1 + \mathi \theta}{2}} \sigma \mathcal{N} (\rho)^{\frac{- 1 - \mathi \theta}{2}} \right] \rho^{\frac{1 + \mathi \theta}{2}} ,
\end{equation}
where $f (\theta) = \frac{\pi}{2 (\cosh (\pi \theta) + 1)}$. Also, we let $\left. \frac{\partial}{\partial \epsilon} (\mathcal{N} (\rho)) \right|_{\epsilon = 0} =\mathcal{L} (\rho) =\mathcal{D} [a] \rho = a \rho a^{\dagger} - \frac{1}{2} (a^{\dagger} a \rho + \rho a^{\dagger} a)$. We can re-write $\mathcal{T}_{\mathcal{N}, \rho}$ into
\begin{equation}
    \mathcal{T}_{\mathcal{N}, \rho} (\sigma) = \int_{- \infty}^{\infty} \dd \theta \, f (\theta) \mathcal{R}_{\mathcal{N}, \rho, \theta} (\sigma),
\end{equation}
where $\mathcal{R}_{\mathcal{N}, \rho} (\sigma) = \rho^{\frac{1 - \mathi \theta}{2}} \mathcal{N}^{\dagger} \left[ \mathcal{N} (\rho)^{\frac{- 1 + \mathi \theta}{2}} \sigma \mathcal{N} (\rho)^{\frac{- 1 - \mathi \theta}{2}} \right] \rho^{\frac{1 + \mathi \theta}{2}}$ is called the \textit{rotated Petz map}. This appendix sub-section aims to compute $\left. \frac{\partial}{\partial \epsilon} (\mathcal{T}_{\mathcal{N}, \rho} (\sigma)) \right|_{\epsilon = 0}$ for the given $\mathcal{N}= e^{\epsilon \mathcal{L}}$.

\subsubsection{Derivative of $\mathcal{N} (\rho)^{\frac{- 1 + \mathi \theta}{2}}$}

Now we let $\chi_{\theta} =\mathcal{N} (\rho)^{\frac{1 - \mathi \theta}{2}}$, namely $\chi_{\theta} \chi^{\dagger}_{\theta} =\mathcal{N} (\rho)$. We notice that $ \chi_{\theta} |_{\epsilon = 0} = \rho^{\frac{1 - \mathi \theta}{2}}$ and
\begin{equation}
    \left. \frac{\partial}{\partial \epsilon} (\chi_{\theta}^{- 1}) \right|_{\epsilon = 0} = - ( \chi_{\theta} |_{\epsilon = 0})^{- 1} \left( \left. \frac{\partial \chi_{\theta}}{\partial \epsilon} \right|_{\epsilon = 0} \right) ( \chi_{\theta} |_{\epsilon = 0})^{- 1} = - \rho^{\frac{- 1 + \mathi \theta}{2}} \left( \left. \frac{\partial \chi_{\theta}}{\partial \epsilon} \right|_{\epsilon = 0} \right) \rho^{\frac{- 1 + \mathi \theta}{2}} .
\end{equation}
Then we only need to compute $\kappa_{\theta} = \left. \frac{\partial \chi_{\theta}}{\partial \epsilon} \right|_{\epsilon = 0}$. Since $\chi_{\theta} \chi^{\dagger}_{\theta} = \chi^{\dagger}_{\theta} \chi_{\theta} =\mathcal{N} (\rho)$, we have
\begin{eqnarray}
    \chi_{\theta} \frac{\partial \chi^{\dagger}_{\theta}}{\partial \epsilon} + \frac{\partial \chi_{\theta}}{\partial \epsilon} \chi^{\dagger}_{\theta} & = \frac{\partial}{\partial \epsilon} (\mathcal{N} (\rho)), \\
    \chi^{\dagger}_{\theta} \frac{\partial \chi_{\theta}}{\partial \epsilon} + \frac{\partial \chi^{\dagger}_{\theta}}{\partial \epsilon} \chi_{\theta} & = \frac{\partial}{\partial \epsilon} (\mathcal{N} (\rho)) . 
\end{eqnarray}
Let $\epsilon = 0$, we get
\begin{align}
    \rho^{\frac{1 - \mathi \theta}{2}} \kappa_{\theta}^{\dagger} + \kappa_{\theta} \rho^{\frac{1 + \mathi \theta}{2}} & = \mathcal{L} (\rho), \\ 
    \rho^{\frac{1 + \mathi \theta}{2}} \kappa_{\theta} + \kappa_{\theta}^{\dagger} \rho^{\frac{1 - \mathi \theta}{2}} & = \mathcal{L} (\rho). 
\end{align}
For $\rho = \sum P_i \ket{i} \bra{i}$, we have (notice that $\bra{i} \rho^{\frac{1 - \mathi \theta}{2}} = P_i^{\frac{1 - \mathi \theta}{2}} \bra{i}$)
\begin{align}
    P_i^{\frac{1 - \mathi \theta}{2}} \bra{i} \kappa_{\theta}^{\dagger} \ket{j} + P_j^{\frac{1 + \mathi \theta}{2}} \bra{i} \kappa_{\theta} \ket{j} & = \bra{i} \mathcal{L} (\rho) \ket{j}, \\
    P_i^{\frac{1 + \mathi \theta}{2}} \bra{i} \kappa_{\theta} \ket{j} + P_j^{\frac{1 - \mathi \theta}{2}} \bra{i} \kappa_{\theta}^{\dagger} \ket{j}  & = \bra{i} \mathcal{L} (\rho) \ket{j} . 
\end{align}
The solution is
\begin{equation}
    \bra{i} \kappa_{\theta} \ket{j} = \frac{P_i^{\frac{1 - \mathi \theta}{2}} - P_j^{\frac{1 - \mathi \theta}{2}}}{P_i - P_j} \bra{i} \mathcal{L} (\rho) \ket{j} .
\end{equation}

Let us define
\begin{equation}
    L_{\rho^{1 / 2}, \theta} (X) = \sum_{i, j} \frac{P_i^{\frac{1 - \mathi \theta}{2}} - P_j^{\frac{1 - \mathi \theta}{2}}}{P_i - P_j} \bra{i} X \ket{j} \ket{i} \bra{j},
\end{equation}
such that $\kappa_{\theta} = L_{\rho^{1 / 2}, \theta} (\mathcal{L} (\rho))$. Finally,
\begin{equation}
    \left. \frac{\partial}{\partial \epsilon} \left( \mathcal{N} (\rho)^{\frac{- 1 \pm \mathi \theta}{2}} \right) \right|_{\epsilon = 0} = - L_{\rho^{1 / 2}, \pm \theta} \left( \rho^{\frac{- 1 \pm \mathi \theta}{2}} \mathcal{L} (\rho) \rho^{\frac{- 1 \pm \mathi \theta}{2}} \right),
\end{equation}
or equivalently,
\begin{equation}
    \mathcal{N} (\rho)^{\frac{- 1 \pm \mathi \theta}{2}} = \rho^{\frac{- 1 \pm \mathi \theta}{2}} - \epsilon L_{\rho^{1 / 2}, \pm \theta} \left( \rho^{\frac{- 1 \pm \mathi \theta}{2}} \mathcal{L} (\rho) \rho^{\frac{- 1 \pm \mathi \theta}{2}} \right) + O (\epsilon^2) .
\end{equation}

\subsubsection{Derivative of $\mathcal{T}_{\mathcal{N}, \rho}$}

Let $\mathcal{T}_{\mathcal{N}, \rho} (\sigma) = \int_{- \infty}^{\infty} \dd \theta \, f (\theta) \mathcal{R}_{\mathcal{N}, \rho, \theta} (\sigma)$, where $\mathcal{R}_{\mathcal{N}, \rho} (\sigma)$ is the rotated Petz map,
\begin{align}
    & \mathcal{R}_{\mathcal{N}, \rho, \theta} (\sigma) \nonumber\\
    =~& \rho^{\frac{1 - \mathi \theta}{2}} \left( \mathcal{N} (\rho)^{\frac{- 1 + \mathi \theta}{2}} \sigma \mathcal{N} (\rho)^{\frac{- 1 - \mathi \theta}{2}} + \epsilon \mathcal{L}^{\dagger} \left( \mathcal{N} (\rho)^{\frac{- 1 + \mathi \theta}{2}} \sigma \mathcal{N} (\rho)^{\frac{- 1 - \mathi \theta}{2}} \right) + O (\epsilon^2) \right) \rho^{\frac{1 + \mathi \theta}{2}} \nonumber\\
    =~& \rho^{\frac{1 - \mathi \theta}{2}} \left( \rho^{\frac{- 1 + \mathi \theta}{2}} - \epsilon L_{\rho^{1 / 2}, \theta} \left( \rho^{\frac{- 1 + \mathi \theta}{2}} \mathcal{L} (\rho) \rho^{\frac{- 1 + \mathi \theta}{2}} \right) + O (\epsilon^2) \right) \sigma \left( \rho^{\frac{- 1 - \mathi \theta}{2}} - \epsilon L_{\rho^{1 / 2}, - \theta} \left( \rho^{\frac{- 1 - \mathi \theta}{2}} \mathcal{L} (\rho) \rho^{\frac{- 1 - \mathi \theta}{2}} \right) + O (\epsilon^2) \right) \rho^{\frac{1 + \mathi \theta}{2}} \nonumber\\
    & + \epsilon \rho^{\frac{1 - \mathi \theta}{2}} \mathcal{L}^{\dagger} \left( \mathcal{N} (\rho)^{\frac{- 1 + \mathi \theta}{2}} \sigma \mathcal{N} (\rho)^{\frac{- 1 - \mathi \theta}{2}} \right) \rho^{\frac{1 + \mathi \theta}{2}} + O (\epsilon^2) \nonumber\\
    =~& \sigma + \epsilon \left( - L_{\rho^{1 / 2}, \theta} \left( \mathcal{L} (\rho) \rho^{\frac{- 1 + \mathi \theta}{2}} \right) \sigma - \sigma L_{\rho^{1 / 2}, - \theta} \left( \rho^{\frac{- 1 - \mathi \theta}{2}} \mathcal{L} (\rho) \right) + \rho^{\frac{1 - \mathi \theta}{2}} \mathcal{L}^{\dagger} \left( \rho^{\frac{- 1 + \mathi \theta}{2}} \sigma \rho^{\frac{- 1 - \mathi \theta}{2}} \right) \rho^{\frac{1 + \mathi \theta}{2}} \right) + O (\epsilon^2) . 
\end{align}
Let us do some further simplification, by introducing $b_{\theta} = \rho^{\frac{1 - \mathi \theta}{2}} a^{\dagger} \rho^{\frac{- 1 + \mathi \theta}{2}}$ and $b^{\dagger}_{\theta} = \rho^{\frac{- 1 - \mathi \theta}{2}} a \rho^{\frac{1 + \mathi \theta}{2}}$:
\begin{align}
    \rho^{\frac{1 - \mathi \theta}{2}} \mathcal{L}^{\dagger} \left( \rho^{\frac{- 1 + \mathi \theta}{2}} \sigma \rho^{\frac{- 1 - \mathi \theta}{2}} \right) \rho^{\frac{1 + \mathi \theta}{2}} & = \rho^{\frac{1 - \mathi \theta}{2}} \left( a^{\dagger} \rho^{\frac{- 1 + \mathi \theta}{2}} \sigma \rho^{\frac{- 1 - \mathi \theta}{2}} a - \frac{1}{2} \left( a^{\dagger} a \rho^{\frac{- 1 + \mathi \theta}{2}} \sigma \rho^{\frac{- 1 - \mathi \theta}{2}} + \rho^{\frac{- 1 + \mathi \theta}{2}} \sigma \rho^{\frac{- 1 - \mathi \theta}{2}} a^{\dagger} a \right) \right) \rho^{\frac{1 + \mathi \theta}{2}} \nonumber\\
    & = \rho^{\frac{1 - \mathi \theta}{2}} a^{\dagger} \rho^{\frac{- 1 + \mathi \theta}{2}} \sigma \rho^{\frac{- 1 - \mathi \theta}{2}} a \rho^{\frac{1 + \mathi \theta}{2}} - \frac{1}{2} \left( \rho^{\frac{1 - \mathi \theta}{2}} a^{\dagger} a \rho^{\frac{- 1 + \mathi \theta}{2}} \sigma + \sigma \rho^{\frac{- 1 - \mathi \theta}{2}} a^{\dagger} a \rho^{\frac{1 + \mathi \theta}{2}} \right) \nonumber\\
    & = b_{\theta} \sigma b_{\theta}^{\dagger} - \frac{1}{2} \left( \rho^{\frac{1 - \mathi \theta}{2}} a^{\dagger} a \rho^{\frac{- 1 + \mathi \theta}{2}} \sigma + \sigma \rho^{\frac{- 1 - \mathi \theta}{2}} a^{\dagger} a \rho^{\frac{1 + \mathi \theta}{2}} \right) . 
\end{align}
We define a Hermitian $R_{\theta}$ satisfying
\begin{align}
    R_{\theta} =~& - \mathi L_{\rho^{1 / 2}, \theta} \left( \mathcal{L} (\rho) \rho^{\frac{- 1 + \mathi \theta}{2}} \right) - \frac{\mathi}{2} \rho^{\frac{1 - \mathi \theta}{2}} a^{\dagger} a \rho^{\frac{- 1 + \mathi \theta}{2}} + \frac{\mathi}{2} b_{\theta}^{\dagger} b_{\theta} \nonumber\\
    =~& - \mathi L_{\rho^{1 / 2}, \theta} \left( \mathcal{L} (\rho) \rho^{\frac{- 1 + \mathi \theta}{2}} \right) - \frac{\mathi}{2} \rho^{\frac{1 - \mathi \theta}{2}} a^{\dagger} a \rho^{\frac{- 1 + \mathi \theta}{2}} + \frac{\mathi}{2} \rho^{\frac{- 1 - \mathi \theta}{2}} a \rho a^{\dagger} \rho^{\frac{- 1 + \mathi \theta}{2}} \nonumber\\
    =~& - \mathi \left( \sum_{i, j} \frac{P_i^{\frac{1 - \mathi \theta}{2}} P_j^{\frac{- 1 + \mathi \theta}{2}} - 1}{P_i - P_j} \bra{i} \mathcal{L} (\rho) \ket{j} \ket{i} \bra{j} \right) - \frac{\mathi}{2} \rho^{\frac{1 - \mathi \theta}{2}} a^{\dagger} a \rho^{\frac{- 1 + \mathi \theta}{2}} + \frac{\mathi}{2} \rho^{\frac{- 1 - \mathi \theta}{2}} a \rho a^{\dagger} \rho^{\frac{- 1 + \mathi \theta}{2}} \nonumber\\
    =~& \mathi \sum_{i, j} \left( - \frac{P_i^{\frac{1 - \mathi \theta}{2}} P_j^{\frac{- 1 + \mathi \theta}{2}} - 1}{P_i - P_j} + \frac{1}{2} P_i^{\frac{- 1 - \mathi \theta}{2}} P_j^{\frac{- 1 + \mathi \theta}{2}} \right) \bra{i} a \rho a^{\dagger} \ket{j} \ket{i} \bra{j} \nonumber\\
    & + \mathi \sum_{i, j} \left( \frac{P_i^{\frac{1 - \mathi \theta}{2}} P_j^{\frac{- 1 + \mathi \theta}{2}} - 1}{P_i - P_j} \cdot \frac{P_i + P_j}{2} - \frac{1}{2} P_i^{\frac{1 - \mathi \theta}{2}} P_j^{\frac{- 1 + \mathi \theta}{2}} \right) \bra{i} a^{\dagger} a \ket{j} \ket{i} \bra{j} \nonumber\\
    =~& \mathi \sum_{i, j} \frac{2 - P_i^{\frac{1 - \mathi \theta}{2}} P_j^{\frac{- 1 + \mathi \theta}{2}} - P_i^{\frac{- 1 - \mathi \theta}{2}} P_j^{\frac{1 + \mathi \theta}{2}}}{2 (P_i - P_j)} \bra{i} a \rho a^{\dagger} \ket{j} \ket{i} \bra{j} + \mathi \sum_{i, j} \frac{2 P_i^{\frac{1 - \mathi \theta}{2}} P_j^{\frac{1 + \mathi \theta}{2}} - P_i - P_j}{2 (P_i - P_j)} \bra{i} a^{\dagger} a \ket{j} \ket{i} \bra{j} \nonumber\\
    =~& \mathi \sum_{i, j} \frac{2 P_i^{\frac{1 + \mathi \theta}{2}} P_j^{\frac{1 - \mathi \theta}{2}} - P_i - P_j}{2 (P_i - P_j)} \bra{i} b_{\theta}^{\dagger} b_{\theta} \ket{j} \ket{i} \bra{j} + \mathi \sum_{i, j} \frac{2 P_i^{\frac{1 - \mathi \theta}{2}} P_j^{\frac{1 + \mathi \theta}{2}} - P_i - P_j}{2 (P_i - P_j)} \bra{i} a^{\dagger} a \ket{j} \ket{i} \bra{j} . 
\end{align}
Therefore, for $\mathcal{R}_{\mathcal{N}, \rho} (\sigma) = \int_{- \infty}^{\infty} \dd \theta \, f (\theta) \mathcal{R}_{\mathcal{N}, \rho, \theta} (\sigma)$ with $f (\theta) = \frac{\pi}{2 (\cosh (\pi \theta) + 1)}$, and $b_{\theta} = \rho^{\frac{1 - \mathi \theta}{2}} a^{\dagger} \rho^{\frac{- 1 + \mathi \theta}{2}}$,
\begin{align}
    \left. \frac{\partial}{\partial \epsilon} (\mathcal{R}_{\mathcal{N}, \rho, \theta} (\sigma)) \right|_{\epsilon = 0} & = - \mathi R_{\theta} \sigma + \mathi \sigma R_{\theta} + b_{\theta} \sigma b_{\theta}^{\dagger} - \frac{1}{2} b_{\theta}^{\dagger} b_{\theta} \sigma - \frac{1}{2} \sigma b_{\theta}^{\dagger} b_{\theta} \nonumber\\
    & = - \mathi [R_{\theta}, \sigma] +\mathcal{D} [b_{\theta}] \sigma, 
\end{align}
and
\begin{equation}
    \left. \frac{\partial}{\partial \epsilon} (\mathcal{T}_{\mathcal{N}, \rho} (\sigma)) \right|_{\epsilon = 0} = - \mathi \left[ \int_{- \infty}^{\infty} \dd \theta \, f (\theta) R_{\theta}, \sigma \right] + \int_{- \infty}^{\infty} \dd \theta \, f (\theta) \mathcal{D} [b_{\theta}] \sigma .
\end{equation}


\section{Decoherence Limit of Petz Map and Twirled Petz Map are Diffusion Models}
\label{apxsec:classical_limit}

The most natural way of thinking of the probability distribution as the classical counterpart of a quantum state is through the Wigner distribution. Consider a state with Wigner distribution $W (x, p) = \frac{1}{2 \pi} P (x)$. Its corresponding density matrix is
\begin{equation}
  \hat{\rho} = \int \dd x P (x) \ket{x} \bra{x} .
\end{equation}
On the other hand, it is well known that the $\mathcal{D} [\hat{a}] \hat{\rho} =\mathcal{D} [\hat{p}] \hat{\rho}$, using the momentum operator as the jump operator induces transformation on the Wigner distribution,
\begin{equation}
    \mathcal{D} [\hat{p}] \hat{\rho} = \frac{1}{2} \int \dd x \frac{\partial^2 P}{\partial x^2} (x) \ket{x} \bra{x} .
\end{equation}
Therefore, we can treat the process $\dot{\hat{\rho}} =\mathcal{D} [\hat{p}] \hat{\rho}$ with $\hat{\rho} = \int \dd x P (x) \ket{x} \bra{x}$, as the classical decoherence of quantum Lindbladian evolution. Now, a natural question is: what do the continuous-time Petz map and the continuous-time twirled Petz map look like for such a quantum channel? In this appendix section, we provide that answer: any continuous-time rotated Petz map (namely, including the original continuous-time Petz map) is simply the standard denoising Fokker-Planck equation!

In and only in this appendix section, we always denote the quantum operator $\hat{L}$ (with hat) on the Hilbert space of states, and denote its corresponding differential operator $L$ (without hat) on function space.

\subsection{Differential operator representation}

For calculation convenience, we first state the differential operator representation $L$ for any operator $\hat{L}$. Here, $\hat{L}$ is an operator defined on the Hilbert space of states. Under basis of $\left\{ \ket{x} \right\}_{x \in \mathbb{R}}$, $\hat{L}$ has form of
\begin{equation}
    \hat{L} = \int \dd x \dd x'  \bra{x} \hat{L} \ket{x'} \ket{x} \bra{x'} .
\end{equation}
We define the kernel:
\begin{equation}
    K (x, x') := \bra{x} \hat{L} \ket{x'} .
\end{equation}
Let $\hat{L}$ acting on $\ket{\psi} = \int \dd x \psi (x) \ket{x}$, where the wavefunction $\psi (x)$ can be expressed by
\begin{equation}
    \braket{x}{\psi} = \int \dd x' \psi (x') \braket{x}{x'} = \int \dd x' \psi (x') \braket{x}{x'} = \int \dd x' \psi (x') \delta (x - x') = \psi (x) .
\end{equation}
Therefore,
\begin{align}
    \hat{L} \ket{\psi} & = \int \dd x \dd x' \dd x''  \bra{x} \hat{L} \ket{x'} \psi (x'') \ket{x} \braket{x'}{x''} \nonumber\\
    & = \int \dd x \dd x'  \bra{x} \hat{L} \ket{x'} \psi (x') \ket{x} \nonumber\\
    & = \int \dd x \left( \int \dd x' K (x, x') \psi (x') \right) \ket{x} . 
\end{align}
If we define a differential operator $L$, acting on any wavefunction $\psi$, such that
\begin{equation}
    L \psi (x) := \int \dd x' K (x, x') \psi (x'),
\end{equation}
we get
\begin{equation}
    \hat{L} \ket{\psi} = \int \dd x L \psi (x) \ket{x} .
\end{equation}
This means that $\hat{L}$ acting on $\ket{\psi}$ in state space corresponds to the differential operator $L$ acting on the wavefunction $\psi$. We denote this correspondence
\begin{equation}
    \hat{L} \ket{\psi} \leftrightarrow L \psi .
\end{equation}
On the other hand, we can easily check that the operator product between any $\hat{L}_1$ and $\hat{L}_2$ corresponds to the differential operator composite between $L_1$ and $L_2$. In fact, let $K_1 (x, x') = \bra{x} \hat{L}_1 \ket{x'}$ and $K_2 (x, x') = \bra{x} \hat{L}_2 \ket{x'}$, we have
\begin{align}
    \hat{L}_1 \hat{L}_2 \ket{\psi} & = \int \dd x_1 \dd x_1' \dd x_2 \dd x_2'  \bra{x_1} \hat{L}_1 \ket{x_1'} \bra{x_2} \hat{L}_2 \ket{x_2'} \ket{x_1} \braket{x_1'}{x_2} \braket{x_2'}{\psi} \nonumber\\
    & = \int \dd x_1 \dd x_1' \dd x_2'  \bra{x_1} \hat{L}_1 \ket{x_1'} \bra{x_1'} \hat{L}_2 \ket{x_2'} \ket{x_1} \braket{x_2'}{\psi} \nonumber\\
    & = \int \dd x_1 \dd x_1' \dd x_2' K_1 (x_1, x'_1) K_2 (x_1', x'_2) \psi (x_2') \ket{x_1} \nonumber\\
    & = \int \dd x_1 \dd x_1' K_1 (x_1, x'_1) \left( \int \dd x_2' K_2 (x_1', x'_2) \psi (x_2') \right) \ket{x_1} \nonumber\\
    & = \int \dd x_1 \left( \int \dd x_1' K_1 (x_1, x'_1) L_2 \psi (x_1') \right) \ket{x_1} \nonumber\\
    & = \int \dd x_1 (L_1 L_2 \psi (x_1)) \ket{x_1}, 
\end{align}
namely, $\hat{L}_1 \hat{L}_2 \ket{\psi} \leftrightarrow L_1 L_2 \psi$.

Also, let us recall that $\hat{p} \leftrightarrow - \mathi \partial_x$ and $\hat{p}^2 \leftrightarrow - \partial_x^2$ have kernels $\mathi \frac{\partial}{\partial x'} \delta (x - x')$ and $- \frac{\partial^2}{{\partial x'}^2} \delta (x - x')$, this is because
\begin{align}
    \int \dd x'  \left( \mathi \frac{\partial }{\partial x'} \delta (x - x') \right) \psi (x') & = - \mathi \int \dd x'  \left( \frac{\partial }{\partial x'} \psi (x') \right) \delta (x - x') = \mathi \partial_x \psi (x), \\
    - \int \dd x'  \left( \frac{\partial^2}{{\partial x'}^2} \delta (x - x') \right) \psi (x') & = - \int \dd x'  \left( \frac{\partial^2}{{\partial x'}^2} \psi (x') \right) \delta (x - x') = - \partial^2_x \psi (x) . 
\end{align}

\subsection{Forward process is classical diffusion}
\label{apxsec:diag_forward}

Now let us consider a diagonal state $\hat{\rho} = \int \dd x P (x) \ket{x} \bra{x}$, its differential operator is simply a function multiplier:
\begin{equation}
    \hat{\rho} \ket{\psi} = \int \dd x P (x) \ket{x} \braket{x}{\psi} = \int \dd x P (x) \psi (x) \ket{x} .
\end{equation}
Also for $\hat{p}$, it is well known that $\hat{p} \leftrightarrow - \mathi \partial_x$. We can derive that, $\mathcal{D} [\hat{a}] \hat{\rho} =\mathcal{D} [\hat{p}] \hat{\rho}$ in the forward process is
\begin{align}
    \mathcal{D} [\hat{p}] \hat{\rho} \ket{\psi} & \leftrightarrow (- \mathi \partial_x) p (- \mathi \partial_x) \psi + \frac{1}{2} \partial^2_x (P \psi) + \frac{1}{2} p \partial^2_x (\psi) \nonumber\\
    & = - (P' \psi' + P \psi'') + \left( \frac{1}{2} P'' \psi + P' \psi' + \frac{1}{2} P \psi'' \right) + \frac{1}{2} P \psi'' = \frac{1}{2} P'' \psi . 
\end{align}
Here, we adopt the abbreviation $f' (x) = \frac{\partial f}{\partial x} (x)$ and $f'' (x) = \frac{\partial^2 f}{\partial x^2} (x)$ for any function $f$. Here $\frac{1}{2} P'' (x) = \frac{1}{2} \frac{\partial^2 P}{\partial x^2} (x)$ is a simple function multiplication, that is
\begin{equation}
    \mathcal{D} [\hat{p}] \hat{\rho} = \int \dd x \left( \frac{1}{2} \frac{\partial^2 P}{\partial x^2} (x) \right) \ket{x} \bra{x} .
\end{equation}
This is exactly the standard diffusion term in the classical diffusion model.

\subsection{Continuous-time Petz map under decoherence limit}
\label{apxsec:diag_petz}

\subsubsection{Dissipative term in continuous-time Petz map}

Now we can compute $\mathcal{D} [\hat{b}] \hat{\sigma} =\mathcal{D} \! \left[ \hat{\rho}^{\frac{1}{2}} \hat{p} \hat{\rho}^{- \frac{1}{2}} \right] \hat{\sigma}$ where $\hat{\sigma} = \int \dd x \, Q (x) \ket{x} \bra{x}$. Firstly,
\begin{equation}
    \hat{\rho}^{\frac{1}{2}} \hat{p} \hat{\rho}^{- \frac{1}{2}} \ket{\psi} \leftrightarrow \sqrt{P} (- \mathi \partial_x) \frac{1}{\sqrt{P}} \psi 
    = - \mathi \left( \partial_x - \frac{1}{2} (\partial_x \ln P) \right) \psi, 
\end{equation}
namely $\hat{b} \leftrightarrow b = - \mathi \left( \partial_x - \frac{1}{2} (\partial_x \ln P) \right)$. Similarly,
\begin{equation}
    \hat{\rho}^{- \frac{1}{2}} \hat{p} \hat{\rho}^{\frac{1}{2}} \ket{\psi} \leftrightarrow \frac{1}{\sqrt{P}} (- \mathi \partial_x) \sqrt{P} \psi 
    = - \mathi \left( \partial_x + \frac{1}{2} (\partial_x \ln P) \right) \psi, 
\end{equation}
namely $\hat{b}^{\dagger} \leftrightarrow b = - \mathi \left( \partial_x + \frac{1}{2} (\partial_x \ln P) \right)$. From now on, let us introduce the score function
\begin{equation}
    s (x) := \partial_x  (\ln P (x)) = \frac{P' (x)}{P (x)} .
\end{equation}
\textit{Term $\hat{b} \hat{\sigma} \hat{b}^{\dagger}$}: for test function $\psi$,
\begin{equation}
    \hat{b} \hat{\sigma} \hat{b}^{\dagger} \ket{\psi} \leftrightarrow - \left( \partial_x - \frac{1}{2} s \right) Q \left( \partial_x + \frac{1}{2} s \right) \psi = \left( - Q \partial_x^2 - Q' \partial_x + \left( - \frac{1}{2} s Q' - \frac{1}{2} s' Q + \frac{1}{4} s^2 Q \right) \right) \psi . 
\end{equation}
\textit{Term $\hat{b}^{\dagger} \hat{b} \hat{\sigma}$}: for test function
$\psi$,
\begin{equation}
    \hat{b}^{\dagger} \hat{b} \hat{\sigma} \ket{\psi} \leftrightarrow - \left( \partial_x + \frac{1}{2} s \right) \left( \partial_x - \frac{1}{2} s \right) Q \psi = \left( - Q \partial_x^2 - 2 Q' \partial_x + \left( - Q'' + \frac{1}{2} s' Q + \frac{1}{4} s^2 Q \right) \right) \psi . 
\end{equation}
\textit{Term $\hat{\sigma} \hat{b}^{\dagger} \hat{b}$}: for test function
$\psi$,
\begin{equation}
    \hat{\sigma} \hat{b}^{\dagger} \hat{b} \ket{\psi} \leftrightarrow - Q \left( \partial_x + \frac{1}{2} s \right) \left( \partial_x - \frac{1}{2} s \right) \psi = \left( - Q \partial_x^2 + \frac{1}{2} s' Q + \frac{1}{4} s^2 Q \right) \psi . 
\end{equation}
Eventually,
\begin{equation}
    \mathcal{D} [\hat{b}] \hat{\sigma} \ket{\psi} = \hat{b} \hat{\sigma} \hat{b}^{\dagger} \ket{\psi} - \frac{1}{2} \hat{b}^{\dagger} \hat{b} \hat{\sigma} \ket{\psi} - \frac{1}{2} \hat{\sigma} \hat{b}^{\dagger} \hat{b} \ket{\psi} \leftrightarrow \left( - \frac{1}{2} s q' - s' q + \frac{1}{2} q'' \right) \psi . 
\end{equation}
Here $- \frac{1}{2} s (x) \frac{\partial Q}{\partial x} (x) - \frac{\partial s}{\partial x} (x) Q (x) + \frac{1}{2} \frac{\partial^2 Q}{\partial x^2} (x)$ is a simple function multiplication, that is
\begin{equation}
    \mathcal{D} [\hat{b}] \hat{\sigma} = \int \dd x \left( - \frac{1}{2} s (x) \frac{\partial Q}{\partial x} (x) - \frac{\partial s}{\partial x} (x) Q (x) + \frac{1}{2} \frac{\partial^2 Q}{\partial x^2} (x) \right) \ket{x} \bra{x} .
\end{equation}

\subsubsection{Hamiltonian term in continuous-time Petz map}

Before computing $- \mathi [\hat{R}, \sigma]$, we recall that
\begin{equation}
    \hat{R} = - \frac{\mathi}{2} \int \dd x \dd x'  \frac{\sqrt{P (x)} - \sqrt{P (x')}}{\sqrt{P (x)} + \sqrt{P (x')}} \bra{x} \hat{p}^2 + \hat{b}^{\dagger} \hat{b} \ket{x'} \ket{x} \bra{x'} .
\end{equation}
We first check that
\begin{align}
    \hat{p}^2 \ket{\psi} & \leftrightarrow - \partial_x^2 \psi, \\
    \hat{b}^{\dagger} \hat{b} \ket{\psi} & \leftrightarrow \left( - \partial_x^2 + \frac{1}{2} s' + \frac{1}{4} s^2 \right) \psi . 
\end{align}
We notice that
\begin{equation}
    \int \dd x \dd x'  \frac{\sqrt{P (x)} - \sqrt{P (x')}}{\sqrt{P (x)} + \sqrt{P (x')}} \bra{x} \left( \frac{1}{2} s' (x) + \frac{1}{4} s (x)^2 \right) \ket{x'} \ket{x} \bra{x'} = 0.
\end{equation}
Remember $\hat{p}^2 \leftrightarrow - \partial_x^2$ has kernel $- \frac{\partial^2}{{\partial x'}^2} \delta (x - x')$. Then, for computing $\hat{R}$, we just need to compute
\begin{align}
    \hat{R} \ket{\psi} & = - \mathi \int \dd x \dd x'  \frac{\sqrt{P (x)} - \sqrt{P (x')}}{\sqrt{P (x)} + \sqrt{P (x')}} \left( - \frac{\partial^2}{{\partial x'}^2} \delta (x - x') \right) \psi (x') \ket{x} \nonumber\\
    & = \mathi \int \dd x \dd x'  \frac{\partial^2}{{\partial x'}^2}  \left( \frac{\sqrt{P (x)} - \sqrt{P (x')}}{\sqrt{P (x)} + \sqrt{P (x')}} \psi (x') \right) \delta (x - x') \ket{x} \nonumber\\
    & = \mathi \int \dd x \left( - \frac{P'}{2 P} \psi' + \frac{{P'}^2 - P P''}{4 P^2} \psi \right) \ket{x} \nonumber\\
    & = \mathi \int \dd x \left( - \frac{s}{2} \partial_x - \frac{1}{4} s' \right) \psi \ket{x} .
\end{align}
This means that the Hermitian
\begin{equation}
    \hat{R} \leftrightarrow R = - \frac{\mathi}{2} s \partial_x - \frac{\mathi}{4} s' .
\end{equation}
Eventually,
\begin{equation}
    - \mathi [\hat{R}, \hat{\sigma}] \ket{\psi} \leftrightarrow - \frac{1}{2} s \partial_x (Q \psi) + \frac{1}{2} s Q \partial_x \psi = \left( - \frac{1}{2} s Q' \right) \psi .
\end{equation}
Here $- \frac{1}{2} s (x) \frac{\partial Q}{\partial x} (x)$ is a simple function multiplication, that is
\begin{equation}
    - \mathi [\hat{R}, \hat{\sigma}] = \int \dd x \left( - \frac{1}{2} s (x) \frac{\partial Q}{\partial x} (x) \right) \ket{x} \bra{x} .
\end{equation}

\subsubsection{Final expression of continuous-time Petz map under decoherence limit}

Finally, we have (remember $s = \partial_x  (\ln P (x))$ is the score function)
\begin{align}
    - \mathi [\hat{R}, \hat{\sigma}] \ket{\psi} & \leftrightarrow \left( - \frac{1}{2} s Q' \right) \psi, \\ 
    \mathcal{D} [\hat{b}] \hat{\sigma} \ket{\psi} & \leftrightarrow \left( - \frac{1}{2} s Q' - s' Q + \frac{1}{2} Q'' \right) \psi, \\
    (- \mathi [\hat{R}, \hat{\sigma}] +\mathcal{D} [\hat{b}] \hat{\sigma}) \ket{\psi} & \leftrightarrow \left( - \partial_x (s Q) + \frac{1}{2} Q'' \right) \psi . 
\end{align}
We note here that both $- \mathi [\hat{R}, \hat{\sigma}]$ and $\mathcal{D} [\hat{b}]$ are not trace-class, but their summation is trace-class. Finally, for momentum jump operator $\hat{p}$, and for any state $\hat{\rho} = \int \dd x \, P (x) \ket{x} \bra{x}$, $\hat{\sigma} = \int \dd x \, Q (x) \ket{x} \bra{x}$, we have
\begin{equation}
     - \mathi [\hat{R}, \hat{\sigma}] +\mathcal{D} [\hat{b}] \hat{\sigma} = \int \dd x \left( - \frac{\partial}{\partial x} \left( \left( \frac{\partial}{\partial x} \ln P (x) \right) Q (x) \right) + \frac{1}{2} \frac{\partial^2}{\partial x^2} (Q (x)) \right) \ket{x} \bra{x} .
\end{equation}
This is exactly the standard denoising term in the classical diffusion model.

\subsection{Continuous-time rotated and twirled Petz map under decoherence limit}
\label{apxsec:diag_rotated}

Now consider jump operator $\hat{b}_{\theta} = \hat{\rho}^{\frac{1 - \mathi \theta}{2}} \hat{p} \hat{\rho}^{\frac{- 1 + \mathi \theta}{2}}$, we have
\begin{equation}
    \hat{b}_{\theta} \ket{\psi} = P^{\frac{1 - \mathi \theta}{2}} (- i \partial_x) P^{\frac{- 1 + \mathi \theta}{2}} \psi = \left( - i \partial_x + \frac{i + \theta}{2} (\partial_x \ln P) \right) \psi .
\end{equation}
Similarly, for $\hat{b}^{\dagger}_{\theta} = \hat{\rho}^{\frac{- 1 - \mathi \theta}{2}} \hat{p} \hat{\rho}^{\frac{1 + \mathi \theta}{2}}$, we have
\begin{equation}
    \hat{b}^{\dagger}_{\theta} \ket{\psi} = P^{\frac{1 - \mathi \theta}{2}} (- i \partial_x) P^{\frac{- 1 + \mathi \theta}{2}} \psi = \left( - i \partial_x + \frac{- i + \theta}{2} (\partial_x \ln P) \right) \psi .
\end{equation}
Then, we can compute that
\begin{align}
    \mathcal{D} [\hat{b}_{\theta}] \hat{\sigma} \ket{\psi} & = \hat{b}_{\theta} \hat{\sigma} \hat{b}^{\dagger}_{\theta} \ket{\psi} - \frac{1}{2} \hat{b}^{\dagger}_{\theta} \hat{b}_{\theta} \hat{\sigma} \ket{\psi} - \frac{1}{2} \hat{\sigma} \hat{b}^{\dagger}_{\theta} \hat{b}_{\theta} \ket{\psi} \nonumber\\
    & \leftrightarrow \left( - \frac{1}{2} s Q' - s' Q + \frac{1}{2} Q'' \right) \psi . 
\end{align}
Recall in general $- \mathi \hat{R}_{\theta} = \sum_{i, j} \frac{2 P_i^{\frac{1 - \mathi \theta}{2}} P_j^{\frac{1 + \mathi \theta}{2}} - P_i - P_j}{2 (P_i - P_j)} \bra{i} \hat{a}^{\dagger} \hat{a} \ket{j} \ket{i} \bra{j} + \sum_{i, j} \frac{2 P_i^{\frac{1 + \mathi \theta}{2}} P_j^{\frac{1 - \mathi \theta}{2}} - P_i - P_j}{2 (P_i - P_j)} \bra{i} \hat{b}^{\dagger}_{\theta} \hat{b}_{\theta} \ket{j} \ket{i} \bra{j}$. We first get
\begin{equation}
    \hat{b}^{\dagger}_{\theta} \hat{b}_{\theta} \ket{\psi} \leftrightarrow \left( - \partial_x^2 - \mathi \theta s \partial_x + \frac{1 - \mathi \theta}{2} s' + \frac{1 + \theta^2}{4} s^2 \right) \psi . 
\end{equation}
In order to simplify $\hat{R}_{\theta}$, we take the expansion
\begin{align}
    \!\!\! \frac{2 P (x)^{\frac{1 \mp \mathi \theta}{2}} P (x')^{\frac{1 \pm \mathi \theta}{2}} - P (x) - P (x')}{2 (P (x) - P (x'))} & = \mp \frac{\mathi \theta}{2} - \frac{1 + \theta^2}{8} s (x) (x - x') + (1 + \theta^2) \left( \frac{3 s' (x) \pm \mathi \theta s (x)^2}{48} \right) (x - x')^2 + \cdots.
\end{align}
The kernel of $\hat{b}^{\dagger}_{\theta} \hat{b}_{\theta}$ is
\begin{equation}
    \bra{x} \hat{b}^{\dagger}_{\theta} \hat{b}_{\theta} \ket{x'} = \left( \frac{1 - \mathi \theta}{2} s' (x) + \frac{1 + \theta^2}{4} s (x)^2 \right) \delta (x - x') + \mathi \theta s \frac{\partial}{\partial x'} \delta (x - x') - \frac{\partial^2}{{\partial x'}^2} \delta (x - x') .
\end{equation}
We need to use the following relations:
\begin{align}
    \int \dd x' (x - x') \left( \frac{\partial}{\partial x'} \delta (x - x') \right) \psi (x') & = \psi, \\ \int \dd x' (x - x') \left( \frac{\partial^2}{{\partial x'}^2} \delta (x - x') \right) \psi (x') & = - 2 \partial_x \psi, \\
    \int \dd x' (x - x')^2 \left( \frac{\partial^2}{{\partial x'}^2} \delta (x - x') \right) \psi (x') & = 2 \psi . 
\end{align}
This yields
\begin{align}
    & \int \dd x'  \frac{2 P (x)^{\frac{1 - \mathi \theta}{2}} P (x')^{\frac{1 + \mathi \theta}{2}} - P (x) - P (x')}{2 (P (x) - P (x'))} \left( - \frac{\partial^2}{{\partial x'}^2} \delta (x - x') \right) \psi (x') \nonumber\\
    =~& - \int \dd x' \left( - \frac{\mathi \theta}{2} - \frac{1 + \theta^2}{8} s (x) (x - x') + (1 + \theta^2) \left( \frac{3 s' (x) + \mathi \theta s (x)^2}{48} \right) (x - x')^2 \right) \left( \frac{\partial^2}{{\partial x'}^2} \delta (x - x') \right) \psi (x') \nonumber\\
    =~& \frac{\mathi \theta}{2} \partial_x^2 \psi - \frac{1 + \theta^2}{4} s \partial_x \psi - (1 + \theta^2) \left( \frac{3 s' + \mathi \theta s^2}{24} \right) \psi . \\
    & \int \dd x'  \frac{2 P (x)^{\frac{1 + \mathi \theta}{2}} P (x')^{\frac{1 - \mathi \theta}{2}} - P (x) - P (x')}{2 (P (x) - P (x'))} \bra{x} \hat{b}^{\dagger}_{\theta} \hat{b}_{\theta} \ket{x'} \psi (x') \nonumber\\
    =~& \int \dd x' \left( \frac{\mathi \theta}{2} - \frac{1 + \theta^2}{8} s (x) (x - x') + (1 + \theta^2) \left( \frac{3 s' (x) - \mathi \theta s (x)^2}{48} \right) (x - x')^2 \right) \nonumber\\
    & \times \left( \left( \frac{1 - \mathi \theta}{2} s' (x) + \frac{1 + \theta^2}{4} s (x)^2 \right) \delta (x - x') + \mathi \theta s \frac{\partial}{\partial x'} \delta (x - x') - \frac{\partial^2}{{\partial x'}^2} \delta (x - x') \right) \psi (x') \nonumber\\
    =~& \frac{\mathi \theta}{2} \left( \frac{1 - \mathi \theta}{2} s' + \frac{1 + \theta^2}{4} s^2 \right) \psi - \mathi \theta s \left( \frac{\mathi \theta}{2} \partial_x + \frac{1 + \theta^2}{8} s \right) \psi + \left( - \frac{\mathi \theta}{2} \partial_x^2 - \frac{1 + \theta^2}{4} s \partial_x - (1 + \theta^2) \left( \frac{3 s' - \mathi \theta s^2}{24} \right) \right) \psi . 
\end{align}
Therefore, we add these two equations together
\begin{equation}
    \hat{R}_{\theta} \leftrightarrow - \frac{\mathi}{2} s \partial_x - \frac{\mathi + \theta}{4} s' .
\end{equation}
And immediately, $- \mathi [\hat{R}_{\theta}, \hat{\sigma}] \ket{\psi} \leftrightarrow \left( - \frac{1}{2} s q' \right) \psi$.

Finally, we have (remember $s = \partial_x  (\ln P (x))$ is the score function)
\begin{align}
    - \mathi [\hat{R}_{\theta}, \hat{\sigma}] \ket{\psi} & \leftrightarrow \left( - \frac{1}{2} s Q' \right) \psi, \\ 
    \mathcal{D} [\hat{b}] \hat{\sigma} \ket{\psi} & \leftrightarrow \left( - \frac{1}{2} s Q' - s' Q + \frac{1}{2} Q'' \right) \psi, \\ (- \mathi [\hat{R}_{\theta}, \hat{\sigma}] +\mathcal{D} [\hat{b}]
    \hat{\sigma}) \ket{\psi} & \leftrightarrow \left( - \partial_x (s Q) + \frac{1}{2} Q'' \right) \psi . 
\end{align}

In other word, for momentum jump operator $\hat{p}$, and for any state $\hat{\rho} = \int \dd x P (x) \ket{x} \bra{x}$, $\hat{\sigma} = \int\dd x \, Q (x) \ket{x} \bra{x}$, we have
\begin{equation}
    - \mathi [\hat{R}_{\theta}, \hat{\sigma}] +\mathcal{D} [\hat{b}] \hat{\sigma} = \int \dd x \left( - \frac{\partial}{\partial x} \left( \left( \frac{\partial}{\partial x} \ln P (x) \right) Q (x) \right) + \frac{1}{2} \frac{\partial^2}{\partial x^2} (Q (x)) \right) \ket{x} \bra{x} .
\end{equation}
Again, this is still exactly the standard denoising term in the classical diffusion model!

\subsection{Derivation of discrete variable diffusion models by Petz map}

Suppose $\hat{\rho} = \sum_i P_i \ket{i} \bra{i}$ and $\hat{\sigma} = \sum_i Q_i \ket{i} \bra{i}$, and jump operators $\hat{a}_{i j} = \lambda^{1 / 2}_{i j} \ket{i} \bra{j}$ with $\lambda_{i j} \geq 0$ and $i \neq j$. Then
\begin{equation}
    \mathcal{D} [\hat{a}_{i j}] \hat{\rho} = \lambda_{i j} \ket{i} \bra{j} \rho \ket{j} \bra{i} - \frac{\lambda_{i j}}{2} \left( \ket{j} \bra{j} \hat{\rho} + \hat{\rho} \ket{j} \bra{j} \right) = \lambda_{i j} P_j \ket{i} \bra{i} - \lambda_{i j} P_j \ket{j} \bra{j},
\end{equation}
The forward process of the diffusion model for the discrete variable is exactly the classical master equation
\begin{equation}
    \dot{P}_i = \bra{i} \left( \sum_{i' j'} \mathcal{D} [\hat{a}_{i' j'}] \hat{\rho} \right) \ket{i} = \sum_{j'} \lambda_{i j'} P_{j'} - \left( \sum_{i'} \lambda_{i' i} \right) P_i = \sum_{j \neq i} \lambda_{i j} P_j - \left( \sum_{j \neq i} \lambda_{j i} \right) P_i .
\end{equation}
Now, let us compute Petz map, which satisfies $-\mathcal{D} [\hat{a}] \hat{\rho} = - \mathi [R, \hat{\rho}] +\mathcal{D} [\hat{b}] \hat{\rho}$. First of all, $\hat{b}_{i j} = \hat{\rho}^{\frac{1}{2}} \hat{a}^{\dagger}_{i j} \hat{\rho}^{- \frac{1}{2}} = \sqrt{\lambda_{i j} P_j / P_i} \ket{j} \bra{i}$. Namely,
\begin{equation}
    \mathcal{D} [\hat{b}_{i j}] \hat{\rho} = \lambda_{i j} P_j \ket{j} \bra{j} - \lambda_{i j} P_j \ket{i} \bra{i} = -\mathcal{D} [\hat{a}_{i j}] \hat{\rho} .
\end{equation}
And for $\hat{a}_{i j}^{\dagger} \hat{a}_{i j} = \lambda_{i j} \ket{j} \bra{j}$ and $\hat{b}_{i j}^{\dagger} \hat{b}_{i j} = \lambda_{i j} P_j / P_i \ket{i} \bra{i}$, we have
\begin{equation}
    \hat{R}_{i j} = - \frac{\mathi}{2} \sum_{i', j'} \frac{\sqrt{P_{i'}} - \sqrt{P_{j'}}}{\sqrt{P_{i'}} + \sqrt{P_{j'}}} \bra{i'} (\hat{a}_{i j}^{\dagger} \hat{a}_{i j} + \hat{b}_{i j}^{\dagger} \hat{b}_{i j}) \ket{j'} \ket{i'} \bra{j'} .
\end{equation}
$\sqrt{P_{i'}} - \sqrt{P_{j'}}$ can be non-zero only if $i' \neq j'$. But then $i' \neq j'$ implies that $\bra{i'} (\hat{a}_{i j}^{\dagger} \hat{a}_{i j} + \hat{b}_{i j}^{\dagger} \hat{b}_{i j}) \ket{j'}$ must be zero (because $\hat{a}_{i j}^{\dagger} \hat{a}_{i j}, \hat{b}_{i j}^{\dagger} \hat{b}_{i j}$ are diagonal), and $\hat{R}_{i j} = 0$ for any $i, j$. 
Namely, the denoiser in discrete space is
\begin{equation}
    \dot{Q}_i = - \left( \sum_{j \neq i} \left( \lambda_{i j} \frac{P_j}{P_i} \right) \right) Q_i + \sum_{j \neq i} \left( \left( \lambda_{j i} \frac{P_i}{P_j} \right) Q_j \right) .
\end{equation}
Therefore, for any forward jump $j \rightarrow i$ with strength $\lambda_{i j}$ with $i \neq j$, the reversal process induced by the Petz map is a transition process $j \rightarrow i$ with strength $\lambda_{j i} P_i / P_j$. This exactly reproduces the result we derived in SM \ref{apxsec:dv}.

\end{widetext}

\end{document}